\documentclass[twoside,11pt]{article}
\usepackage[abbrvbib, preprint]{jmlr2e}
\usepackage[square,numbers]{natbib}

\usepackage{graphics}
\usepackage{trimclip} % to clip around image, after scaling the image itself

\usepackage[linesnumbered,ruled,vlined]{algorithm2e}					  
\usepackage{enumitem}
\usepackage{float}

\usepackage{amsthm}
\newtheorem{theorem}{Theorem}

\newtheorem{lemma}{Lemma}

\newtheorem{assumption}{Assumption}

\newtheorem{definition}{Definition}

\usepackage{verbatim}
\usepackage{multirow}

\usepackage{xr}
\usepackage{balance}

\usepackage{hyperref}
\usepackage{url}
\usepackage{amsmath}
\usepackage{amssymb}

\usepackage{mathtools}

\usepackage{dsfont}
\usepackage{upgreek}
\usepackage{bbm}			% mathbb vs mathds

\usepackage{graphicx}
\usepackage{epstopdf}
\graphicspath{ {./figures/} } % setups path to find graphics and images
\usepackage{tikz}
\usepackage{float}

\usepackage{booktabs} % to thicken table lines (\toprule, \midrule, \bottomrule)

\usepackage[labelfont={footnotesize,bf}, textfont={footnotesize}]{caption}

\usepackage{subcaption}
\usepackage{adjustbox }

\usepackage{xspace}

\usepackage[scaled=.7]{beramono}

\DeclareRobustCommand\sampleline[1]{%
  \tikz\draw[#1] (0,0) (0,\the\dimexpr\fontdimen22\textfont2\relax)
  -- (1.4em,\the\dimexpr\fontdimen22\textfont2\relax);%
}
\makeatletter
\g@addto@macro\bfseries{\boldmath}
\makeatother

\newlength{\phaserulewidth}
\newcommand{\setphaserulewidth}{\setlength{\phaserulewidth}}

\makeatother
\setphaserulewidth{.3pt}

\newcommand{\Sec}[1]		{Sec.\,\ref{#1}}
\newcommand{\Fig}[1]		{Fig.\,\ref{#1}}

\newcommand{\Eq}[1]			{Eq.\,\ref{#1}}
\newcommand{\Expr}[1]			{Expr.\,\ref{#1}}
\newcommand{\Exprs}[1]			{Expressions\,\ref{#1}}

\newcommand{\Alg}[1]		{Alg.\,\ref{#1}}
\newcommand{\Theorem}[1]{Theorem\,\ref{#1}}
\newcommand{\Theorems}[1]{Theorems\,\ref{#1}}

\newcommand{\Lemma}[1]{Lemma\,\ref{#1}}
\newcommand{\Lemmas}[1]{Lemmas\,\ref{#1}}

\newcommand{\Assumption}[1]{Assumption\,\ref{#1}}
\newcommand{\Assumptions}[1]{Assumptions\,\ref{#1}}

\newcommand{\Problem}[1]{Problem\,\ref{#1}}	
	
\newcommand{\Appendix}[1]{Appendix\,\ref{#1}}	
\newcommand{\ie}   			{i.e.\@\xspace}
\newcommand{\eg}   			{e.g.\@\xspace}

\newcommand{\wrt}   		{w.r.t.\@\xspace}

\newcommand{\iid}   		{iid\@\xspace}

\newcommand{\zero}      {\mathbf{0}}
\newcommand{\zeros}[1]  {\zero_{#1}}
\newcommand{\bigO} 		  {\mathcal{O}} % complexity order O
\newcommand{\ind}       {\mathds{1}}%{\mathbb{1}}
\newcommand{\Ind}[1]    { \ind{\{#1\}} }%{\mathbb{1}_{\{#1\}}}{\raisebox{0.3em}{\{#1\}}}}%

\newcommand{\Hilbert}      {\mathbb{H}}
\newcommand{\R}      {\mathbb{R}}
\newcommand{\N}      {\mathbb{N}}
\newcommand{\Ltwo}      {\mathcal{L}^2_{p^{\alpha}}}
\newcommand{\ltwo}{L^2_{p^{\alpha}}}
\newcommand{\LP}        {L}
\newcommand{\LQ}        {R}
\newcommand{\ELP}        {\mathbf{L}}
\newcommand{\ELQ}        {\mathbf{R}}
\newcommand{\ELPQ}        {\mathbf{LR}}
\newcommand{\CO}        {\mathcal{L}_K}
\newcommand{\transpose}		{{\mathsmaller \top}}

\newcommand{\inlinetitle}[2]  {\noindent\textbf{\emph{#1}{#2}}}%\parskip

\renewcommand*{\top} %{{\mathsmaller \top}}
{{\mkern-1.5mu\mathsf{T}}}

\newcommand{\PE}{P\!E}  % variables for Pearson's divergence

\newcommand{\fdiv} 	{$\phi$-divergence\xspace}
\newcommand{\fdivs} {$\phi$-divergences\xspace}
\newcommand{\KLdiv} {KL-divergence\xspace}
\newcommand{\PEdiv} {$\chi^2$-divergence\xspace}

\newcommand{\OLRE}	{OLRE\xspace} 
 
\newcommand{\LRE}	  {LRE\xspace} %\density-ratio estimation

\newcommand{\Expec}[1]{\mathbb{E}[#1]}
\newcommand{\ExpecN}[1]{\mathbb{E}_{p(x)}[#1]}

\newcommand{\ExpecAlpha}[1]{\mathbb{E}_{p^{\alpha}(y)}[#1]}

\newcommand{\expec}{\mathbb{E}}

\newcommand{\ExpecA}[1]{\mathbb{E}_{q(x')}[#1]}
\newcommand{\ExpecRHO}[1]{\mathbb{E}_{(p(x),q(x'))}[#1]}

\newcommand{\dotH}[2]{\langle #1,#2 \rangle_{\Hilbert}}

\newcommand{\Nabla} {\nabla\!}

\newcommand{\norm}[1]{\left\lVert#1\right\rVert}

\newcommand{\id}{I}

\newcommand{\abs}[1]{\left|#1\right|}
\newcommand{\Npre}{n}
\newcommand{\Npost}{n'}
\newcommand{\Nref}{n}

\newcommand{\X}{{\mathcal{X}}}
\newcommand{\setpreunivariate}{X}
\newcommand{\setpostunivariate}{X'}
\newcommand{\sourceparameter}{\beta}

\newcommand{\q}{q}  %or p'

\newcommand{\padded}[1] {\,#1\,}
\newcommand{\void}[1] {\padded{\cdot}}

\newcommand{\pdfuncs}		{pdfs\xspace}
\newcommand{\pdf}		    {pdf\xspace}
\newcommand{\pdfs}	 	  {pdfs\xspace}

\DeclareMathOperator*{\argmin}{\arg\!\min}

\newcommand{\conj}{\star}

\usepackage[normalem]{ulem} % for strikethrough
\newcounter{marginNoteCounter}

\newcommand{\intLTWO}[1]{\int_{\mathcal{X}}#1dP^{\alpha}(x)}

\newcommand{\idH}{\id_{\Hilbert}}

\makeatletter
\newcommand{\pushright}[1]{\ifmeasuring@#1\else\omit\hfill$\displaystyle#1$\fi\ignorespaces}
\newcommand{\pushleft}[1]{\ifmeasuring@#1\else\omit$\displaystyle#1$\hfill\fi\ignorespaces}
\makeatother

\usepackage[font=small,labelfont=bf]{caption}

\ShortHeadings{Online non-parametric likelihood-ratio estimation}{de la Concha and Vayatis and Kalogeratos}
\firstpageno{1}

\begin{document}

\title{Online non-parametric likelihood-ratio estimation by Pearson-divergence functional minimization.}

\author{\name Alejandro de la Concha \quad  Nicolas Vayatis \quad 
\quad  Argyris Kalogeratos \\
\centering
\addr Centre Borelli, ENS Paris-Saclay, Université Paris-Saclay,  France \\
\email \smaller\{name.surname\}@ens-paris-saclay.fr \qquad \qquad \qquad \qquad \qquad \qquad \qquad \qquad } 

\maketitle

\begin{abstract}
Quantifying the difference between two probability density functions, $p$ and $q$, using available data, is a fundamental problem in Statistics and Machine Learning. %
A usual approach for addressing this problem is the likelihood-ratio estimation (\LRE) between $p$ and $q$, which -to our best knowledge- has been investigated mainly for the offline case. This paper contributes by introducing a new framework for online non-parametric \LRE (\OLRE) for the setting where pairs of \iid 
observations $(x_t\!\sim\!p, x'_t\!\sim\!q)$ are observed over time. The non-parametric nature of our approach has the advantage of being agnostic to the forms of 
$p$ and $q$. Moreover, we capitalize on the recent advances in Kernel Methods and functional minimization to develop an estimator that can be efficiently updated online%at every iteration
. We provide theoretical guarantees for the performance of the \OLRE method along with empirical validation in synthetic experiments.
\end{abstract}

\section{Introduction}

The likelihood-ratio between two probability density functions (\pdfs) is a quantity omnipresent in Statistics. For instance, the likelihood-ratio test has optimal %\RMV{properties in terms of} 
statistical power and it is a core tool %\RMV{in general the warhorse }
in statistical hypothesis testing \citep{Neymann1933,Casella2006}. %
In one of the related problems, change-point detection, the most widely-used methods, such as CUSUM \citep{Page1954} or Sriryaev-Roberts \citep{Shiryaev1963}, depend on the likelihood-ratio (see also \cite{Tartakovsky2014,Xie2021}). In Transfer Learning, it is possible to define a weighted cost function to solve a new problem taking into account prior knowledge provided by %\acquired from 
a different dataset; interestingly, this weighting function coincides with the likelihood-ratio \citep{Fishman1996,Zhuang2021}.

$f$-divergences have a similar role as they are also ubiquitous in Statistics. Classical problems such as Maximum Likelihood Estimation, Dimensionality Reduction, and Generative Modeling, to mention just a few, can be restated as $f$-divergence minimization problems \citep{Liese2006,Nguyen2010,Sugiyama2012,Agrawal2021}.

From a Machine Learning perspective, the interplay between the likelihood-ratio and $f$-divergences has been described via its variational formulation \citep{Nguyen2007}. There were identified situations where the $f$-divergence estimation between two measures amounts to a likelihood-ratio estimation (LRE) \emph{as an element} in of a functional space. 
This kind of result has motivated a plethora of non-parametric methods, based on Kernel Methods and Neural Networks \citep{Moustakides2019}, which do not need any further hypotheses regarding the functional form of $p$ and $q$ and just depend on observations coming from both those probability 
densities. These techniques have a wide range of applications in different domains  \citep{Basseville2013,Liu2013,Rubenstein2019,Zhang2021}.

Despite the aforementioned success of non-parametric \LRE methods for offline %(batch)
processing, to our best knowledge, there has been hardly any investigation about how the estimation of $f$-divergence and likelihood-ratio can be adapted to online settings and streaming data. The motivation for covering this gap is to pave the way so that non-parametric \LRE-based methods %\RMV{could be used and} 
bring gains in online learning, hypothesis testing, or various detection settings. 

\inlinetitle{Contribution}{.}~
To begin with, in this paper we introduce the new Online \LRE (\OLRE) problem, where one observes a stream of incoming pairs of observations $(x_t\!\sim\!p$, $x'_t\!\sim\!q)$, $t=1,2,...$, and the likelihood-ratio needs to be estimated on the fly. Then, we present the homonymous non-parametric \OLRE framework, along with a theoretical characterization of its convergence. Our approach nurtures mainly from three elements:%
\vspace{-3mm}
\begin{itemize}[leftmargin=1em,topsep=1em,itemsep=0px,noitemsep]
\item[$\bullet$] The formulation of the typical offline \LRE problem as a functional minimization problem seeking a solution among functions of a Reproducing Kernel Hilbert Space (RKHS) \citep{Nguyen2007}. 
\item[$\bullet$] The adaptation of a first-order optimization method, the %which is based on 
stochastic functional gradient descent (\cite{Kivinen2004}), combined with the framework of regularized paths in Hilbert spaces %proposed in 
\citep{Tarres2014}. 
\item[$\bullet$] The organic integration of the best practices for kernel-based \LRE %\RMV{via Kernel Methods } 
that have been developed recently. %since the first works of \cite{Nguyen2007,Sugiyama2012}. 
\end{itemize}
\vspace{-1em}
The \OLRE framework combines the above %ideas and 
elements, and enjoys the following technical properties: 
\vspace{-3mm}
\begin{itemize}[leftmargin=1em,topsep=1em,itemsep=0px,noitemsep]
    \item[$\bullet$] It does not require to know in advance the sample size, which can be even infinite. %\RMV{in advance, thus it could potentially manage an infinite number of incoming observations.}
    \item[$\bullet$] Our stochastic approximation aims to minimize the generalization error by solving the original functional minimization problem, instead of performing empirical risk minimization, and this way avoids over-fitting.
    \item[$\bullet$] Our analysis and the performance of the proposed method, highlight the bias of existing offline approaches that are based on empirical risk minimization, which rely on simple heuristics to  manage large amounts of data. 
    \item[$\bullet$] The cost of the iteration at time $t$ is $O(t)$, hence in total $O(t^2)$ for up to time $t$. 
    \item[$\bullet$] Our convergence results provide guidelines on how to select the hyperparameters of our method and its sensibility to different configurations. 
\end{itemize}

\section{Problem statement and background}{\label{sec:Preliminaries}}

%\mN{This sec. is a long concatenation of elements. The point of the problem we solve is not direct. In the first part, we need to give first the simple notations, then give the necessary LRE definitions, and directly give a formal problem definition. In a second part, we can have all the different building blocks...}

%The online setting considered in this paper supposes that at every time $t$ we observe a new pair of \iid observations $(x_t\!\sim\! p$, $x'_t\!\sim\! q)$. We will denote by $\Xi_{t}$ the minimum $\sigma$-algebra generated by the incoming observations up to time $t$, \ie $\sigma(\{(x_1,x'_1),...,(x_t,x'_t)\})$. %
%\NOTE{%In this section, 
%Next, we introduce the main building blocks used in \OLRE. }
%

In this section, we begin by presenting the likelihood-ratio estimation (\LRE) problem, and by defining the Online \LRE (\OLRE) problem version. Then, we present the main building blocks we use %to address \OLRE by developing 
for developing the homonymous \OLRE framework.

%\subsection{Important notions for \LRE}
\subsection{Likelihood-ratio estimation}
\label{sec:notions_LRE}

Let us denote the feature space $\mathcal{X} \subset \R^d$ and consider two probability measures $P$ and $Q$ which are absolutely continuous with respect to the Lebesgue measure denoted by $dx$, with densities $p$ and $q$ respectively. We also define the convex $\alpha$-mixture of the probability measures $P$ and $Q$ computed by $P^{\alpha}=(1-\alpha)P+\alpha Q$, and similarly for their densities $p$, $q$, where $0 \leq \alpha<1$ is user-defined parameter.

\inlinetitle{Relative likelihood-ratio}{.}~
We focus on the approximation of the relative likelihood-ratio between the \pdfs $q$ and $p$:
\vspace{-0.5em}
\begin{equation}\label{eq:likelihood-ratio}
  r^{\alpha}(x)=\frac{q(x)}{(1-\alpha)p(x)+\alpha q(x)} \in \mathbb{R}^+, \ \ \forall x \in \mathcal{X},
\end{equation}
where $0 \leq \alpha<1$ acts %here 
as a user-defined regularization parameter \citep{Yamada2011}. When $\alpha=0$, \Eq{eq:likelihood-ratio} recovers the usual likelihood-ratio $r_*^{\alpha=0}(x) = r(x) = \frac{q(x)}{p(x)}$. The $\alpha$-regularization addresses certain instability issues appearing when %trying to 
approximating an unbounded function. 
Specifically, when $\alpha>0$, it holds $r^{\alpha} \leq \frac{1}{\alpha}$%is upper-bounded by $\frac{1}{\alpha}$
, which is an upper-bound that will be proven to be important when we later study theoretically the convergence of the proposed method (see \Sec{sec:theoretical_guarantees}).  Typically, $\alpha$ should be close to $0$ to ensure that the approximated $r^{\alpha}(x)$ will remain relevant for the intended application of the likelihood-ratio, which is of course the core quantity of interest.

\inlinetitle{Defining the Online \LRE (\OLRE) setting}{.}~The online setting we introduce in this paper supposes that a new pair of \iid observations $(x_t\!\sim\! p$, $x'_t\!\sim\! q)$ is observed at every time $t$. Then, the objective is to approximate the relative likelihood-ratio $r^{\alpha}$ through a function $f_t$, which is updated at every time $t$. % with the incoming pair of observation $(x_t,x'_t)$. 
The function $f_t$ is an element of a non-parametric functional space, so there is no need to make a hypothesis about the nature of $p$ nor $q$. We denote by $\Xi_{t}$ the minimum $\sigma$-algebra generated by the incoming observations up to time $t$, \ie $\sigma(\{(x_1,x'_1),...,(x_t,x'_t)\})$.
%\textbf{Non-parametric estimation} 
%

\inlinetitle{Reproducing Kernel Hilbert Space}{.}~%
We aim to estimate $r^{\alpha}(x)$ with regards to a Reproducing Kernel Hilbert Space (RKHS) %denoted by 
$\Hilbert$ containing as elements %of, each element of which is a 
functions $f:\mathcal{X} \rightarrow \R$. %The space 
$\Hilbert$ is equipped with the inner product $\dotH{\cdot}{\cdot}:\Hilbert \times \Hilbert \rightarrow \R$, which will be reproduced by a Mercer Kernel; \ie~by a continuous symmetric real function, which is the  positive semi-definite kernel function $K(\cdot,\cdot) : \mathcal{X} \times \mathcal{X} \rightarrow \R$. %\NEW{where $K(x,x') = f(x)^\top f(x')$}, and $K(\cdot,\cdot)$ 
Then, the space $\Hilbert$ satisfies the following properties: 
\begin{equation}{\label{RKHS_properties}}
\!\!\!\!\!\!\!%
\begin{aligned}
\bullet  &  \ \ \dotH{K(x,\cdot)}{f} = f(x)\text{, for any } f \in \Hilbert;  \ \ \
 \\
\bullet  & \ \ \Hilbert= \overline{\operatorname{span}}(\{K(x,\cdot) : \forall x \in \mathcal{X}\}),%
\end{aligned}
\end{equation}\\
where $\overline{\operatorname{span}}$ refers to the closure of all the linear combinations of the elements $K(x,\cdot)$, $\forall x \in \mathcal{X}$. 

The first equality is known as the RKHS reproducing property. 

\subsection{Important notions for first-order optimization}{\label{sec:first_order_optimization}}

The main idea behind \OLRE is to see $r^{\alpha}$ as the solution of the functional optimization problem $\min_{f \in \Hilbert} L(f)$, where $L(f) : \Hilbert \rightarrow \R$ is a cost functional representing the real risk (\ie generalization error), with respect to an instantaneous loss-function $\ell(f) : \Hilbert \rightarrow \R$. Our optimization schema is based on the functional gradient of the cost function $\ell(f)$, and  produces a stochastic approximation $f_t$ that approaches the relative likelihood-ratio $r^{\alpha}$ at every time $t$. The estimation of $f_t$ requires only the previous estimate $f_{t-1}$ and the new observations $(x_t, x'_t)$. The geometry of $\Hilbert$, and more precisely the reproducing property of its elements, lead to an elegant closed-form expression for $f_t$.

\inlinetitle{Functional gradient \citep{Bauschke2011}}{.} Let $L: \Hilbert \rightarrow \R$ be a Gâteaux differentiable functional, and $[DL(f)](\cdot)$ its Gâteaux derivative. By $\Nabla_{f\!} L(f)$ we denote the functional gradient of $L$ at $f$, defined to be the element of $\Hilbert$ that satisfies: 
\begin{equation}
    [DL(f)](g)= \dotH{\Nabla_{f\!} L(f)}{g}, \ \ \ \forall g \in \Hilbert.
\end{equation}
The Riesz representation theorem tells us that $\Nabla_{f\!} L(f)$ exists and is unique. %
When $L(f)$ is also Fréchet differentiable at $f$, then the Gâteaux derivative and Fréchet derivative coincide. The Fréchet derivative has the advantage of satisfying the chain rule in a more natural way.

\inlinetitle{Functional Stochastic Gradient \citep{Kivinen2004}}{.} %
Suppose that the cost function $L(f)$ takes values in an RKHS, \ie $f \in \Hilbert$, and it has the form $L(f)=\mathbb{E}_{x}[\ell(f(x)]$. Given an independent realization $x \in \mathcal{X}$, we can compute the Fréchet derivative of $\ell(f(\cdot))$ \wrt $f$ as: %\mN{what is nice is that the expectation that L computes, also implies that you go for the generalization error. so maybe would be nice to somehow define the difference between going for the loss and going for the expectation of the loss. + maybe adding the index $x$ in the expectation helps to be more clear: $L(f)=\mathbb{E}_x[\ell(f(x)]$}
%%%
\begin{equation}{\label{eq:gradient}}
\Nabla_{f\!}  \ell(f(x))(\cdot)  
    = \frac{\partial  \ell(f(x))}{\partial f(x)}\frac{\partial f(x)}{\partial f} (\cdot)= \frac{\partial  \ell(f(x))}{\partial f(x)} K(x,\cdot).
\end{equation}
%%%
The first equality is a consequence of the chain rule for the Fréchet derivative; the second one is due to $\Hilbert$'s reproducing property that %provides the expression
due to which $\frac{\partial f(x)}{\partial f} (\cdot)=\frac{\partial \dotH{f}{K(x,\cdot)}}{\partial f}(\cdot)=K(x,\cdot)$. The operator $\Nabla_{f\!}  \ell(f(x))(\cdot)$ is %known as 
the functional stochastic gradient of $L(f)$ at $f$.

%In our approach, we use %the 
%$\Nabla_{f\!}  \ell(f(x))(\cdot)$ to minimize directly $L(f)$, \ie minimize the generalization error through minimizing the expected value of the loss, instead of going for empirical risk minimization through %that minimizes 
%the surrogate loss-function $\frac{1}{n}\sum_{i=1}^{n} \ell(f(x_i))$, which is prone to over-fitting. 

\section{Online LRE by \fdiv  minimization\!\!}{\label{sec:framework}}

\subsection{\LRE via \fdiv minimization}

\inlinetitle{\fdiv}{.}~To avoid notation conflicts, henceforth we will be referring to $f$-divergences by \fdivs. A \fdiv functional 
quantifies the similarity between two probability measures that are described by their \pdfuncs $p$, $q$:%\RMV{. Formally}:
%
%note: the integration domain below is X
\begin{equation}\label{eq:f-divergence}
%\begin{aligned}
    \!\!\!\!\!\! \mathcal{D}_\phi(P \Vert Q) %& 
		= \!\!\int \!p(x)\,\phi\!\left(\frac{q(x)}{p(x)}\right) dx 
   = \!\!\int \!\phi\!\left(r^*\right)\!(x) \,dP(x).\!\!\!\!
%\end{aligned}
\end{equation}
Interesting cases are those when the likelihood-ratio can be defined, hence when the support of $q$ is included in the support of $p$, and also when $\phi: \R \rightarrow \R$ is a convex and semi-continuous real function with $\phi(1) = 0$ \citep{Csiszar1967}.

The formulation of our optimization problem relies mainly on the following %lemma that provides a 
variational formulation for \fdivs. 
\begin{lemma}{\label{lemma:var_for}} (Lemma 1 in \cite{Nguyen2007}). 
For any class of functions %\RMV{$F$ mapping from $\mathcal{X}$ to $\R$}
%$\mathcal{F} = \{g: \mathcal{X} \rightarrow \R\}$, we have the lower-bound:
$\mathcal{F} : \mathcal{X} \rightarrow \R$, the lower-bound for the similarity between two probability measures is:
\begin{subequations}{\label{eq:variational_formulation}}
\begin{align}
%\begin{equation}
%\begin{aligned}
    \!\!\!\! \mathcal{D}_{\phi}(P \Vert Q) &= \!\!\int\!\! \phi\left(\frac{q}{p}\right)\!(x) \,dP(x) \label{eq:variational_formulation_1}\\ &\geq\, \sup_{ g \in \mathcal{F}} \int\!\! g(x') dQ(x') - \!\!\int\!\! \phi^{\conj}(g)(x) \,dP(x)\!\!\!\label{eq:variational_formulation_2}%
%\end{aligned}
%\end{equation}
\end{align}%
\end{subequations}%
where $\phi^{\conj}$ denotes the convex conjugate of $\phi: \R \rightarrow \R$. %
The equality \Eq{eq:variational_formulation_1} %of \Expr{eq:variational_formulation} 
holds if and only if the subdifferential $\nabla \phi (\frac{q}{p})$ contains an element of $\mathcal{F}$.
\end{lemma}
%
%Let us suppose 
The likelihood-ratio in terms of the solution to \Problem{eq:variational_formulation_2}, %\RMV{, denoted by }
$g^* = \sup_{ g \in \mathcal{F}} \!\int\!  g(x') dQ(x') - \!\int\! \phi^{\conj}(g)(x) d P(x)$, can be inferred by simply applying $\frac{q}{p}= (\nabla \phi)^{-1}(g^*)= \nabla \phi^{\conj} (g^*)$. For %the inversion $(\nabla \phi)^{-1}$ 
this to be possible, $\phi$ needs to be continuously-differentiable and strictly convex. %. Thanks to this hypothesis, 
%we can infer the likelihood-ratio in terms of the solution to \Problem{eq:variational_formulation}: %\RMV{, denoted by }
%$g^* = \sup_{ g \in \mathcal{F}} \!\int\!  g dQ - \!\int\! \phi^{\conj}(g) d P$, by simply applying $\frac{q}{p}= (\nabla \phi)^{-1}(g^*)= \nabla \phi^{\conj} (g^*)$. %
%$\frac{q}{p}= (\nabla \phi)^{-1}(g_*)= \nabla \phi^*(g_*)$.
%
As stated in the introduction, though, we will focus on approximating the relative likelihood-ratio $r^{\alpha}$ instead of the usual likelihood-ratio. Then, by fixing $\phi(y)=\frac{(y-1)^2}{2}$ for $y \in \R$, whose convex conjugate is $\phi^{\conj}(y^*)=\frac{(y^*)^2}{2}+ y^*$ for $y^* \in \R$, we recover the \PEdiv  (also known as Pearson-divergence):
%%%%
\begin{equation}\label{eq:PEaPQ}
 \PE(P^{\alpha} \Vert Q)=\int\! \left[\frac{(r^{\alpha}-1)^2}{2}\right]\!(x) \,dP^{\alpha}(x).
\end{equation}
%%%%%
Note that the factor $\frac{1}{2}$ is only introduced to facilitate later calculations. According to \Lemma{lemma:var_for}, the latter can be lower-bounded via its variational formulation:
%%%%
%\begin{equation}{\label{eq:PE_variational}}
%\begin{aligned}
\begin{subequations}
\begin{align}
&\!\!\!\PE(P^{\alpha}\Vert Q)  \geq
\sup_{f \in \Hilbert}
     \int \! (f-1)(x') \,dQ(x') \\
&  \ \ \ \ - \int \! \left[ \frac{(f-1)^2}{2} + (f-1) \right]\!(x)\,dP^{\alpha}(x)  \label{eq:PE_variational_1}\\
& =  \sup_{f \in \Hilbert}
     \int \! f(x') \,dQ(x') - \int \!  \frac{f^2(x)}{2} \,dP^{\alpha}(x) - \frac{1}{2}   \label{eq:PE_variational_2}\\
& =  \sup_{f \in \Hilbert}
     \int \! \bigg[f r^{\alpha} - \frac{f^2}{2}\bigg]\!(x)\,dP^{\alpha}(x) - \frac{1}{2}.\label{eq:PE_variational_3}%
\end{align}%
\label{eq:PE_variational}%
\end{subequations}
%\end{aligned}
%\end{equation}
%%%
\Eq{eq:PE_variational_2} is explained by the change of measure identity $\ExpecAlpha{f(y) r^{\alpha}(y)}=\ExpecA{g(x')}$. Then, thanks to %the conclusion of 
\Lemma{lemma:var_for}, we can obtain an estimator of the $r^{\alpha}$ by solving the following quadratic functional optimization problem defined in terms of the RKHS $\Hilbert$:%
\begin{subequations}
\begin{align}
\argmin_{f \in \Hilbert}  L^{\text{PE}}(f) & = \argmin_{f \in \Hilbert}
     \int \! \bigg[\frac{f^2(x)}{2} - f(x) r^{\alpha}(x) \bigg] dP^{\alpha}(x) + \frac{1}{2} \label{eq:PE_optimization_2_1} \\
& =  \argmin_{f \in \Hilbert}\int \! \frac{(f-r^{\alpha})^2(x)}{2}  dP^{\alpha}(x) + C \label{eq:PE_optimization_2_2}\\
 & =
  \argmin_{f \in \Hilbert}\  
 (1-\alpha) \! \int\! \frac{f^2(x)}{2} dP(x)  
 + \alpha \! \int \!\frac{f^2(x')}{2} dQ(x') -  \!\int\! f(x') \,dQ(x') + C. \label{eq:PE_optimization_2_3}%
\end{align}%
\label{eq:PE_optimization_2}%
\end{subequations}
To get \Eq{eq:PE_optimization_2_2}, we used that $\ExpecAlpha{r^{\alpha}(y)^2}=C$, where $C$ is some constant. For the final \Eq{eq:PE_optimization_2_3}, we used $\ExpecAlpha{f(y)}=\alpha \ExpecN{f(x)} +(1-\alpha)\ExpecA{f(x')}$. 

RULSIF is a popular offline \LRE algorithm  \citep{Yamada2011,Liu2013} aiming to solve \Problem{eq:PE_optimization_2} when the  user has access to $\setpreunivariate=\{x_t\}_{t=1}^{\Npre}$ and $\setpostunivariate=\{x_t\}_{t=1}^{\Npost}$ observations. \Problem{eq:PE_optimization_2} is then approximated via penalized empirical risk minimization. In this surrogate problem,  the space $\Hilbert$ is replaced by a finite-dimension subspace $S \subset \Hilbert$, a finite linear combination of $M$ elements selected uniformly at random from $\setpostunivariate$.  We will see in the experiments how this strategy along with the choice of a penalization term defined in terms of the Euclidean norm $\Vert{\cdot}\Vert_{2}$ (instead of the Hilbert norm) leads to an approximation error that does not disappear as $\Nref$ and $\Npost$ increase. (For further  details see Appendix A) %\ref{sec:RULSIF} )

\subsection{Online LRE by \PEdiv minimization}{\label{sec:online_minimization}}

In the previous section, we put forward our goal to solve the functional optimization \Problem{eq:PE_optimization_2} as new pairs of \iid observations $(x_t\!\sim\!p,x'_t\!\sim\!q)$ arrive over time. %, $\{x_t\!\sim\!p\}_{t \in \N}$ ($\{x'_t\!\sim\!q\}_{t \in \N}$)  are \iid and \NOTE{$x_t \sim P$ and $x'_t \sim Q$}. 
The proposed homonymous algorithm, which is called as well \OLRE, makes use of the approach of regularization paths.  

%Let us begin by defining $f_{\lambda_t}$ as the solution 
%to the following regularized cost function: 
Let us define next the %following 
regularized cost function with the help of a time-dependent regularization parameter $\lambda_t >0$:%
%%%
\begin{equation}{\label{eq:PE_optimization_3}}
\begin{aligned}
\!\!%f_{\NOTE{\lambda_t}} 
\!\!\!\!\!\!\!\!\!\!\!\!\!\!\!\!\!\!\!\!\!\!\!\!\!\!\!\!\!\! \min_{f \in \Hilbert}\ 
 (1-\alpha) \!\!\int\! \frac{f^2(x)}{2} dP(x)  
 + \alpha \!\!\int\! \frac{f^2(x')}{2} dQ(x')  -  \!\int\! f(x') \,dQ(x') + \frac{\lambda_t}{2}\norm{f}_{\Hilbert}
\end{aligned}\!\!\!\!\!\!\!\!\!\!\!\!\!%
\end{equation}%
%%%

The idea of stochastic approximation via regularization paths is to use a decreasing regularization sequence $\{\lambda_t \in \R^+\}_t$ in the regularized \Problem{eq:PE_optimization_3} in order to generate a sequence of estimated relative likelihood-ratios $\{f_t\}_{t \in \N}$ that converges to the target: $f_t \rightarrow r^{\alpha}$ as $\lambda_t \rightarrow 0$. 

To describe precisely the stochastic approximation strategy for the online setting, we first define the regularized instantaneous cost function 
$\ell_t^{\text{PE}}(f)$, $f \in \Hilbert$, based on \Eq{eq:PE_optimization_3}% 
%$\ell_t^{\text{PE}}(\cdot)$ based on \Eq{eq:PE_optimization_3}, which takes values $f \in \Hilbert$ in the Hilbert space 
:%
\begin{equation}
\ell_t^{\text{PE}}(f)=(1-\alpha)\frac{f^2(x_t)}{2}+\alpha \frac{f^2(x'_t)}{2}-f(x'_t)+\frac{\lambda_t}{2}\norm{f}^2_{\Hilbert}.
\end{equation}
The functional stochastic gradient $\Nabla_{f\!}(\ell_t^{\text{PE}}(f))(\cdot)$ gives the random direction of the stochastic update. Thanks to its properties listed in \Sec{sec:first_order_optimization}, we can compute it easily by:
\begin{equation}{\label{eq:PE_gradient}}
\begin{aligned}
\Nabla_{f\!}(\ell_t^{\text{PE}}(f))(\cdot) &= \frac{(1-\alpha)}{2}\Nabla_{f\!}\left(f^2(x_t)\right)(\cdot)  + \frac{\alpha}{2}\Nabla_{f\!}\left(f^2(x'_t)\right)(\cdot) -\Nabla_{f\!}(f(x'_t))(\cdot) +\lambda_t f(\cdot) \\
&=(1-\alpha) f(x_t) K(x_t,\cdot) +(\alpha f(x'_t)-1) K(x'_t,\cdot) +\lambda_t f(\cdot). 
\end{aligned}
\end{equation}
where the last equality is a consequence of \Expr{eq:gradient}.

Let us denote by $SL(\Hilbert)$ be the set of self-adjoint bounded linear operators in $\Hilbert$.
Then, we can define the random variables $A_t:\mathcal{X}\times\mathcal{X} \rightarrow SL(\Hilbert)$ and $b_t:\mathcal{X}\times\mathcal{X} \rightarrow \Hilbert$ as: 
\begin{equation}
\begin{aligned}
A_t &=(1-\alpha) \dotH{\cdot}{K(x_t,\cdot)}K(x_t,\cdot)  + \alpha \dotH{\cdot}{K(x'_t,\cdot)}K(x'_t,\cdot) + \lambda_t \idH \\
& = A(x_t,x'_t) + \lambda_t \idH
\\ 
b_t &= K(x'_t,\cdot),
\end{aligned}
\end{equation}
%% $$
where $\idH$ is the identity operator in $\Hilbert$, and 
$A:\mathcal{X}\times\mathcal{X} \rightarrow SL(\Hilbert)$ is such that when applied to $f \in \Hilbert$: %we have: 
\begin{equation}
A(x,x')f= (1-\alpha)f(x)K(x,\cdot)+ \alpha f(x')K(x',\cdot).
\end{equation}
Then, the functional stochastic gradient can be rewritten in terms of $A_t$  and $b_t$ as:
\begin{equation}
\!\!\!\Nabla_{f\!}(\ell_t^{\text{PE}}(f))(\cdot)=A(x_t,x'_t)f + \lambda_t f(\cdot) - b_t = A_t f - b_t,
\end{equation}
%%
%Then 
and the stochastic update for \Problem{eq:PE_optimization_2} becomes: %takes the form: 
\begin{equation}{\label{eq:PE_gradient_step}}
\begin{aligned}
 f_{t}(\cdot) &= f_{t-1}(\cdot)-\eta_t  \Nabla_{f\!}(\ell_t^{\text{PE}}(f_{t-1}))(\cdot)
 \\
&= (1- \eta_t \lambda_t) f_{t-1}(\cdot)  - \eta_t \large[ A(x_t,x_t)f_{t-1} - b_t\large] \\
&= (1- \eta_t \lambda_t) f_{t-1}(\cdot)  -\eta_t \large[ (1-\alpha) f_{t-1}(x_t) K(x_t,\cdot)  +  (\alpha f_{t-1}(x'_t)-1) K(x'_t,\cdot) \large],
\end{aligned}
\end{equation}
where $\eta_t>0$ is a given step-size at time $t$. We will discuss in \Sec{sec:theoretical_guarantees} which are the conditions to be satisfied by the sequence $\{ \eta_t \}_{t \in \N}$ and $\{ \lambda_t \}_{t \in \N}$ so that $f_t$ converges.

Suppose a dictionary $D_{t-1}$ made of $M_{t-1}$ basis functions, $\{x_1, ..., x_{M_{t-1}} \in \Hilbert \}$, and a kernel function $K : \X \rightarrow \R^{M_{t-1}}$ that maps input data to %$M_{t-1}$-sized 
vectors:
\begin{equation}
K(D_{t-1},\cdot) = (K(x_1,\cdot),...,K(x_{M_{t-1}},\cdot))^\top
\end{equation}
Then, if we express $f_{t-1}(\cdot)$ using $D_{t-1}$ and a weight vector $\theta_{t-1} \in \R^{M_{t-1}}$:
\begin{equation}
\!\!\!\!f_{t-1}(\cdot) = \sum_{m=1}^{M_{t-1}} \theta_{t-1,m} K(x_m,\cdot)=  K(D_{t-1},\cdot)^{\top} \theta_{t-1},\!\!
\end{equation}
we can also express the subsequent $f_{t}$ using the extended dictionary $D_{t}=D_{t-1} \cup \{K(x_t,\cdot), K(x'_t,\cdot)\}$. In that case, the new weights come from the concatenation of the previous weights and two new terms depending on  $f_{t-1}$ evaluated at $x_{t}$ and $x'_{t}$: 
\begin{equation}
\theta_{t}=[(1- \eta_t \lambda_t)\theta_{t-1},\eta_t (\alpha-1)f_{t-1}(x_t), \eta_t (1-\alpha f_{t-1}(x_t))] \in \R^{M_{t-1}+2}.
\end{equation}
The relationship between $f_{t-1}$ and $f_{t}$ implies that the cost per iteration is mainly for computing $f_{t-1}(x_t)$, which requires $2(t-1)$ kernel function evaluations. Therefore, the cost per iteration scales rate $\bigO(t)$, and that the number of kernel function evaluations up to time $t$ is $\bigO(t^2)$. A sketch of the \OLRE algorithm is shown in \Alg{alg:LRE_PEdiv}.

\begin{algorithm}[t]
\SetAlgoLined
\SetKwInOut{Input}{Input}
\SetKwInOut{Output}{Output}
\Input{$\{x_t\!\sim\! p ,x'_t\!\sim\! \q \}_{t=1,...}$: stream of observation pairs;\\% $x_t\!\sim\! p$ and $x'_t\!\sim\!q$; \\
$t_0$: size of the warm-up period;  \\
 $a \geq 4,\frac{1}{2} \leq \sourceparameter \leq 1$: fixed constants \\
$0<\alpha<1$: prefixed regularization parameter;\\
$K$: predefined kernel (form and hyperparameters).\!\!\!\\
%{\bfseries Output:} 
}
\Output{$\{f_t\}_{t=1}^{T}$:  set of estimated relative likelihood-ratios.}
\vspace{1mm}
\hrule
\vspace{1mm}

Initialize $f_0(\cdot)=0; \ D_0=[\,]; \ \theta_0=[\,]$

\For{$t=1,2,...$}
{Get the incoming \iid  pair of observations $(x_t,x'_t)$\!\!\! \\
%%%
Compute the step-size and the penalization parameter:% 
\vspace{-3mm}
\begin{equation}
\!\!\!\!\!\eta_t= a \left( \frac{1}{t_0+t} \right)^{\frac{2\sourceparameter}{2\sourceparameter+1}}\!\!\!\!, \ \ \ \ \lambda_t= \frac{1}{a} \left( \frac{1}{t_0+t} \right)^{\frac{1}{2\sourceparameter+1}} 
\end{equation}
\vspace{-2mm}

Update the dictionary: 
$D_{t}=D_{t-1} \cup \{x_t,x'_t\}$ \\

 Update the weights: \\%
 $\theta_{t}=[(1 \!-\! \eta_t \lambda_t)\theta_{t-1},\eta_t (\alpha-1)f_{t-1}(x_t), \eta_t (1\!-\!\alpha f_{t-1}(x'_t))]$\!\!\!\!\!\!\!\!\!\!\\
 Update the relative likelihood-ratio estimate: $f_t(\cdot)=K(D_{t},\cdot)^{\transpose} \theta_{t}$
}
\textbf{return} $\{f_t\}_{t=1}^{T}$
\caption{Online LRE (\OLRE)}
\label{alg:LRE_PEdiv}
%\end{adjustbox}
\end{algorithm}

\section{Theoretical guarantees}{\label{sec:theoretical_guarantees}}

%\subsection{Technical contribution compared to existing work}

Previous convergence analyses of \LRE  are restricted to the offline setting where $\Npre$ pairs of observations from $p$ and $\Npost$ $q$ are available at the time of estimation. Works such as \cite{Sugiyama2007,Nguyen2007,Nguyen2010,Yamada2011,Sugiyama2012}, capitalized over available theoretical results for $M$-estimators \citep{Van_de_Geer2000}. That framework is successfully adapted to derive convergence rates as most of the \LRE rely on a penalized cost function based on an empirical approximation of \fdivs% \cite{Nguyen2007,Sugiyama2007,Nguyen2010,Yamada2011,Sugiyama2012}
. The metrics that were used to describe the convergence of the likelihood-ratio estimates, which we will denote by $\hat{f}_{\lambda_n}$, depend on the \fdiv that is used for estimation. More precisely, it is common to define an estimator $D_{n}(\hat{f}_{\lambda_n})$ aiming to approximate the real \fdiv $\mathcal{D}_{\phi}(P \Vert Q)$ (\Eq{eq:f-divergence}) to then describe the convergence of the method via an upper-bound of the quantity $\abs{D_{n}(\hat{f}_{\lambda_n})-\mathcal{D}_{\phi}(P \Vert Q)}$. It is common as well to derive convergence rates in terms of a similarity measure between $\hat{f}_{\lambda_n}$ and the real likelihood-ratio $r$; the similarity measure is chosen as well based on the \fdiv. For example, %the works of 
\cite{Sugiyama2007} and \cite{Nguyen2007,Nguyen2010} study the \LRE problem based on the Kullback-Leibler divergence, and the convergence rates between $\hat{f}_{\lambda_n}$ and $r$ are given in terms of the Hellinger distance.

The $M$-estimation approach requires further hypotheses over the functional space $\mathcal{F}$ and the real likelihood-ratio function $r$. For example, the convergence rates depend on the complexity of $\mathcal{F}$  summarized in quantities such as covering numbers or bracketing numbers. It is common to set unrealistic assumptions over the real likelihood-ratios, such as a strictly positive lower-bound and a finite upper-bound even when $r$ is unregularized \citep{Nguyen2007,Nguyen2010,Sugiyama2012}. Moreover, although those results assume that all observations are used in the estimation process, their numerical implementations require fixing a finite-dimensional dictionary. The impact of the dictionary selection on those convergence rates has not been detailed. 

Theorems \ref{thm:convergence_results} and \ref{thm:convergence_results_2} summarize the \OLRE convergence rates in terms of the $\ltwo$ and the Hilbert norms. The theoretical approach used to produce these results differs from previous works as we deal directly with the functional optimization problem described in \Lemma{lemma:var_for} without using the empirical risk as a surrogate cost function, nor the hypothesis of a fixed number of observations (\ie fixed horizon). This implies that the proofs of both theorems (see Appendix B) no longer depend on $M$-estimation nor the required restrictive hypotheses of that approach. Instead, we employ %capitalize on the framework of 
stochastic approximation of regularized paths \citep{Tarres2014}, which deals with the online solution of a linear operator equation defined in a Hilbert space $\Hilbert$. In fact, we show in the appendix how the Pearson-based optimization of \Problem{eq:PE_optimization_2} is connected with the regression problem in $\Hilbert$ as both can be written as linear operator equations in an RKHS. This stochastic approach allows us to obtain for the first time convergence rates in terms of the Hilbert norm and 
with milder hypotheses. Furthermore, as we use all the observations in the numerical implementation, there is no gap between theory and practice regarding the convergence rates analyzed in both theorems.

\subsection{Convergence guarantees for \OLRE}
%\inlinetitle{Convergence guarantees for \OLRE}{.}~
%

\inlinetitle{Covariance operator}{.}~%
The covariance operator is a key component for studying \OLRE's convergence properties (see Appendix B). Let $\Ltwo$ be the space of square integrable functions with respect to  $p^{\alpha}$, and $\ltwo$ its quotient space, which is a Hilbert space whose norm is denoted by $\Vert{\cdot}\Vert_{\ltwo}$. Notice that if $p^{\alpha}$ has full support on $\mathcal{X}$, then we can do the usual identification of the elements of $\Ltwo$ and its equivalent classes in $\ltwo$.

Let us denote by $\CO : \ltwo \rightarrow \ltwo$ the linear operator defined by the following integral transform: 
\begin{equation}{\label{eq:covariance_operator}}
    \CO(f)(t)= \intLTWO{\!K(t,x) f(x)\,}.
\end{equation}
%%%%
The operator $\CO$  has been studied in detail in \cite{Dieuleveut2016}. $\CO$ is a bounded self-adjoint semi-definite positive operator on $\ltwo$ and it is trace-class. Furthermore, it is possible to show that there exists an orthonormal eigensystem $\{\mu_{a},\psi_{a}\}_{a \in \N}$ in $\ltwo$, where $\mu_{a}$ is a basis of $\Hilbert$, and that the eigenvalues $\{\mu_{k}\}_{k \in \N}$ are strictly positive and arranged in decreasing order (see Proposition 2.2 in \cite{Dieuleveut2017}).  The eigen-elements can be used to define the operator $\CO^{\sourceparameter }: \ltwo \rightarrow \ltwo$, for $\sourceparameter  \in \R$: 
\begin{equation}{\label{Lr}}
\CO^{\sourceparameter }\bigg( \sum_{k \in \N} c_k \psi_{k}\bigg)= \sum_{k \in \N} c_k \mu_k^\sourceparameter  \psi_{k}.
\end{equation}
The operator $\CO^{\sourceparameter}$ is relevant as it encodes how well the chosen kernel approximates the relative likelihood-ratio. More precisely, the norm $\Vert{\CO^{\sourceparameter }r^{\alpha}}\Vert_{\Hilbert}$ defines a notion of smoothness of $r^{\alpha}$ \wrt $\Hilbert$. In particular, for $\sourceparameter =\frac{1}{2}$, $\CO^{\frac{1}{2}}$ defines an isometric isomorphism of Hilbert spaces (see Proposition\,3 in \cite{Dieuleveut2016}), that is $\norm{f}_{\ltwo}= \norm{\CO^{\frac{1}{2}} f}_{\Hilbert}$.

When $\CO$ is restricted to elements $f \in \Hilbert \subset \ltwo$, we recover the covariance operator, which is known to satisfy that $ \forall f,g \in \Hilbert$, $\dotH{f}{\CO(g)}=\ExpecAlpha{f(y)g(y)}$.

\inlinetitle{Main convergence results}{.}

\begin{assumption}\label{ass:independence}
The pairs of observations $(x_t,x'_t),t=1,2,...$ are iid in time and satisfy $x_t \sim p$ and $x'_t \sim q$.
\end{assumption}%
\vspace{-0.5em}
The independence hypothesis is present in the seminal work of \cite{Nguyen2007, Nguyen2010} and in the general theoretical framework for \LRE of \cite{Sugiyama2012}. 

\begin{assumption}\label{ass:kernel_map_upperbound} 
The reproducing kernel map can be upper-bounded by a constant $C > 0$: $\sup_{x \in \mathcal{X}} \sqrt{K(x,x)} \leq C < \infty$.
\end{assumption}%
\vspace{-0.5em}
This assumption allows %us 
to bound the functions $f \in \Hilbert$ in terms of the $\Vert \cdot \Vert_\Hilbert$ %Hilbert norm
. It is satisfied by commonly used kernels, such as the Gaussian and the Laplacian kernels, and in general for any continuous $K(\cdot,\cdot)$ defined in a compact input feature space $\mathcal{X}$.  
\begin{assumption}{\label{ass:support_PA}}
$p^{\alpha}$ has full support on the feature space $\mathcal{X}$.\!\!
\end{assumption}%
\vspace{-0.5em}
This statement enhances the use of the covariance operator \citep{Dieuleveut2016} and it is an important hypothesis for the framework presented in \cite{Tarres2014}. 
\begin{assumption}{\label{ass:smoothness_t}}
  $r^{\alpha} \in L_{K}^{\sourceparameter}(\Ltwo)$  for $\frac{1}{2} \leq \sourceparameter \leq 1$.
\end{assumption}%
\vspace{-0.5em}
The parameter $\sourceparameter$ controls the smoothness of $r^\alpha$ %the relative likelihood-ratio 
%with respect to the RKHS 
in $\Hilbert$.  \Assumption{ass:smoothness_t} implies that the proposed model is well-defined, in the sense that $r^{\alpha} \in \Hilbert$, which is the usual hypothesis made in the \LRE literature \citep{Sugiyama2012}. Moreover, as $\sourceparameter$ increases, $\CO^{\sourceparameter}(\ltwo)$ defines a sequence of decreasing subspaces of $\ltwo$, \ie higher $\sourceparameter$ values assume a stronger smoothness of $r^{\alpha}$.

\Theorem{thm:convergence_results} %The following theorem 
gives \OLRE's convergence with respect to the space $\Ltwo$. The norm $\norm{f_t-r^{\alpha}}^2_{\Ltwo}$ equals to the real least-squared error $\ExpecAlpha{(f_t-r^{\alpha})^2(y)}$. Moreover, this convergence result can be easily applied to describe the convergence with respect to the excess risk $L^{\text{PE}}(f)-L^{\text{PE}}(r^{\alpha})$.

\begin{theorem}{\label{thm:convergence_results}}(\textbf{\OLRE's convergence in $\Ltwo$})
Given \Assumptions{ass:independence}-\ref{ass:smoothness_t}, $a \geq 4$ and $t_0 \geq (2+4 C^2 a)^{\frac{(2\sourceparameter +1)}{2\sourceparameter }}$. Then if the learning rate sequence is fixed as $\eta_t=a \left( \frac{1}{\bar{t}} \right)^{\frac{2\sourceparameter }{2\sourceparameter +1}}$ and $\lambda_t=\frac{1}{a} \left( \frac{1}{\bar{t}} \right)^{\frac{1}{2\beta+1}}$. Then for all $t \in \N$ and $\delta \in (0,1)$, with probability at least $1- \delta$:
%%%
\begin{equation}
\begin{aligned}
& \norm{f_t- r^{\alpha}}_{\Ltwo} \leq \frac{C_1}{\bar{t}} + \left( C_2 a^{(-\sourceparameter)} + C_3 \sqrt{a}  \log(\frac{2}{\delta}) \right) \left(  \frac{1}{\bar{t}} \right)^{\frac{\sourceparameter}{2\sourceparameter+1}} \\
& + \left( C_4 a^{\frac{5}{2}}  + C_5 a^{\frac{7}{2}}  \sqrt{log(\bar{t})}  \right) (\log^2 (\frac{2}{\delta})) \left( \frac{1}{\bar{t}} \right)^{\frac{4\sourceparameter-1}{4\sourceparameter+2}},
\end{aligned}
\end{equation}
where:
\begin{equation*}
\begin{aligned}
C_1=\frac{2 t_0}{\alpha}, C_2 = \frac{5\sourceparameter+1}{\sourceparameter(1+\sourceparameter)} \norm{L_K^{(-\sourceparameter)}r^{\alpha}}_{\Ltwo},C_3= \frac{16 C}{\alpha},   C_4 = \frac{32 C^3}{\alpha},C_5 = \frac{8 C^3(10C+3)}{\alpha}.
\end{aligned}
\end{equation*}
%%%
\end{theorem}%
\vspace{-0.5em}
Notice that the convergence rate in $\Ltwo$ can be decomposed into three terms. The first depends on the initialization, and decreases at rate $\bigO(t^{-1})$. The second one is related to the smoothness of the likelihood-ratio in $\Hilbert$ and the noise in the observations, and decreases at rate $\bigO(t^{-\frac{2}{2\sourceparameter+1}})$. The third term is related to the variance of the observations, and decreases at a rate $\bigO(\log^{\frac{1}{2}}(t) t^{- \frac{4\sourceparameter-1}{4\sourceparameter+2}})$. When $\sourceparameter \in (\frac{1}{2}, 1]$, the second term becomes dominant, which implies %we expect 
a faster convergence as the smoothness of $r^{\alpha}$ increases. When $\sourceparameter=\frac{1}{2}$, the convergence rate becomes $\bigO(\log^{\frac{1}{2}}(t) t^{- \frac{1}{4}})$.  

The convergence rate with respect to $\Hilbert$ is more restrictive than in $\Ltwo$. For $\Hilbert$, \Assumption{ass:smoothness_t} needs to be replaced by \Assumption{ass:smoothness_t_2}; the main difference is that $r^{\alpha}$ is required to be smoother with respect to $\Hilbert$ %as we ask 
for higher $\sourceparameter$ values. 
\begin{assumption}{\label{ass:smoothness_t_2}}
  $r^{\alpha} \in L_{K}^{\sourceparameter}(\Ltwo)$  for $\frac{1}{2} <\sourceparameter \leq \frac{3}{2}$.
\end{assumption}
\begin{theorem}{\label{thm:convergence_results_2}}(\textbf{\OLRE's convergence in $\Hilbert$})   Given \Assumptions{ass:independence}-\ref{ass:support_PA} and \ref{ass:smoothness_t_2}, $a \geq 4$ and $t_0 \geq (a C^2 + 1 )^{\frac{(2\sourceparameter+1)}{2\sourceparameter}}$. Then if the learning rate sequence is fixed as $\eta_t=a \left( \frac{1}{t+t_0} \right)^{\frac{2 \sourceparameter}{2\sourceparameter+1}}$ and $\lambda_t=\frac{1}{a} \left( \frac{1}{t+t_0} \right)^{\frac{1}{2\sourceparameter+1}}$. Then for all $t \in \N$ and $\delta \in (0,1)$, with probability at least $1- \delta$:
%%%
\vspace{-0.5em}
\begin{equation}
\begin{aligned}
& \norm{f_t- r^{\alpha}}_{\Hilbert} \leq \frac{C'_1}{\bar{t}} + \left(\! C'_2 a^{\frac{1}{2}-\sourceparameter}  + C'_3 a \log \left( \frac{2}{\delta} \right)\!\! \right)  \left( \frac{1}{\bar{t}}\right)^{\frac{2\sourceparameter-1}{4\sourceparameter+2}}\!\!\!\!,
\end{aligned}
\end{equation}
where $\bar{t}=t+t_0$ and,
\begin{equation*}
\begin{aligned}
C'_1=\frac{2 \sqrt{a} t_0^{\frac{4\sourceparameter+1}{4\sourceparameter+2}}}{\alpha}, C'_2= \frac{20\sourceparameter-2}{(2\sourceparameter-1)(2\sourceparameter+3)} \norm{L_K^{(-\sourceparameter)}r^{\alpha}}_{\Ltwo}  C'_3= 6 \left(\frac{(C+1)^2}{C \alpha}\right) .
\end{aligned}
\end{equation*}
\end{theorem}%
We can see that the upper-bound appearing in \Theorem{thm:convergence_results_2} is made of two components. The first component is related to the constant $C'_1$ and summarizes the impact of the initialization. This term converges at rate $\bigO(t^{-1}))$. The second term, which is the leading term of the expression, converges at rate $\bigO(t^{-\frac{2\sourceparameter-1}{4\sourceparameter+2}})$ and it mainly depends on the smoothness parameter $\sourceparameter$. The bigger $\sourceparameter$
, the faster the convergence. Notice that the case $\sourceparameter=\frac{1}{2}$ is not considered in this theorem, in fact, the algorithm may not converge in $\Hilbert$, as indicated by Theorem A in \cite{Tarres2014}.

Both Theorems $\ref{thm:convergence_results_2}$ and $\ref{thm:convergence_results}$ provide useful information on how to fix the step sizes $\{\eta_t\}_{t \in \N}$ and regularization constants $\{\lambda_t\}_{t \in \N}$, and explain their its impact to the convergence rates. Notice that there is an interplay between the selection of $a$ and the smoothness parameter $\sourceparameter$.  The results suggest that \OLRE converges faster in $\Ltwo$ than in $\Hilbert$ if the hyperparameters are the same. Both results shed light on the impact of the parameter $\alpha$, as values close to $1$ will lead to better convergence rates. Nevertheless, $\alpha=1$ render $r^{\alpha}$ a constant, which is meaningless for most applications. For this reason, the value of $\alpha$ should take into account both the convergence rate of the optimization schema and the intended application. 

%Previous 
Convergence results for likelihood-ratio estimates based on the Pearson-divergence can be found in \cite{Yamada2011}. %, where the method RuLSIF is proposed.  
Those results are given in terms of the difference between the real Pearson-divergence $L^{\text{PE}}(r^{\alpha})$ and an empirical approximation $L^{\text{PE}}_n(\hat{f}_{\lambda_n})$. It was shown that if the regularization constant decreases at speed $\lambda_n=\bigO(n^{-\frac{2}{2+\gamma}})$, where the parameter %$0\!<\!\gamma\!<\!2$ 
$\gamma \in (0,2)$ quantifies the complexity of %the functional space 
$\Hilbert$, then RULSIF could achieve a convergence rate $L^{\text{PE}}(r^{\alpha})-L^{\text{PE}}_n(\hat{f}_{\lambda}) \leq \bigO(n^{-\frac{1}{2+\gamma}})$ with high probability. 
 
\begin{figure*}[t]
\begin{minipage}[t]{\linewidth}
\ \quad \colorbox{gray!15}{ \  \qquad \qquad \qquad \small \ \ \ \ Setup \ \ \ \ \qquad \qquad \qquad \ \ \ \ \ }\qquad %
\colorbox{gray!15}{ \ \qquad \qquad  \quad \quad \ \small Results\phantom{p} \ \qquad \qquad \qquad \quad }\\%
\end{minipage}\\%
%\vspace{-2em}
{\centering%
\begin{minipage}[t]{0.03\linewidth}
	\rotatebox{90}{\ \ \ \colorbox{gray!15}{\ \ \small \phantom{I}Experiment I\phantom{I}\ \ }}
\end{minipage}
\hspace{-0.5em}
\includegraphics[width=0.25\linewidth, viewport=0 -15 730 640, clip]{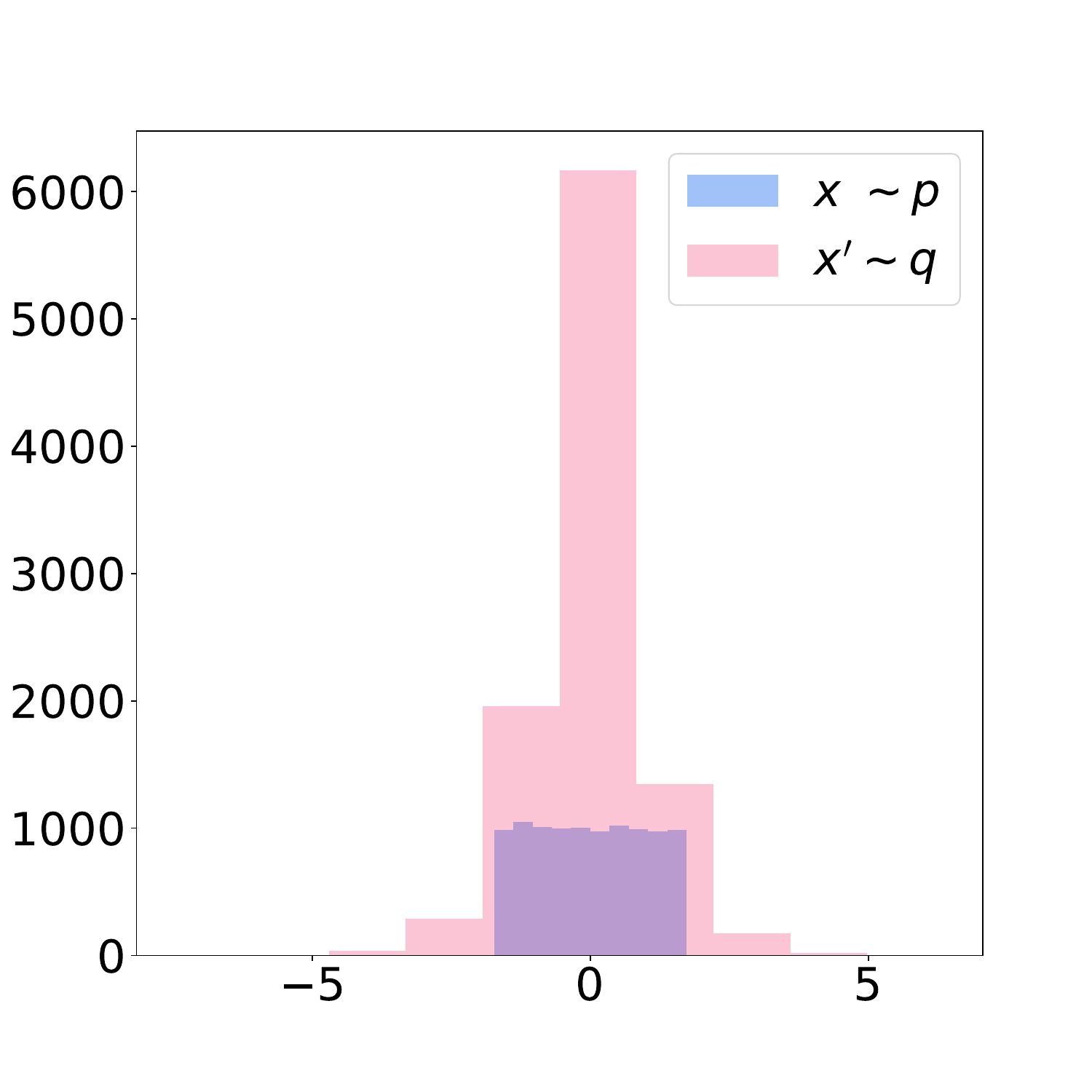}
\includegraphics[width=0.25\linewidth, viewport=0 -7 365 318, clip]{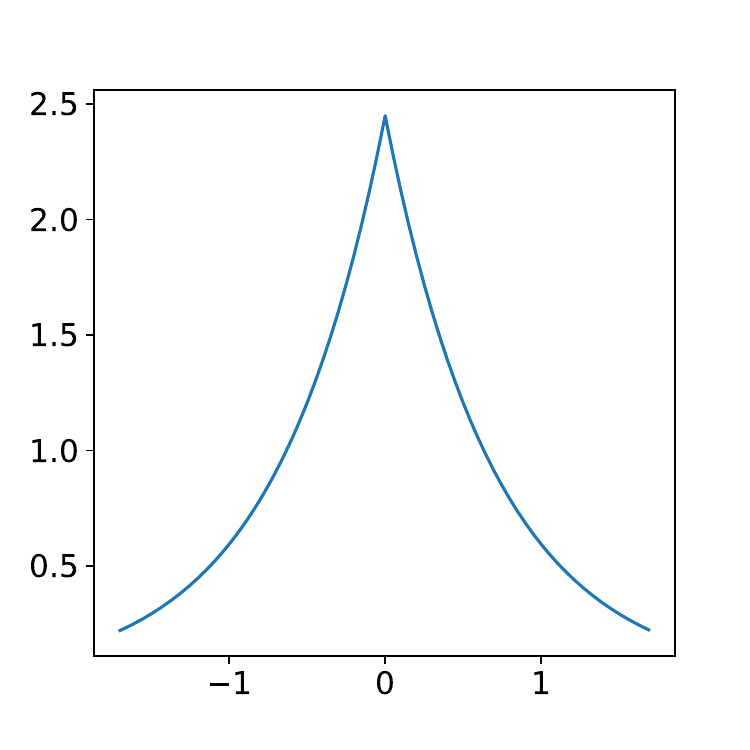}
{\centering
\includegraphics[width=0.30\linewidth, viewport=20 -35 1300 955, clip]{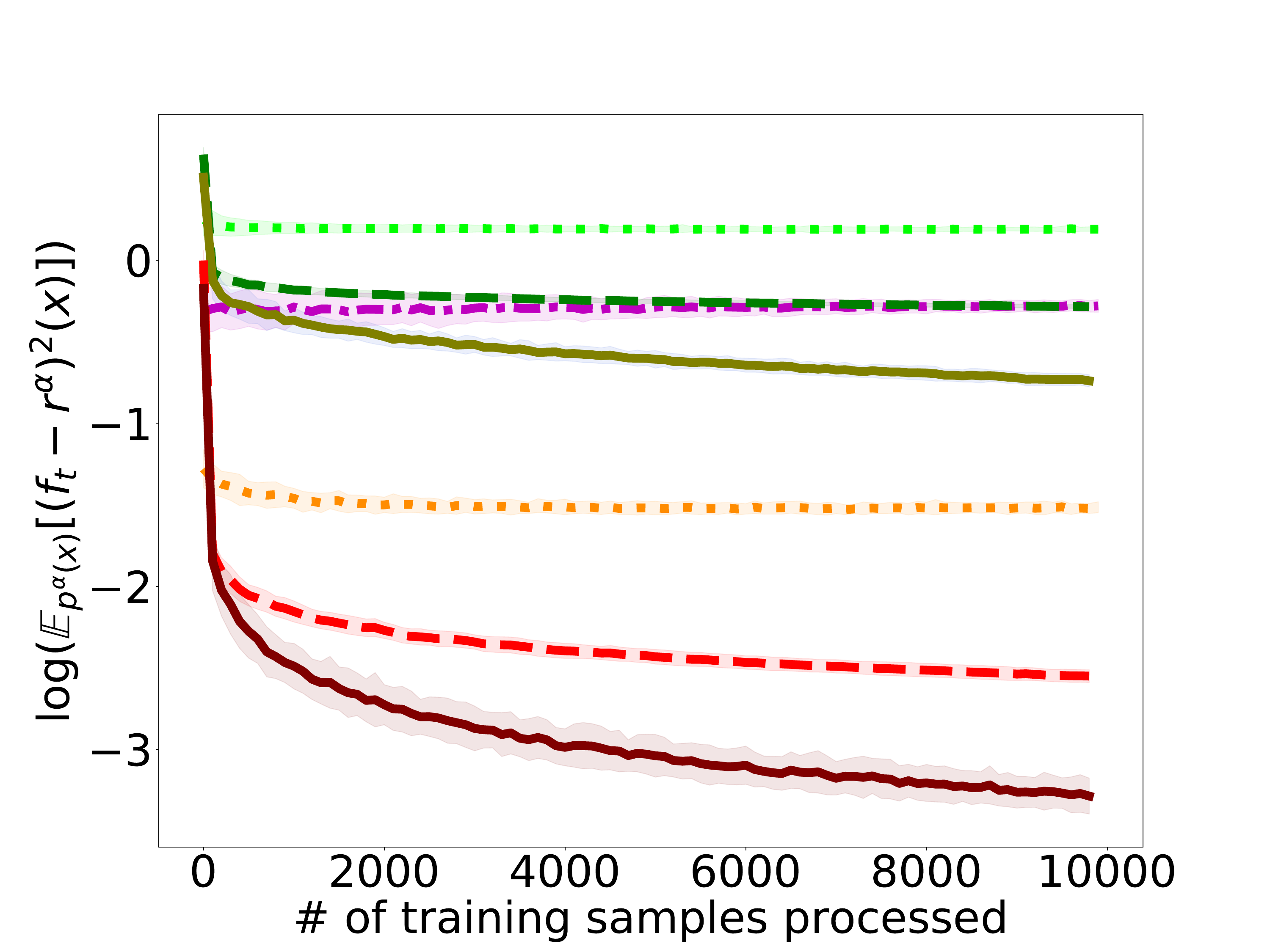}}
\begin{minipage}[t]{0.16\linewidth}
\vspace{-8em}
\includegraphics[width=\linewidth]{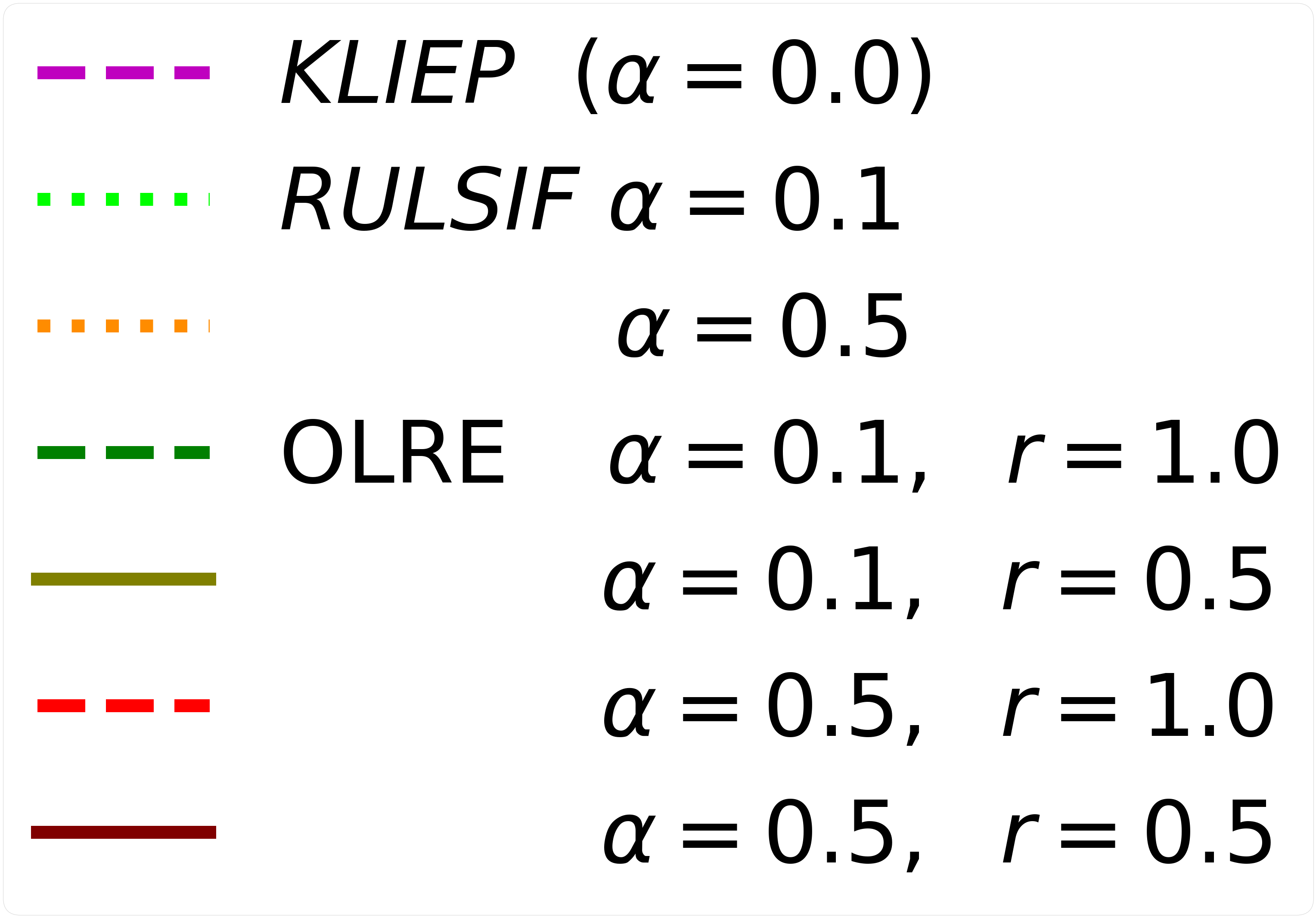}
\end{minipage}%
\!\!\!\!
\\
\centering
\begin{minipage}[t]{0.03\linewidth}
	\rotatebox{90}{\ \ \ \colorbox{gray!15}{\ \ \small \phantom{I}Experiment II\ \ }}
\end{minipage}
\hspace{-0.05em}
\includegraphics[width=0.25\linewidth, viewport=0 -15 730 640, clip]{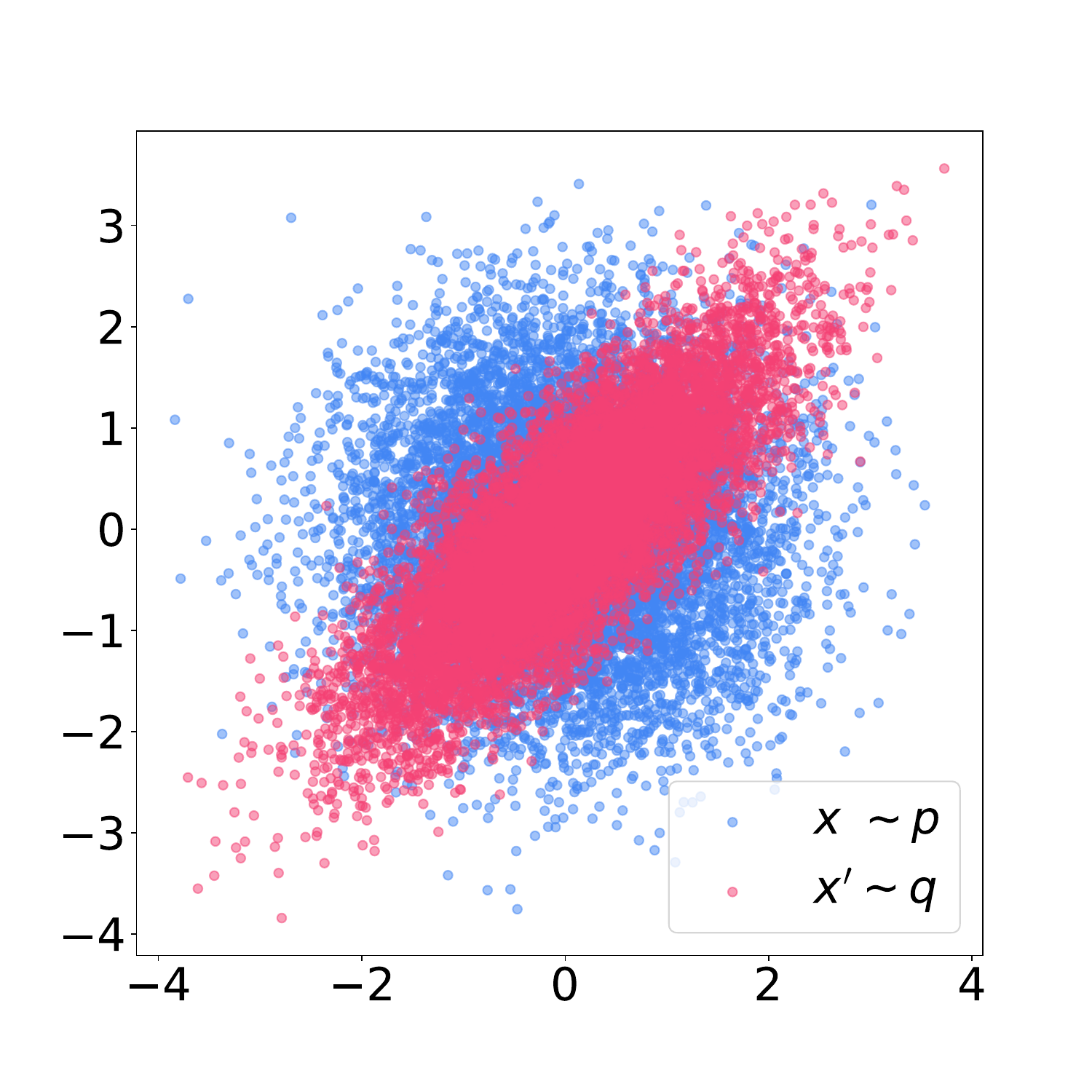}
\includegraphics[width=0.25\linewidth, viewport=0 -15 730 640, clip]{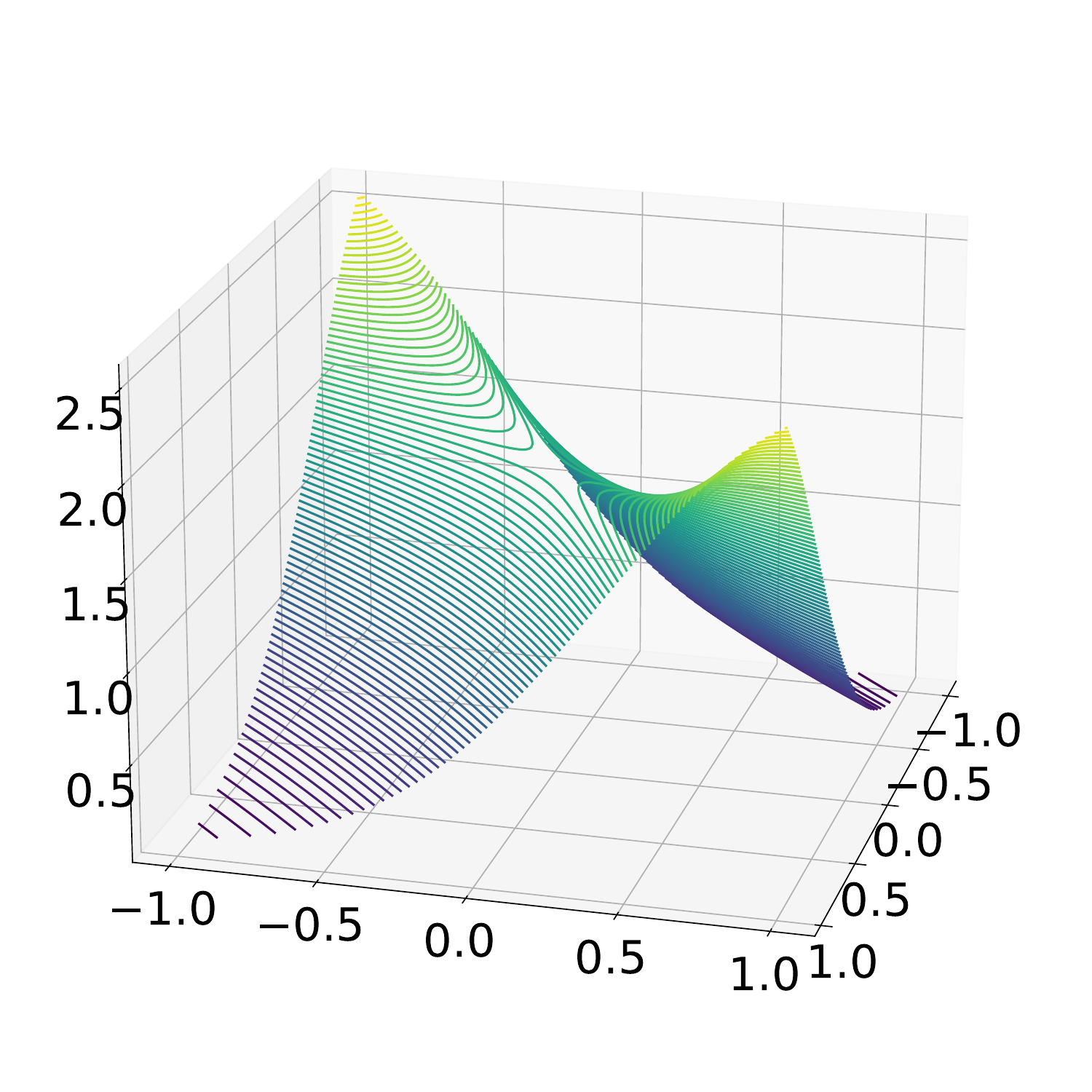}
\includegraphics[width=0.30\linewidth, viewport=20 -35 1300 955, clip]{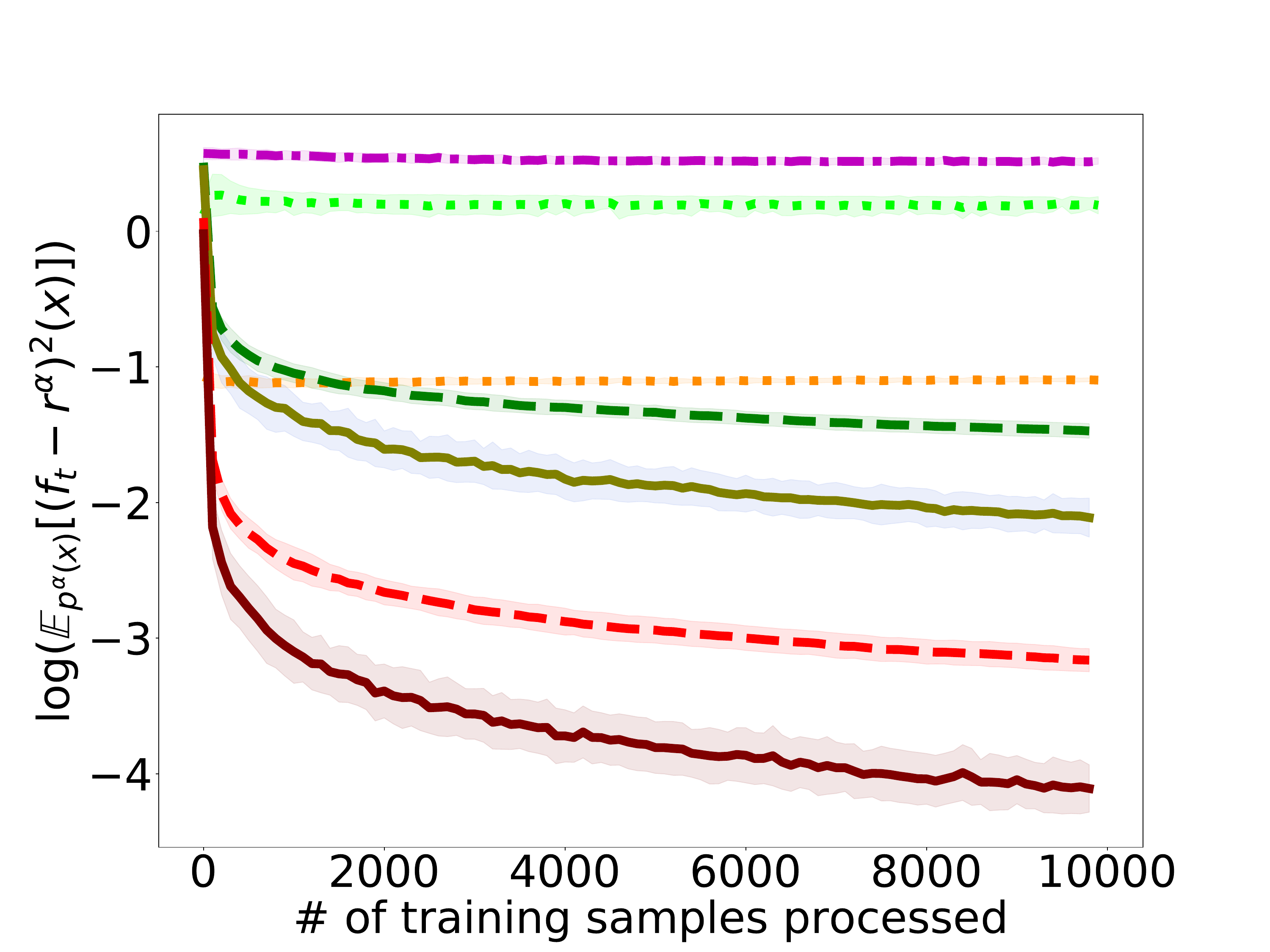}
\hspace{5.0em}
\begin{minipage}[t]{0.16\linewidth}
\end{minipage}
\\
\begin{minipage}[t]{0.03\linewidth}
	\rotatebox{90}{\ \ \ \colorbox{gray!15}{\ \ \small \phantom{I}Experiment III\ \ }}
\end{minipage}
\hspace{-0.05em}
\includegraphics[width=0.25\linewidth, viewport=0 -15 730 640, clip]{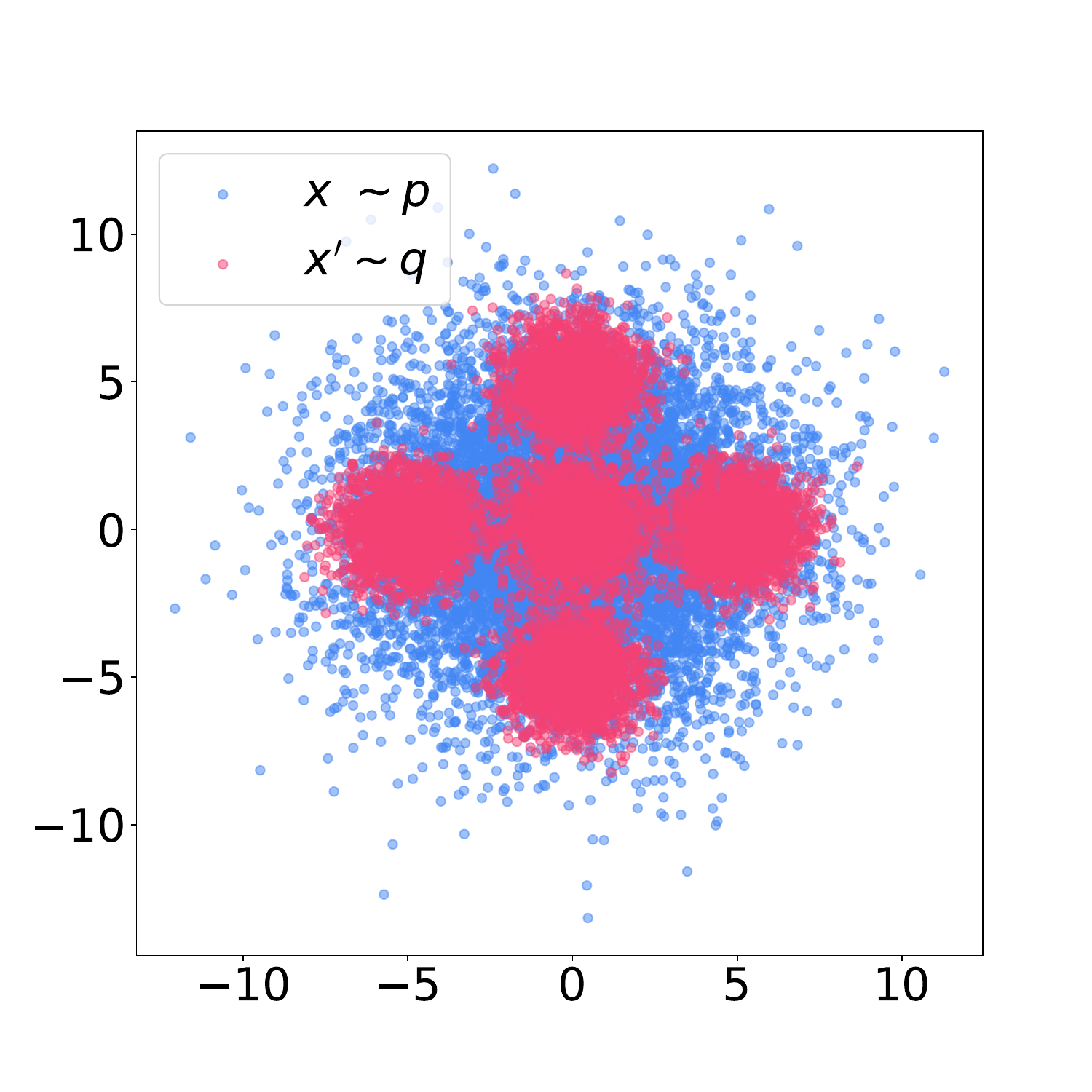}
\includegraphics[width=0.25\linewidth, viewport=0 -15 730 640, clip]{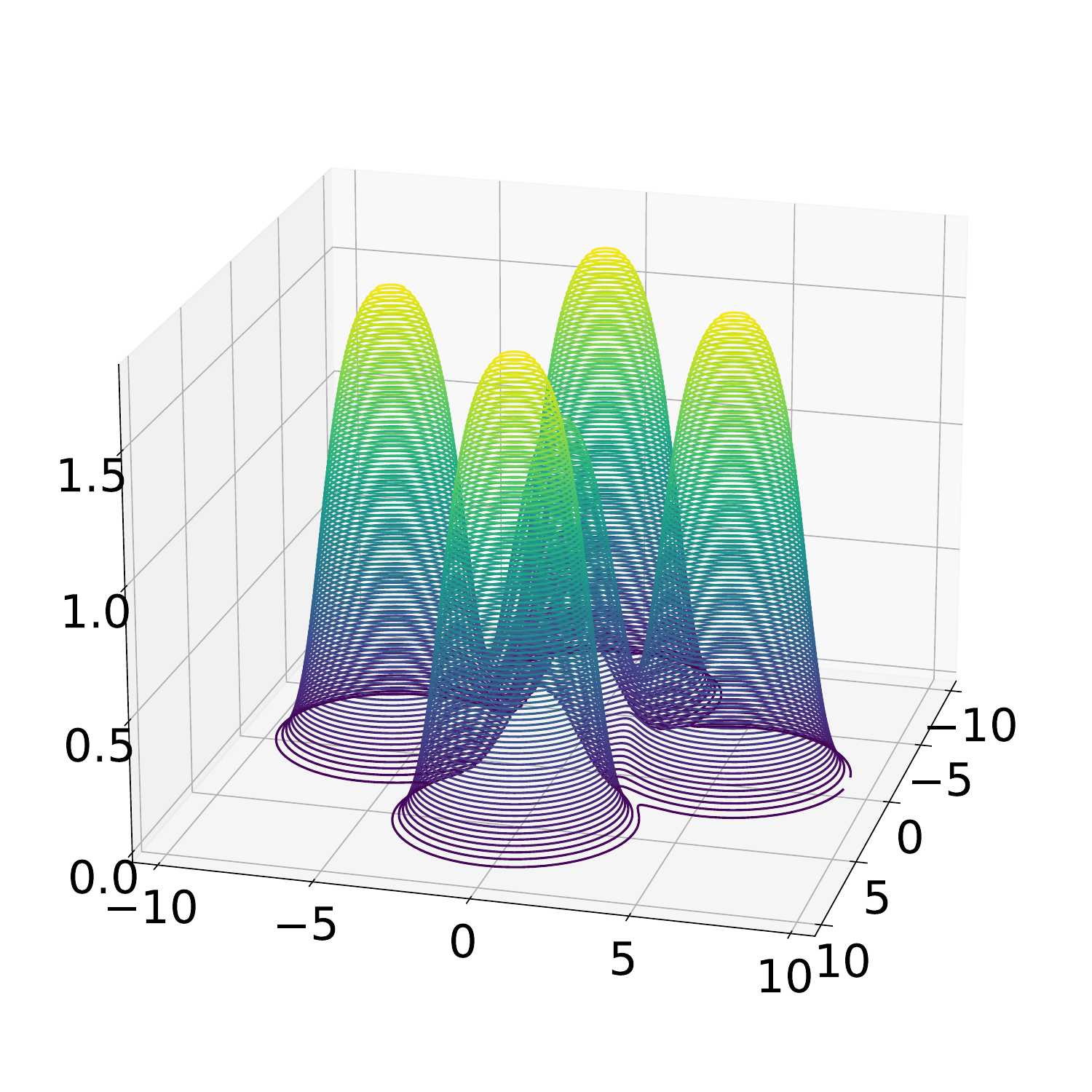}
\includegraphics[width=0.30\linewidth, viewport=20 -35 1300 955, clip]{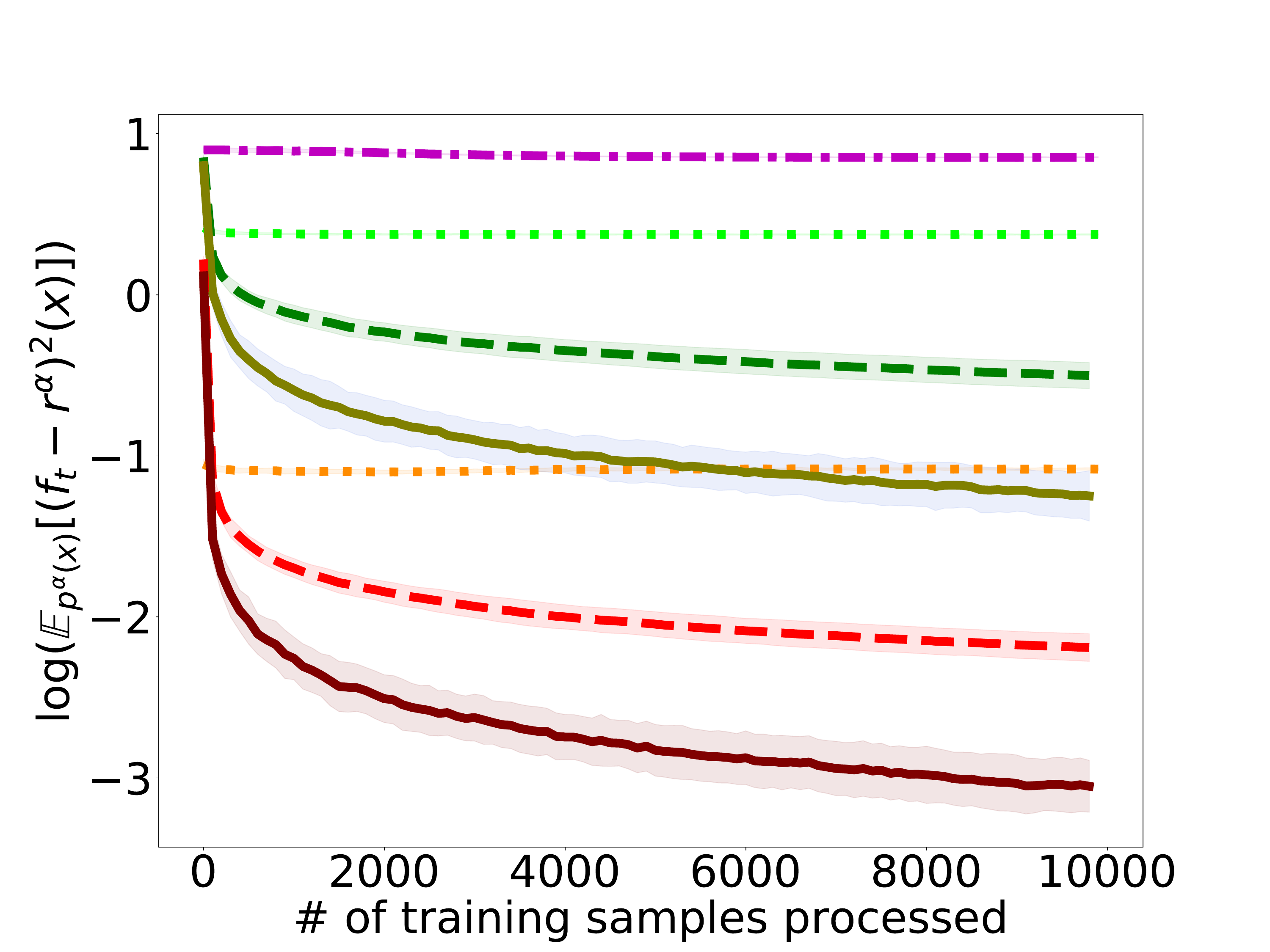}
\hspace{5.0em}
\begin{minipage}[t]{0.16\linewidth}
\end{minipage}%
\caption{ \footnotesize{Each row presents results for one of the synthetic scenarios described in \Sec{sec:experiments}. The first column shows the generated samples from $p$ and $q$, and the second column illustrates the real relative likelihood-ratio $r^{\alpha}$. The third column compares the different algorithms in terms of the expected $\Ltwo$-distance of the likelihood-ratio estimates $f_t$ and the real  $r^{\alpha}$, as a function of the number of pairs of observations processed. The expected value $\ExpecAlpha{(f_t-r^{\alpha})^2}$ is computed by averaging over $10,000$ independent pairs of observations that were not used during the training phase. The empirical convergence curve is the average of $100$ experiment instances, and the error-bar indicates $1$ standard deviation around the average performance. A safer comparison between two methods can be made when they use the same $\alpha$-regularization, hence they optimize the same target likelihood-ratio functional.}
}\label{fig:results}%
%\vspace{-2.1em}%
}
\end{figure*}

\section{Experiments}{\label{sec:experiments}}%
In this section, we carry out synthetic experiments to evaluate the performance of the proposed \OLRE (\Alg{alg:LRE_PEdiv}), as well as its sensitivity to %the selection of 
its hyperparameters. %
We compare \OLRE variants %of our online method 
against two existing offline approaches, more precisely RULSIF \citep{Yamada2011} and KLIEP \citep{Sugiyama2007}. RULSIF is based on the \PEdiv; it drops the requirement for positiveness of the likelihood-ratio estimates in favor of computational efficiency (see also \Sec{sec:RULSIF}). On the other hand, KLIEP is based on the \KLdiv, and does not use the $\alpha$-regularization (equivalent setting $\alpha=0$ in \Eq{eq:likelihood-ratio}). For both methods, we follow the recommendation to select a random subset of basis functions to reduce their computational complexity; \eg  
%For example, 
\cite{Sugiyama2007} take $100$ basis functions associated to observations coming from 
$q$.

An important component of \OLRE is the choice of the kernel function and its hyperparameters. We choose a Gaussian kernel, but other options are possible as mentioned in \Sec{sec:theoretical_guarantees}. To tune the kernel hyperparameters we perform cross-validation over the first $n=100$ observations using RULSIF, which, as mentioned, has a closed-form and therefore allows for fast model selection. %
Following the results of \Theorem{thm:convergence_results}, we let the learning rate and the penalization rate to depend on the smoothness of the parameters $a$, $t_0$, and $\sourceparameter$. We fix $a$ at the lower-bound provided by \Theorem{thm:convergence_results}, that is $a=4$ and the lower-bound for $t_0$ is fixed as $100$, which is equal to the number of observations used for identifying the hyperparameters at the beginning of the procedure. The user needs to provide only two parameters, $\alpha$ and $\sourceparameter$, %that cannot be known in advance 
which, according to \Theorem{thm:convergence_results}, play an important role in \OLRE's convergence. We report results with different values in order to show the sensibility of our approach. 

We run experiments that approximate the likelihood-ratio between two \pdfs $p$ and $q$, using three setups:
\begin{itemize}[leftmargin=1em,topsep=1em,itemsep=0px,noitemsep]
\item \emph{Experiment I: }
$p$ is a uniform continuous distribution with zero mean and unit variance ($p = \mathcal{U}(-\sqrt{3},\sqrt{3})$); $q$ is a Laplace distribution with zero mean and unit variance.
\item \emph{Experiment II: }
$p$ is a bivariate Gaussian distribution with zero mean, % vector, 
and a covariance matrix equal to the identity matrix ($p = \mathcal{N}(\zero_{2\times 1}, \id_{2\times2})$); $\q$ is a bivariate Gaussian distribution with zero mean and covariance matrix such that $\Sigma_{1,1}=\Sigma_{2,2}=1$, $\Sigma_{1,2}=\frac{4}{5}$ ($\q = \mathcal{N}(\zero_{2\times 1},\Sigma)$).
\item \emph{Experiment III: } $p$ is bivariate Gaussian distribution with mean vector $\mu$ and covariance matrix $\Sigma_1=10 \times \id_{2\times 2}$ ($p=\mathcal{N}(\zeros{2\times 1},\Sigma_1)$), and $q$ is a mixture of five bivariate Gaussian distributions with the same covariance matrix $\Sigma_2=5 \times \id_{2\times 2}$ and $\mu$ vectors: $\mu_1=(0,0),\mu_2=(0,5),\mu_3=(0,-5),\mu_4=(5,0),\mu_5=(-5,0)$, each of them with the same proportion.
\end{itemize}
We compare the algorithms in approximating the relative likelihood-ratio $r^{\alpha}$ with respect to the norm $\Vert{\cdot}\Vert_{\Ltwo}$, which is the real least-squared error $\ExpecAlpha{(f_t-r^{\alpha})^2}$, a quantity that we approximate by averaging over $10,000$ testing pairs of observations that were not used during training. For the offline setting, $f_t$ stands for the estimated likelihood-ratio computed at each time from scratch by minimizing an empirical risk with respect to the first $t$ pairs of training observations. For \OLRE, $f_t$ is the approximation to the solution of the functional minimization \Problem{eq:PE_optimization_2} found via the functional stochastic gradient descent. 

\Fig{fig:results} reports our results that carry clear messages. The first thing to notice is that \OLRE achieves substantially faster convergence rates when compared with available offline methods (when comparing for the same $\alpha$ value). Furthermore, we can see how KLIEP's and RULSIF's strategy of selecting a random dictionary introduces a bias to their performance over time. %We can notice that 
\OLRE's behavior with respect to the hyperparameter $\alpha$ is well described by \Theorem{thm:convergence_results}. Higher values of $\alpha$ lead to faster convergence. The value of $\sourceparameter$ also impacts the performance of \OLRE. Recall that a higher $\sourceparameter$ value implies we assume $r^{\alpha}$ to be smoother with respect to the RKHS $\Hilbert$. \Fig{fig:results} suggest also that higher values of $\sourceparameter$ lead to a lower variance, but also a higher bias.
\section{Conclusions and further work}%
To the best of our knowledge, this is the first work to introduce and addresse the problem of online likelihood-ratio estimation (\OLRE). %More specifically, 
We presented the homonymous non-parametric framework %for \OLRE 
that processes a stream of pairs of observations coming from two \pdfs. Our approach leads to an easy implementation that, contrary to the existing methods for the offline setting, does not require knowing the length of the stream in advance. Moreover, our theoretical results shed light on the limitations of previous convergence analyses and may motivate further work on studying the \LRE problem with techniques used in functional optimization that can optimize directly the real risk.

\section{Acknowledgments}\label{sec:acknowledgments}

This work was supported by the Industrial Data Analytics and Machine Learning (IdAML) Chair hosted at ENS Paris-Saclay, University Paris-Saclay, and grants from Région Ile-de-France.

%The convergence guarantees and the experimental results characterize the sensibility of \OLRE to the hyperparameters.

%in terms of 
%the regularization parameter $\alpha$, which is used to avoid an ill-defined likelihood-ratio problem, and $\sourceparameter$ which encodes the smoothness of $r^{\alpha}$ with respect to the Hilbert space%. 

% The natural next step is to employ our methodology in online learning or detection problems, where the estimation of the likelihood-ratios is an important component. 
%Another interesting direction is to adapt stochastic approximation of functional problems to other \fdivs. 

{\small
\bibliography{References}
}

\newpage

\appendix

\section{Comparison with RULSIF}{\label{sec:RULSIF}}

RULSIF is a popular offline \LRE algorithm  \citep{Yamada2011,Liu2013} aiming to solve \Problem{eq:PE_optimization_2}. The user has access to $\setpreunivariate=\{x_t\}_{t=1}^{\Npre}$ and $\setpostunivariate=\{x_t\}_{t=1}^{\Npost}$ so \label{eq: PE_optimization_2} and the relative likelihood-ratio $r^\alpha$ is approximated by minimizing the  empirical expectation of the loss. The authors assume that $r^{\alpha}$ can be approximated by a finite linear combination of the $M$ elements of a given fixed dictionary $D_M=\{K(x_m,\cdot)\}_{m=1}^{M}$, meaning the approximation $\hat{f}_{\lambda}$ should belong to $S = \operatorname{span}(\{K(x_m,\cdot)|x_m \in D_M\})$ and can take the form $\hat{f}_{\lambda}=K(D_{M},\cdot)^{\top}\hat{\theta}$.
These assumptions lead to the optimization problem:  
\begin{equation}{\label{eq:problem_Rulsif_original}}
\begin{aligned}
%\begin{subequations}{\label{eq:problem_Rulsif}}
%\begin{align}
& \!\!\!\! \hat{f}_\lambda = \argmin_{f \in S} \frac{(1-\alpha)}{2} \sum_{x \in \setpreunivariate} \frac{f^2(x)}{\Npre} + \frac{\alpha}{2} \sum_{x' \in \setpostunivariate} \frac{f^2(x')}{\Npost}  - \sum_{x' \in \setpostunivariate}  \frac{f(x')}{\Npost} + \frac{\lambda}{2} \theta^{\top} \theta,\!\!\!\!\!\!
%\\
%& f_\lambda = K(D_{M},\cdot)^{\top}\hat{\theta}\ \ \min_{\theta \in \R^M} \frac{\theta^{\top} H \theta}{2} - \theta^{\top} h + \frac{\lambda}{2} \theta^{\top} \theta,
%\end{align}
%\end{subequations}
\end{aligned}
\end{equation}
where $\lambda>0$ is a fixed regularization constant. Since $\hat{f}_{\lambda}=K(D_{M},\cdot)^{\top}\hat{\theta}$, the above is equivalent to optimizing over $\theta$:%
\begin{equation}{\label{eq:problem_Rulsif}}
\begin{aligned}
\hat{\theta} & = \argmin_{\theta \in \R^M} \frac{\theta^{\top} H \theta}{2} - \theta^{\top} h + \frac{\lambda}{2} \theta^{\top} \theta = (H+\lambda \id_{M})^{-1}h,
\end{aligned}
\end{equation}
%%
%where
%\begin{equation}
%\begin{aligned}
%\!\!\!\!\!\!\!\!\!\!\!\!\text{where\ \ } H &= \sum_{x \in \setpreunivariate} \frac{(1-\alpha) K(D_M,x)K(D_M,x)^{\top}}{\Npre} \\
%& \ \ \ \ \ +\sum_{x' \in \setpreunivariate} \frac{  \alpha K(D_M,x')K(D_M,x')^{\top}}{\Npost}, \\ 
%h &=\sum_{x' \in \setpostunivariate} \frac{K(D_M,x')}{\Npost}.
%\end{aligned}
%\end{equation}
where: 

\begin{equation}
\begin{aligned}
H &= \sum_{x \in \setpreunivariate} \frac{(1-\alpha) K(D_M,x)K(D_M,x)^{\top}}{\Npre}  + \sum_{x' \in \setpostunivariate} \frac{\alpha K(D_M,x')K(D_M,x')^{\top}}{\Npost}, \\
h &= \sum_{x' \in \setpostunivariate} \frac{K(D_M,x')}{\Npost}. 
\end{aligned}
\end{equation}

The closed form solution of \Problem{eq:problem_Rulsif} implies that the final cost 
requires $M \times (\Npre+\Npost)$ kernel evaluations to estimate $H$ and $h$, and solving a linear function of cost $\bigO(M^3)$. In the case where all the observations of $\setpostunivariate$ are included in $D_M$, RULSIF scales in $\bigO(\Npost(\Npre+\Npost))$ with respect to kernel evaluations, and in $\bigO(\Npost^3)$ for the matrix inversion. This means that, for cases where $\Npre$ is large, reducing the dimension of the problem is needed. It is commonly suggested in the literature \cite{Sugiyama2012}, and has been implemented in practice\footnote{\scriptsize RULSIF:~\url{https://riken-yamada.github.io/RuLSIF}}, to reduce $M$ by sampling uniformly at random from $\setpostunivariate$. However, as we will see in the experiments, this strategy along with the choice of a penalization defined in terms of the Euclidean norm $\Vert{\hat{\theta}}\Vert_{2}$ (instead of the Hilbert norm $\Vert{\hat{f}_{\lambda}}\Vert_{\Hilbert}$) leads to an approximation error that does not disappear as $\Nref$ and $\Npost$ increase.

\section{Technical results}{\label{app:technical_results}}

This section contains the technical details of the results presented in \Sec{sec:theoretical_guarantees}. First, we introduce the necessary elements to define a linear operator equation in a Hilbert space, and then we  present the  regularized paths framework proposed to \cite{Tarres2014} aiming to solve this kind of problem. After that,  we detail how the  likelihood-ratio estimation problem can be reformulated as a linear operator equation and its similarities with the regression problem in Hilbert spaces. Finally, we provide detailed proofs of \Theorem{thm:convergence_results} and \ref{thm:convergence_results_2}.

\subsection{Sequential stochastic approximations of regularization paths in Hilbert spaces}{\label{sec:SA_RP_H}}

\cite{Tarres2014} considers the general case of minimizing a quadratic map defined over elements of a Hilbert space via stochastic approximation. Let us begin by denoting $SL(\Hilbert)$ as the vector space of self-adjoint bounded linear operators on $\Hilbert$ endowed with the canonical norm:
\begin{equation*}
    \norm{A}=\sup_{\norm{f}_{\Hilbert} \leq 1} \norm{Af}_{\Hilbert}, \ \ \ \ A \in SL(\Hilbert).
\end{equation*}
Notice that we have used the convention that $Af$ denotes 
linear operator $A \in SL(\Hilbert)$  applied to $f \in \Hilbert$. We will keep this notation for the rest of the section. 

Let us denote by $\mathcal{X}$ and $\mathcal{Y}$ two topological spaces and define $\mathcal{Z}= \mathcal{X}\times\mathcal{Y}$ as their Cartesian product. We define a probability measure 
$\rho$ on the Borel $\sigma$- algebra of $\mathcal{Z}$. Let  $A: \mathcal{Z} \rightarrow SL(\Hilbert)$ and $b:\mathcal{Z} \rightarrow \Hilbert$ be two random variables defined in terms of the space $\mathcal{Z}$ whose expected values are denoted by: 
\begin{equation*}
\begin{aligned}
\mathbf{A}= \mathbb{E}_{\rho}[A], \ \  \ \  \mathbf{b}= \mathbb{E}_{\rho}[b].
\end{aligned}
\end{equation*}
The goal in \cite{Tarres2014} is to solve find $\mathbf{w} \in \Hilbert$ solving the  linear operator equation:
\begin{equation*}
\begin{aligned}
\mathbf{A}\mathbf{w}=\mathbf{b}.
\end{aligned}
\end{equation*}
where $\mathbf{A}$ and $\mathbf{b}$ are known and $\mathbf{A}$ is  a strictly positive operator with an unbounded inverse. 

Alternatively, $\mathbf{w}$ can be defined as the solution to the quadratic optimization problem:
\begin{equation}{\label{eq:quadratic_map_minization}}
   \argmin_{f \in \Hilbert} Q(f)=
  \argmin_{f \in \Hilbert}  \frac{1}{2} \dotH{\mathbf{A}(f-\mathbf{w})}{(f-\mathbf{w})} 
\end{equation}
%
%\Problem{eq:quadratic_map_minization} can be ill-posed when $\mathbf{A}^{(-1)}$ is an unbounded operator.
%
The stochastic approximation approach proposed by \cite{Tarres2014} consists on defining a sequence of random variables $\{A_t\}_{t \in \N}$ and $\{b_t\}_{t \in \N}$ depending on incoming observations $z_t=(x_t,y_t)$ such that the sequence $\{f_t\}_{t \in \N}$ generated by the iterative algorithm: 
\begin{equation}{\label{eq:general_gradient}}
    f_t(\cdot)=f_{t-1}(\cdot) - \eta_t \left( A_t(z_t) f_{t-1}(\cdot) - b_t(z_t)(\cdot) \right),
\end{equation}
will converge toward the solution of \Problem{eq:quadratic_map_minization}. 

\cite{Tarres2014} study the required conditions to guarantee the convergence of \Eq{eq:general_gradient} with respect to the norms $\norm{}_{\ltwo}$ and $\norm{\cdot}_{\Hilbert}$. Among these conditions, the authors assume  the random variables $\{A_t\}_{t \in \N}$ and $\{b_t\}_{t \in \N}$ are such that their expected values  $\mathbf{A}_t= \mathbb{E}_{\rho}[A_t]$ and $\mathbf{b}_t= \mathbb{E}_{\rho}[b_t]$ satisfy  $\mathbf{A}_t \rightarrow \mathbf{A} $ and $\mathbf{b}_t \rightarrow \mathbf{b}$, as $t\rightarrow \infty$ and  each of the elements of the sequence $\{\mathbf{A}_t\}_{t \in \N}$ has a bounded inverse. Finally, the authors provide the required decreasing rate for the sequence of step-sizes $\{\eta_t\}_{t\in \N}$. 

\subsection{Application to the OLRE problem}
The Online \LRE %Likelihood-Ratio Estimation problem 
described in \Sec{sec:framework} can be written in terms of the framework introduced in \cite{Tarres2014}. In this context, $\mathcal{Z}=\mathcal{X}\times\mathcal{X}=\mathcal{X}^2$ and the associated probability measure is given by the joint \pdf $\rho$ with marginal \pdfs $p$ and $q$. The incoming data observations $z_t=(x_t,x'_t)$ are \iid pairs such that $x_t \sim p$ and $x'_t \sim q$.

The random variables $A:\mathcal{Z} \rightarrow SL(\Hilbert) $ and $b:\mathcal{Z} \rightarrow  \Hilbert$ are defined based on the functional stochastic gradient: 
\begin{equation}
A(x,x')=(1-\alpha)\dotH{\cdot}{K(x,\cdot)}K(x,\cdot)+\alpha \dotH{\cdot}{K(x',\cdot)}K(x',\cdot) \ \ \ \ b(x,x')= K(x',\cdot).
\end{equation}
Given the reproducing property of $\Hilbert$, we have that for $f \in \Hilbert$ :
\begin{equation*}
A(x,x')f=(1-\alpha)f(x)K(x,\cdot) + \alpha f(x')K(x',\cdot).
\end{equation*}
Under this configuration: 
\begin{equation}{\label{eq:A_cov}}
\begin{aligned}
\mathbf{A} &=\ExpecRHO{(1-\alpha)\dotH{\cdot}{K(x,\cdot)}K(x,\cdot)+\alpha \dotH{\cdot}{K(x',\cdot)}K(x',\cdot)} \\
&=\ExpecAlpha{\dotH{\cdot}{K(y,\cdot)}K(y,\cdot)}=\CO, \\
\end{aligned}
\end{equation}
where the second equality is given by the linearity of the integral with respect to the mixture measure $P^{\alpha}$ and the definition of the covariance operator when restricted to elements of $\Hilbert$ (see \Sec{sec:notions_LRE}).
\begin{equation}{\label{eq:cov_ra}}
\mathbf{b}=\ExpecRHO{K(x',\cdot)}=\ExpecAlpha{r^{\alpha}(y)K(y,\cdot)}=\CO r^{\alpha}.
\end{equation}
The second equality is given by the change of measure expression
$\ExpecA{g(x')}=\ExpecAlpha{r^{\alpha}(y)g(y)}$ and the last one is due to the definition of the covariance operator and the hypothesis that $r^{\alpha} \in \Hilbert$ (see \Eq{eq:covariance_operator}). 

We can rewrite the \LRE problem described in \Eq{eq:PE_optimization_2} as trying to  minimize  the quadratic function: 
\begin{equation}{\label{eq:quadratic_function}}
Q(f)=\dotH{\CO(f-r^{\alpha})}{f-r^{\alpha})}= \frac{1}{2} \ExpecAlpha{(f-r^{\alpha})^2(y)},
\end{equation}
where the last equality is a consequence of property:

\begin{equation}{\label{eq:dot_covariance}}
   \dotH{f}{\CO(g)}=\ExpecAlpha{f(y)g(y)} \ \ \forall f,g \in \Hilbert
\end{equation}

The sequence of random variables $\{A_t\}_{t \in \N}$ and $\{b\}_{t \in \N}$ are given by the updates described in \Alg{alg:LRE_PEdiv}.
\begin{equation}
\begin{aligned}
A_t=A((x_t,x'_t))+ \lambda_t \idH; \ \ \ 
b_t=K(x'_t,\cdot).
\end{aligned}
\end{equation}
We can easily corroborate that $\mathbf{A}_t$ and $\mathbf{b}_t$ satisfy: 
\begin{equation}
\begin{aligned}
&\mathbf{A}_t=\CO+ \lambda_t \idH \ \ \text{ and } \ \ \mathbf{A}_t \rightarrow \mathbf{A} \ \ \text{ as } \ \ \lambda_t \rightarrow 0  ; \\
&\mathbf{b}_t=\CO r^{\alpha}.
\end{aligned}
\end{equation}
Moreover, by the properties of the covariance operator stated in \Sec{sec:notions_LRE}, $\mathbf{A}_t$ has a bounded inverse.

After putting together these elements, we can see how the stochastic approximation schema takes the form:  
\begin{equation}
\begin{aligned}
f_t(\cdot) &= f_{t-1}(\cdot) - \eta_t \left[ A_t f_{t-1}(\cdot) -b (x_t,x'_t)(\cdot) \right] \\
&= f_{t-1}(\cdot) - \eta_t \left[ (1-\alpha)f_{t-1}(x_{t}) K(x_t,\cdot)+\alpha f_{t-1}(x'_t) K(x'_t,\cdot) + \lambda_t f_{t-1}(\cdot)- K(x'_t,\cdot)  \right],
\end{aligned}
\end{equation}
which coincides with the functional stochastic gradient descent described in 
\Eq{eq:PE_gradient_step}. 

A term that will be important for studying the convergence of the online optimization schema is the solution to the regularized optimization problem: 
\begin{equation}{\label{eq:f_lambda}}
    f_{\lambda_t}= \argmin_{f \in \Hilbert}  \frac{1}{2} \dotH{\mathbf{A}_t(f-r^{\alpha})}{f-r^{\alpha}} = \argmin_{f \in \Hilbert} \frac{1}{2} \ExpecAlpha{(f-r^{\alpha})^2(y)}  + \frac{\lambda_t}{2} \norm{f}^2_{\Hilbert}.
\end{equation}
In fact $f_{\lambda_t}$ can be written as: 
\begin{equation}{\label{eq:w_t}}
f_{\lambda_t}=\mathbf{A}_t^{(-1)}\mathbf{b}_t=(\CO+\lambda_t \idH)^{(-1)}\mathbf{b}_t, 
\end{equation}

\subsection{Similarities between OLRE and Online Regression Problem}{\label{sec:similarities_OLRE_Regression}}

The framework described in \Sec{sec:SA_RP_H} was originally proposed to solve a regression problem in $\Hilbert$ as data observations arrive. % \citep{Tarres2014}. 
In this context, $\mathcal{Z}=(\mathcal{X},\mathcal{Y})$, where $\mathcal{X}$ is the feature space and $\mathcal{Y} \subset \R$ represents noisy observations of the regression function to be approximated ($f_{\rho}$). $\rho$ states for the joint probability function of $(x,y)$ whose marginal in the first entry is $\rho_{\mathcal{X}}$. The regression problem can be written as: 

\begin{equation*}{\label{eq:regression_problem}}
  \min_{f \in \Hilbert}  L^{\text{Reg}}(f)=    \min_{f \in \Hilbert}   \int_{\mathcal{X} \times \mathcal{Y}} (f(x)-y)^2 d \rho,
\end{equation*}
Given the previous problem,  the random variables
to be updated as $(x_t,y_t)$ arrive take the form: 
\begin{equation*}
\begin{aligned}
A^{\text{Reg}}(x,y)&=\dotH{\cdot}{K(x,\cdot)} K(x,\cdot)  \ \ \ \ & b^{\text{Reg}}(x,y) = y K(x,\cdot) \\
\mathbf{A}^{\text{Reg}}&= \CO^{\rho_x} \ \ \ \   & \mathbf{b}^{\text{Reg}}= \CO^{\rho_x} f_{\rho}   \\ 
A_t^{\text{Reg}}&=A^{\text{Reg}}(x_t,y_t)+ \lambda_t \idH  \ \ \ \ & b_t^{\text{Reg}} = y_t K(x_t,\cdot) \\
\mathbf{A}_t^{\text{Reg}}&=\CO^{\rho_x}+ \lambda_t \idH \ \ \ \ & \mathbf{b}_t^{\text{Reg}} = \CO^{\rho_x} f_{\rho}
\end{aligned}
\end{equation*}
and $f^{\text{Reg}}_{\lambda_t}=(\CO^{\rho_x}+\lambda_t \idH)^{(-1)}\mathbf{b}_t$.

As it can be seen, the main difference between Online Likelihood-Ratio Estimation and the Online Regression Problem is the definition of the random variables $\{A_t\}_{t \in \N}$ and $\{b_t\}_{t \in \N}$, while  the expected values of these random variables as well as the regularized term $f_{\lambda_t}$ take the same form. The covariance operator $\CO^{\rho_x}$ translates to $\CO$ defined in terms of the measure $P^{\alpha}$ and the regression function $f_{\rho}$ to the relative likelihood-ratio $r^{\alpha}$. These similarities facilitate the convergence analysis as we can reuse results provided in \cite{Tarres2014} regarding the deterministic terms and we only rework the terms involving  the random variables $\{A_t\}_{t \in \N}$ and $\{b_t\}_{t \in \N}$. 

\subsection{Required elements for convergence analysis}{\label{sec:martingale_reversed_martingale}}

The proof of Theorems \ref{thm:convergence_results} and \ref{thm:convergence_results_2} depends mainly on two iterative decompositions of the residuals between the solution to the approximation $f_t$ and the solution to the regularization problem $f_{\lambda_t}$. A martingale decomposition will lead to convergence rates with respect to the norm $\norm{\cdot}_{\ltwo}$, while a reversed martingale decomposition will be useful when analyzing the convergence rates associated with the norm $\norm{\cdot}_{\Hilbert}$.  

In order to enhance reading, we will denote by $\expec[\cdot]$ the expected value with respect to the joint distribution $\ExpecRHO{\cdot}$. We will call $\Xi_t$ the $\sigma$-algebra generated by the pairs of observations observed up to $t$ that is $\Xi_t=\sigma((x_1,x'_1),(x_2,x'_2),...,(x_t,x'_t))$. $\mathcal{B}_i$ will denote the sigma-algebra generated by the observation observed after $i$, $\mathcal{B}_i=\sigma((x_i,x'_i),(x_{i+1},x'_{i+1}),,...)$.

\inlinetitle{Martingale Decomposition}{.}~
Let us denote by $\operatorname{res}_t$ the difference between the stochastic approximation $f_t$, obtained via function stochastic gradient descent, and $f_{\lambda_t}$ the solution to the regularized problem \ref{eq:f_lambda}:
\begin{equation}{\label{eq:decomposition}}
\begin{aligned}
\operatorname{res}_t & := f_t-f_{\lambda_t} \\
&=f_{t-1} - \eta_t \left[ A_t f_{t-1} -b_t\right] -f_{\lambda_t} \\
&= (\idH- \eta_t \mathbf{A}_t)(f_{t-1} - f_{\lambda_{t}})  + \eta_t \left[ \left(\mathbf{A}_t -A_t \right) f_{t-1} + (b_t- \mathbf{A}_t f_{\lambda_{t}})) \right] \\ 
& =(\idH- \eta_t \mathbf{A}_t)(f_{t-1} - f_{\lambda_{t}})  + \eta_t \left[ \left(\mathbf{A}_t -A_t \right) f_{t-1} + (b_t- \mathbf{b}_t ) \right]
\\ &= (\idH- \eta_t \mathbf{A}_t) (f_{t-1}- f_{\lambda_{t-1}}) - (\idH- \eta_t \mathbf{A}_t) ( f_{\lambda_{t}}- f_{\lambda_{t-1}}) + \eta_t \left[ \left(\mathbf{A}_t -A_t \right) f_{t-1} + (b_t- \mathbf{b}_t ) \right] \\
& = (\idH- \eta_t \mathbf{A}_t) \operatorname{res}_{t-1} - (\idH- \eta_t \mathbf{A}_t) \Delta_t  + \eta_t  \epsilon_t,
\end{aligned}
\end{equation}
where we have used the iterative \Alg{eq:general_gradient}
and expression $\mathbf{b}_t= \mathbf{A}_t f_{\lambda_{t}}$ (\Eq{eq:w_t}). The term $\Delta_t := f_{\lambda_t}- f_{\lambda_{t-1}}$ denotes the difference between the solution of adjacent solutions to the regularized problem. The path $t \rightarrow f_{\lambda_t}$ is known as the regularization path. Finally, $\epsilon_t$ denotes the noise term: 
\begin{equation}{\label{eq:chi_t}}
\begin{aligned}
\epsilon_t &:= (\mathbf{A}_t-A_t) f_{t-1}+ (b_t-\mathbf{b}_t) \\
&=
\left( \CO- (1-\alpha)\dotH{\cdot}{K(x_t,\cdot)}K(x_t,\cdot) +\alpha \dotH{\cdot}{K(x'_t,\cdot)}K(x'_t,\cdot) \right) f_{t-1}  + K(x'_t,\cdot) - \CO r^a\\
&= \ExpecAlpha{f_{t-1}(y) K(y,\cdot)} - (1-\alpha) f_{t-1}(x_t) K(x_t,\cdot) -  \alpha f_{t-1}(x'_t) K(x'_t,\cdot) \\
& + K(x'_t,\cdot) - \ExpecAlpha{r^{\alpha}(y) K(y,\cdot)} \ \ \ \ (\text{\tiny{\Eq{eq:covariance_operator} and the first point of \Expr{RKHS_properties}.}})   \\
&=  \ExpecAlpha{f_{t-1}(y) K(y,\cdot)} - (1-\alpha) f_{t-1}(x_t) K(x_t,\cdot) -  \alpha f_{t-1}(x'_t) K(x'_t,\cdot) +  K(x'_t,\cdot) - \ExpecA{ K(x',\cdot)}. 
\end{aligned}
\end{equation}
If we iterate \Expr{eq:decomposition} up to $s \leq t$: 
\begin{equation}{\label{eq:martingale_decomposition}}
    res_{t}= \bar{\Pi}_{s+1}^t res_s -\sum_{j=s+1}^{t} \bar{\Pi}_{j}^t  \Delta_j  + \sum_{j=s+1}^{t} \eta_j\bar{\Pi}_{j+1}^t \epsilon_j,
\end{equation}
\begin{equation}{\label{eq:pi_operator}}
\bar{\Pi}_{j}^t =\begin{cases}
\prod_{i=j}^{t} (\idH- \eta_i \mathbf{A}_i) \ \  \text{, if}  \ \ j \leq t; \\
\idH, \ \ \ \  \ \ \ \ \ \ \ \  \ \ \ \  \ \ \ \  \ \ \  \text{otherwise.} 
\end{cases}
\end{equation}
From the \Eq{eq:chi_t} and the independence of incoming observations, it is easy to verify that the process $\{\eta_j\bar{\Pi}_{j+1}^t \epsilon_j\}_{j \in \N}$ defines a martingale difference with respect to the filtration $\{\Xi_t\}_{t \in \N}$. The decomposition
of \Eq{eq:martingale_decomposition} was first proposed in \cite{Yao2010}. The proof of \Theorem{RKHS_properties} consists of finding an  upperbound for the norm of each of the three terms in \Eq{eq:martingale_decomposition} and the residual  difference $f_{\lambda_t}-r^{\alpha}$.

\inlinetitle{Reversed Martingale Decomposition}{.}

Let us define the following random operator in terms of the  sample $((x_1,x'_1),(x_2,x'_2),...,(x_n,x'_n))$ and indexed by $j,t \in \N$: 
\begin{equation}{\label{eq:pi_random_operator}}
\Pi_{j}^t(\{(x_i,x'_i)\}_{i \in \N}) =\begin{cases}
\prod_{i=j}^{t} (\idH- \eta_i A_i(x_i,x'_i)), \ \  \text{if}  \ \ j \leq t; \\
\idH, \ \ \ \  \ \ \ \ \ \ \ \  \ \ \ \  \ \ \ \  \ \ \ \ \ \ \ \ \ \   \text{otherwise}. 
\end{cases}
\end{equation}
We recover an alternative decomposition for the residual $res_t$: 
\begin{equation*}
\begin{aligned}
res_t&=f_t-f_{\lambda_t}\\
&= f_{t-1}-f_{\lambda_t} - \eta_t( A_t f_{t-1}- b_t)\\
& = (\idH-\eta_t A_t) (f_{t-1}-f_{\lambda_{t-1}})-(\idH-\eta_tA_t)(f_{\lambda_t}-f_{\lambda_{t-1}})- \eta_t(A_t f_{\lambda_t}-b_t) \\
& = (\idH- \eta_t A_t)res_{t-1} - (\idH- \eta_t A_t) \Delta_t 
-\eta_t (A_t f_{\lambda_t} - b_t).
\end{aligned}
\end{equation*}
By iterating the last expression for $s \leq t$, we recover the following equality: 
\begin{equation}{\label{eq:reversed _martingale_decomposition}}
    res_{t}= \Pi_{s+1}^t res_s -\sum_{j=s+1}^{t} \Pi_{j}^t  \Delta_j  - \sum_{j=s+1}^{t} \eta_j \Pi_{j+1}^t (A_j f_{\lambda_j} - b_j).
\end{equation}
This decomposition was first introduced in \cite{Tarres2014}.

Let us show that $\{\Pi_{j+1}^t (A_j f_{\lambda_j} - b_j)\}_{j\in \N}$ is a reversed martingale difference with respect to $\{\mathcal{B}_j\}_{\{j \in \N\}}$. 
\begin{definition}
Let $\{\mathcal{B}_i\}_{\{i \in \N\}}$ be a decreasing sequence of sub-$\sigma$-fields of $\mathcal{A}$ in the probability space $(\mathcal{Z},\mathcal{A},\rho)$. A sequence $\{\zeta_i\}_{i \in \N} $ integrable real random variables is called a reversed martingale difference if: 
\begin{enumerate}
    \item The real random variable $\zeta_i$ is $\mathcal{B}_i$-measurable for all $i \in \N$, 
    \item $\expec \left[\zeta_i \,|\, \mathcal{B}_{i+1} \right]=0$ for all $i \in \N$
\end{enumerate}
\end{definition}
The term $\eta_j \Pi_{j+1}^t (A_j f_{\lambda_j} - b_j)$ defines a reversed martingale with respect to the sequence $\mathcal{B}_j=\sigma((x_j,x'_j),...,(x_t,x'_t),...)$. From its definition $\eta_j \Pi_{j+1}^t (A_j f_{\lambda_j} - b_j)$  is $\mathcal{B}_j$ measurable, moreover given the independence of the observations we have: 
\begin{equation*}
\begin{aligned}
\expec \left[ \eta_j \Pi_{j+1}^t (A_j f_{\lambda_j} - b_j)  \,|\, \mathcal{B}_{j+1} \right]
&= \eta_j \Pi_{j+1}^t  \expec \left[  A_j f_{\lambda_j} - b_j  \,|\, \mathcal{B}_{j+1} \right] \\
& = \eta_j \Pi_{j+1}^t \left(   \mathbf{A}_j f_{\lambda_j} - \mathbf{b}_j \right) \ \ \ \ (\text{\tiny{By the independence hypothesis }})  \\
 &=0 \ \ \ \ 
 (\text{\tiny{\Eq{eq:w_t} }}).
\end{aligned}
\end{equation*}

\subsection{Convergence in $\ltwo$}

The proof of \Theorem{thm:convergence_results} mimics the proof of Theorem C in \cite{Tarres2014}. This theorem is stated in Online Linear Regression and built upon decomposition of \Eq{eq:martingale_decomposition}. As explained in \Sec{sec:similarities_OLRE_Regression}), the regression problem is similar to \OLRE with the main difference being the operators  $A_t$ and $b_t$.  This difference requires us to rework the bounds depending on the random processes $\{A_t\}_{t \in \N}$ and $\{b_t\}_{t \in \N}$. 

Let us start by analyzing the $\ltwo-$norm of the residuals:  
\begin{equation}{\label{eq:error_decomposition}}
\begin{aligned}
\norm{f_t-r^{\alpha}}_{\Ltwo}  &\leq \norm{f_{\lambda_t}-r^{\alpha} }_{\Ltwo}+ \norm{f_t-f_{\lambda_t}}_{\Ltwo} \\
& \leq \norm{f_{\lambda_t}-r^{\alpha} }_{\Ltwo} +  \norm{\bar{\Pi}_{1}^t res_0}_{\Ltwo}   +\norm{\sum_{j=1}^{t} \bar{\Pi}_{j}^t  \Delta_{j}}_{\Ltwo} +  \norm{\sum_{j=1}^{t} \eta_j\bar{\Pi}_{j+1}^t \epsilon_j }_{\Ltwo} \\
& = \mathcal{E}_{\text{init}}(t) +   \mathcal{E}_{\text{drift}}(t)+ \mathcal{E}_{\text{approx}} +\mathcal{E}_{\text{sample}}(t),
\end{aligned}
\end{equation}
where the second line comes from the martingale decomposition of \Eq{eq:martingale_decomposition} applied to s=0. Each of the error terms in \Eq{eq:error_decomposition} is defined as:
\begin{equation}
\begin{aligned}
 \mathcal{E}_{\text{init}}(t) &:=\norm{\bar{\Pi}_{1}^t res_0}_{\Ltwo} 
\ \ \ \ \ \ \ \ \ \ \ \ \ \ 
 \mathcal{E}_{\text{approx}}(t):=\norm{f_{\lambda_t}-r^{\alpha} }_{\Ltwo}, \\
\mathcal{E}_{\text{drift}}(t)&:=\norm{\sum_{j=1}^{t} \bar{\Pi}_{j}^t  \Delta_j}_{\Ltwo} \ \ \quad \qquad
\mathcal{E}_{\text{sample}}(t):=\norm{\sum_{j=1}^{t} \eta_j\bar{\Pi}_{j+1}^t \epsilon_j }_{\Ltwo} 
\end{aligned}
\end{equation}
The first three components have the same  behavior in the OLRE and  Regression Problem as they depend solely on equivalent deterministic terms, meaning we can reuse the upper bounds available in \cite{Tarres2014}. For completeness of exposition, we restate these results. 
The last term differs and an upperbound is derived in Theorem \ref{thm:sample_error}.

For the following statements $\bar{t}=t+t_0$, and $t_0 >0$ will be a given integer, $a,b$ are two positive constants, $\sourceparameter$ is the parameter related to the smoothness of a function in $\Hilbert$ as it was explained in \Sec{sec:notions_LRE} and $\alpha$ is the regularized parameter of the relative likelihood-ratio function $r^{\alpha}$. 
\begin{theorem}{\label{thm:TheoremVI.I}} (\textbf{Theorem VI.1 in \cite{Tarres2014}})
Let $t_0^{\theta} \geq a (C^2+b)$. Then for all $t \in \N$.
\begin{equation}
\mathcal{E}_{\textup{init}}(t) \leq \frac{1}{\alpha} \left(\frac{t_0+1}{\bar{t}} \right)^{ab} \leq B_1 \bar{t}^{-ab},
\end{equation}
where $B_1= \frac{(t_0+1)^{ab}}{\alpha}$.
\end{theorem}
\begin{theorem}{\label{thm:TheoremVI.II}}(\textbf{Theorem VI.2 in \cite{Tarres2014}}) For $\sourceparameter \in (0,1]$ and $\CO^{(-\sourceparameter)} r^{\alpha} \in \Ltwo$,
\begin{equation}
\mathcal{E}_{\textup{approx}}(t) \leq \frac{ b^\sourceparameter \bar{t}^{(-\sourceparameter(1-\theta))} \norm{\CO^{(-\sourceparameter)} r^{\alpha}}_{\Ltwo}}{\sourceparameter} \leq B_2 b^\sourceparameter \bar{t}^{(-\sourceparameter(1-\theta))},
\end{equation}
where $B_2=\frac{\norm{\CO^{(-\sourceparameter)} r^{\alpha}}_{\Ltwo}}{\sourceparameter}$.
\end{theorem}
\begin{theorem}{\label{thm:TheoremVI.III}}(\textbf{Theorem VI.3 in \cite{Tarres2014}}) Assume $t_0^{\theta}=[a(C^2+b) \lor	1]$. Then, if $\sourceparameter \in (0,1]$ and $\CO^{(-\sourceparameter)} r^{\alpha} \in \Ltwo$:
\begin{equation}
\mathcal{E}_{\textup{drift}}(t) =\begin{cases}
B_3 b^\sourceparameter \bar{t}^{-\sourceparameter(1-\theta)} \ \  \text{if}  \ \ ab > \sourceparameter (1-\theta); \\
B_3 b^\sourceparameter \bar{t}^{-ab} \ \  \text{if}  \ \ ab < \sourceparameter (1-\theta),
\end{cases}
\end{equation}
where $B_3= \frac{4(1-\theta)}{\abs{ab-\beta(1-\theta)}} \norm{\CO^{(-\beta)} r^{\alpha}}_{\Ltwo}$.
\end{theorem}
\begin{theorem}{\label{thm:sample_error}}
    Assume that $\CO^{(-\sourceparameter)}(r^{\alpha}) \in \Ltwo$ for some $\sourceparameter \in [\frac{1}{2},1]$, $\theta \in [\frac{1}{2},\frac{2}{3}]$, $ab=1$, $a \geq 4$ and $t^{\theta}_0 \geq 2 + 4 C^2 a$. Then, for all $t \in \N$, with probability at least $1-\delta$:
\begin{equation}
\mathcal{E}_{\textup{sample}}(t) \leq    \frac{\sqrt{a}  B_4}{\bar{t}^{\frac{\theta}{2}}} \log \left(\frac{2}{\delta} \right) + \left[ B_5 a^{\frac{5}{2}} + B_6 a^{\frac{7}{2}} \sqrt{\log{\bar{t}}} \right] \frac{\log^2 \left(\frac{2}{\delta} \right)}{\bar{t}^{\frac{3\theta-1}{2}}},
\end{equation}
where: 
%
%%%%
\begin{equation*}
B_4= \frac{16 C}{\alpha}\ \ \ B_5= \frac{32 C^3}{\alpha} \ \ \ B_6= \frac{8 C^3(10C+3)}{\alpha}.
\end{equation*}
%%%%
\end{theorem}
The proof of the last theorem is given in \Sec{subsection:upperboundsnoise}. 

\inlinetitle{Proof of \Theorem{thm:convergence_results}}{}%
\begin{proof}
By putting together the conditions stated in the statement of \Theorem{thm:convergence_results}
and fixing $\theta= \frac{2\sourceparameter}{2\sourceparameter +1}$, $a \geq 4$, $b \leq \frac{1}{4}$ such that $ab=1$ and $t_0^{\theta} \geq 4a C^2+2 $ we can verify that the requirements of \Theorems{thm:TheoremVI.I}-\ref{thm:sample_error} are satisfied: 
\begin{equation}
\begin{aligned}
\norm{f_t-r^{\alpha}}_{\Ltwo} & \leq \mathcal{E}_{\text{init}}(t) + \mathcal{E}_{\text{approx}}(t) + \mathcal{E}_{\text{drift}}(t) + \mathcal{E}_{\text{sample}}(t) \\
& \leq \frac{B_1}{\bar{t}} + \left( (B_2 + B_3) a^{-\sourceparameter} + \sqrt{a} B_4 \log \left( \frac{2}{\delta} \right) \right) \left( \frac{1}{\bar{t}}  \right)^{\frac{\sourceparameter}{2\sourceparameter+1}} \\ & + \left(  B_5 a^{\frac{5}{2}}  +B_6 a^{\frac{7}{2}} \sqrt{\log(\bar{t})} \right)  \frac{\log^2 \left(\frac{2}{\delta} \right)}{\bar{t}^{\frac{4\sourceparameter-1}{4\sourceparameter+2}}} \\
&= \frac{C_1}{\bar{t}} + \left( C_2 a^{-r} + C_3 \sqrt{a} \log \left( \frac{2}{\delta} \right) \right)  \left( \frac{1}{\bar{t}}  \right)^{\frac{\sourceparameter}{2\sourceparameter+1}} +  \left(  C_4 a^{\frac{5}{2}}  +C_5 a^{\frac{7}{2}} \sqrt{\log(\bar{t})} \right)  \frac{\log^2 \left(\frac{2}{\delta} \right)}{\bar{t}^{\frac{4\sourceparameter-1}{4\sourceparameter+2}}}
\end{aligned}
\end{equation}
where $C_1= \frac{2 t_0}{\alpha} \geq  \frac{t_0+1}{\alpha}$, $C_2=B_2+ B_3=  \frac{5\sourceparameter+1}{\sourceparameter(1+\sourceparameter)} \norm{\CO^{(-\sourceparameter)} r^{\alpha}}_{\Ltwo}$, $C_3=B_4$, $C_4=B_5$ and $C_5=B_6$.
\end{proof}

\subsection{Convergence in $\Hilbert$}

The study of the norm in $\Hilbert$ keeps a lot of similarities with the analysis of $\ltwo$, starting with the decomposition of the norm into four terms that will be upper-bounded independently:
\begin{equation}{\label{eq:error_decomposition_Hilbert}}
\begin{aligned}
\norm{f_t-r^{\alpha}}_{\Hilbert} &\leq \norm{f_t-f_{\lambda_t}}_{\Hilbert} + \norm{f_{\lambda_t}-r^{\alpha}}_{\Hilbert} \\
& \leq \norm{f_{\lambda_t}-r^{\alpha} }_{\Hilbert} +  \norm{\Pi_{1}^t res_0}_{\Hilbert}   +\norm{\sum_{j=1}^{t} \Pi_{j}^t  \Delta_{j}}_{\Hilbert} +  \norm{\sum_{j=1}^{t} \eta_j\Pi_{j+1}^t (A_jf_{\lambda_j}-b_j) }_{\Hilbert} \ \ \ \ (\text{\tiny{\Eq{eq:reversed _martingale_decomposition}}}) \\
& = \mathcal{E}'_{\text{init}}(t) +   \mathcal{E}'_{\text{drift}}(t)+ \mathcal{E}'_{\text{approx}}(t) +\mathcal{E}'_{\text{sample}}(t),
\end{aligned}
\end{equation}
%%%
where, 
\begin{equation}
\begin{aligned}
 \mathcal{E}'_{\text{init}}(t) &:=\norm{\Pi_{j}^t res_0}_{\Hilbert} 
\ \ \ \ \ \ \ \ \ \ \ \ \  
 \mathcal{E}'_{\text{approx}}(t) :=\norm{f_{\lambda_t}-r^{\alpha} }_{\Hilbert}, \\
\mathcal{E}'_{\text{drift}}(t)&:=\norm{\sum_{j=1}^{t} \Pi_{j}^t  \Delta_j}_{\Hilbert} \ \ \ \ \ \ \ \ \ \  \mathcal{E}'_{\text{sample}}(t) :=\norm{\sum_{j=1}^{t} \eta_j\Pi_{j+1}^t (A_j f_{\lambda_j}-b_j) }_{\Hilbert}  \ \
\end{aligned}
\end{equation}
As in the previous case, we will start by restating the required elements from \cite{Tarres2014} to upper-bound the deterministic terms of \Eq{eq:error_decomposition_Hilbert}. 
%%%%
\begin{theorem}
{\label{thm:TheoremV.I}} (\textbf{Theorem V.1 in \cite{Tarres2014}}) Let $t_0^{\theta} \geq a (C^2+b)$. Then, for all $t \in \N$, 
%%%%
\begin{equation}
    \mathcal{E}'_{\text{init}}(t) \leq B'_1 \bar{t}^{-ab},
\end{equation}
 %%%%  
 where $B'_1=(t_0+1)^{ab} \norm{f_{\lambda_0}}_{\Hilbert}$.
\end{theorem}
%%%
%
\begin{theorem}
{\label{thm:TheoremV.II}} (\textbf{Theorem V.2 in \cite{Tarres2014}}) For $\sourceparameter \in (\frac{1}{2}, \frac{3}{2} ]$ and $\CO^{(-\sourceparameter)} r^{\alpha} \in \Ltwo$
%%%%
\begin{equation}
    \mathcal{E}'_{\textup{approx}}(t) \leq B'_2 b^{\sourceparameter-\frac{1}{2}} \bar{t}^{-(\sourceparameter-\frac{1}{2})(1-\theta)},
\end{equation}
 %%%%  
where $B'_2=(\sourceparameter-\frac{1}{2})^{-1} \norm{\CO^{(-\sourceparameter)} r^{\alpha}}_{\Ltwo}$.
\end{theorem}
\begin{theorem}
{\label{thm:TheoremV.III}} (\textbf{Theorem V.3 in \cite{Tarres2014}}) Let $t_0^{\theta} \geq \max{(a(C^2+b),1)}$. Then, for $\sourceparameter \in (\frac{1}{2},\frac{3}{2}]$ and $\CO^{(-\sourceparameter)} r^{\alpha} \in \Ltwo$, 
\begin{equation}
\mathcal{E}'_{\textup{drift}}(t) =\begin{cases}
B'_3 b^{\sourceparameter-\frac{1}{2}} \bar{t}^{-(\sourceparameter-\frac{1}{2})(1-\theta)} \ \  \text{if}  \ \ ab > (\sourceparameter-\frac{1}{2})(1-\theta); \\
B'_3 b^{\sourceparameter-\frac{1}{2}} \bar{t}^{-ab} \ \ \ \  \qquad \ \
 \text{if}  \ \ ab < (\sourceparameter-\frac{1}{2})(1-\theta),
\end{cases}
\end{equation}
where $B'_3= \frac{4(1-\theta)}{\abs{ab-(\sourceparameter-\frac{1}{2})(1-\theta)}} \norm{\CO^{(-\sourceparameter)} r^{\alpha}}_{\Ltwo}$.
\end{theorem}
\begin{theorem}{{\label{thm:sample_error_Hilbert}}} Assume that $t_0^{\theta}\geq \min\{ a(C^2+b),b,1\}$, $t_0^{1-\theta} \geq b$ and $ab \neq \theta- \frac{1}{2}$ or $ab \neq \frac{3\theta- 1}{2}$. Then, with probability at least $1-\delta$ ($\delta \in (0,1)$), 
%%%%%
\begin{equation}
\mathcal{E}'_{\textup{sample}}(t) \leq a b^{-\frac{1}{2}} B'_4 \bar{t}^{-\left(ab \wedge  \frac{3\theta - 1}{2} \right)}    + B'_5 a \bar{t}^{-\left( ab \wedge (\theta- \frac{1}{2})\right)}, 
\end{equation}
%%%%
where 
%%%
\begin{equation}
B'_4= \frac{ 2 e}{3 } \left( \frac{C+1}{\alpha} + C\right) \log \left( \frac{2}{\delta}\right)  \ \ B'_5=2 \sqrt{\frac{1}{\abs{ab-\theta+\frac{1}{2}}}} eC    \log\left( \frac{2}{\delta} \right).
\end{equation}
%%%%%
\end{theorem}
The proof of \Theorem{thm:sample_error_Hilbert} is provided in \Appendix{subsection:upperboundsnoise} and it capitalizes over the properties of the operator $A_t$ and \Lemma{lemma:propositionA3}.

\inlinetitle{Proof of \Theorem{thm:convergence_results_2}}{.}
\begin{proof}
Let us fix $\theta=\frac{2\sourceparameter}{2\sourceparameter+1}$, $a \geq 1$,$b \leq 1$ such that $ab=1$ and we assume $t_0^{\theta} \geq a C^2 +1$. These constants imply that $\theta <1$
and $t_0 \geq 1$ which means $t_0^{1-\theta} \geq b$, $ab = 1 \neq \theta- \frac{1}{2}$, $ab \neq \frac{(3 \theta -1)}{2}$ and $(\sourceparameter-\frac{1}{2})(1-\theta)=\frac{2\sourceparameter-1}{4\sourceparameter+2} \leq 1 $. Then the hypothesis of Theorems \ref{thm:TheoremV.I}-  
\ref{thm:sample_error_Hilbert} are satisfied and we can conclude: 
\begin{equation}
\begin{aligned}
\norm{f_t-r^{\alpha}}_{\Hilbert} & \leq B'_1 \bar{t}^{-ab} + B'_2 b^{\sourceparameter-\frac{1}{2}} \bar{t}^{-(\sourceparameter-\frac{1}{2})(1-\theta)} + B'_3 b^{\sourceparameter-\frac{1}{2}} \bar{t}^{-(\sourceparameter-\frac{1}{2})(1-\theta)} \\
& + a b^{-\frac{1}{2}} B'_4 \bar{t}^{-\left(ab \wedge  \frac{3\theta - 1}{2} \right)}  
+ B'_5 a \bar{t}^{-\left( ab \wedge (\theta- \frac{1}{2})\right)} \\
& \leq   B'_1 \bar{t}^{-ab} + B'_2 a^{\frac{1}{2}-\sourceparameter} \bar{t}^{-\frac{2\sourceparameter-1}{4\sourceparameter+2}} + B'_3 a^{\frac{1}{2}-\sourceparameter}\bar{t}^{-\frac{2\sourceparameter-1}{4\sourceparameter+2}} + a a^{\frac{1}{2}} B'_4 \bar{t}^{-\theta}  \bar{t}^{-\frac{2\sourceparameter-1}{4\sourceparameter+2}}
+ B'_5 a \bar{t}^{-\frac{2\sourceparameter-1}{4\sourceparameter+2}} \\
& = C'_1  \bar{t}^{-ab} + \left[ C'_2 a^{\frac{1}{2}-\sourceparameter}    +  C'_3 a \log \left( \frac{2}{\delta} \right)  \right] \bar{t}^{-\frac{2\sourceparameter-1}{4\sourceparameter+2}},
\end{aligned}
\end{equation}
where:
\begin{equation*}
\begin{aligned}
 B'_1 &= (t_0+1) \norm{f_{\lambda_0}}_{\Hilbert} \leq  \frac{t_0+1}{\alpha \sqrt{\lambda_0}} \ \ \ \ (\text{\tiny{\Eq{eq:upperbound_f_lambda_Hilbert}}}) \\
& \leq  \frac{2 t_0 (t_0)^{\frac{\theta-1}{2}}}{\alpha b^{\frac{1}{2}}} = \frac{2 t_0^{\frac{4\sourceparameter+1}{4\sourceparameter+2}}}{\alpha b^{\frac{1}{2}}} :=   C'_1 \\
C'_2 &:=  B'_2 +B'_3 = \left( (\sourceparameter-\frac{1}{2})^{-1}  + \frac{4(1-\theta)}{\abs{ab-(\sourceparameter-\frac{1}{2})(1-\theta)}} \right) \norm{\CO^{(-\sourceparameter)} r^{\alpha}}_{\Ltwo} \\
&= \left(\frac{2}{2\sourceparameter-1}+ \frac{8}{2\sourceparameter+3} \right)  \norm{\CO^{(-\sourceparameter)} r^{\alpha}}_{\Ltwo} = \frac{20\sourceparameter - 2}{(2\sourceparameter-1)(2\sourceparameter+3)} \norm{\CO^{(-\sourceparameter)} r^{\alpha}}_{\Ltwo}
\end{aligned}
\end{equation*}
\begin{equation*}
\begin{aligned}
\sqrt{a} B'_4 \bar{t}_0^{-\theta} + B'_5 & \leq  \frac{ 2 e}{3 } \left( \frac{1}{\alpha } +  \frac{1}{\alpha C}+ 1\right) \log \left( \frac{2}{\delta}\right) +2 \sqrt{\frac{1}{\abs{ab-\theta+\frac{1}{2}}}} eC    \log\left( \frac{2}{\delta} \right) (\text{\tiny{The fact $t_0^{\theta} \geq aC^2+1$ implies $\sqrt{a} t_0^{-\frac{\theta}{2}} \leq \frac{1}{C}$}}) \\
& \leq e \log \left( \frac{2}{\delta} \right) \left(  3 C + \frac{2}{3} \left( \frac{C+1}{C \alpha} +1 \right) \right) \leq  2  \log \left( \frac{2}{\delta} \right) \left( \frac{9C^2\alpha+ C(2+\alpha) + 2 }{3 C \alpha}\right) \\ &
\leq 2  \log \left( \frac{2}{\delta} \right) \frac{(3C+\sqrt{2})^2}{3 C \alpha} \leq 6 \log \left( \frac{2}{\delta} \right) \frac{(C+1)^2}{C \alpha}  \leq  
C'_3 \log \left( \frac{2}{\delta} \right).
\end{aligned}
\end{equation*}
\end{proof}

\subsection{Upperbounds on the noise terms}{\label{subsection:upperboundsnoise}}

The goal of this section is to detail the upper-bounds on the noise terms presented in Theorems  \ref{thm:sample_error} and \ref{thm:sample_error_Hilbert}.

For the remainder of the discussion we will fix the step-size and the regularization constant sequence as: 
\begin{equation}
\eta_t = \frac{a}{(t+t_0)^{\theta}}  = \frac{a}{\bar{t}^{\theta}} \ \  \ \   \lambda_t= \frac{b}{(t+t_0)^{1-\theta}}=\frac{b}{\bar{t}^{1-\theta}},\ \ \ \  \text{ for some } \ \ \ \  \theta \in [0,1], t_0>0.
\end{equation}
Let us begin with the definition of the following stochastic processes: 
\begin{equation}{\label{eq:LR_terms}}
\begin{aligned}
\LP_t &=  \dotH{\cdot}{K(x_t,\cdot)} K(x_t,\cdot), \qquad\qquad\qquad\qquad\qquad\qquad\qquad  \mathbf{L}= \expec \left[ \LP_t   \,|\, \Xi_{t-1} \right] \\
\LQ_t &= \dotH{\cdot}{K(x'_t,\cdot)} K(x'_t,\cdot), \qquad\qquad\qquad\qquad\qquad\qquad\qquad  \mathbf{R}= \expec \left[ \LQ_t   \,|\, \Xi_{t-1} \right] \\
LR_t &= (1-\alpha) \dotH{\cdot}{K(x_t,\cdot)} K(x_t,\cdot)+ \alpha \dotH{\cdot}{K(x'_t,\cdot)} K(x'_t,\cdot) \ \ \mathbf{LR}=  \expec \left[  (1-\alpha) \LP_t  + \alpha \LQ_t  \,|\, \Xi_{t-1} \right].
\end{aligned}
\end{equation}
%%%
Notice that the operator $\mathbf{LR}$  coincides with the covariance operator for $f \in \Xi_{t-1} $ (\Expr{eq:covariance_operator}): 
\begin{equation}{\label{eq:LR_COV}}
\mathbf{LR}(f) = \expec \left[  (1-\alpha) f(x_t) K(x_t,\cdot) + \alpha  f(x'_t) K(x'_t,\cdot)  \,|\, \Xi_{t-1} \right] =  \ExpecAlpha{f(y) K(y,\cdot)} = \CO(f)
\end{equation}

\inlinetitle{Upper-bound for $\mathcal{E}_{\text{sample}}(t)=\norm{\sum_{j=s+1}^{t} \eta_j\bar{\Pi}_{j+1}^t \epsilon_j }_{\Ltwo}$}{.}

We can rewrite the residual between the approximation at time $f_t$ and the relative likelihood-ratio $r^{\alpha}$ in terms of the stochastic processes defined in \ref{eq:LR_terms}: %%%
%%%%
\begin{equation}
\begin{aligned}
f_t- r^{\alpha} & = 
\left[ \idH - \eta_t \left((1-\alpha)\LP_t+ \alpha\LQ_t + \lambda_t \idH \right) \right] \left( f_{t-1}  \right) + \eta_t  K(x'_t,\cdot) - r^{\alpha}   \\
& = \left[ \idH - \eta_t \left((1-\alpha)\LP_t+ \alpha\LQ_t + \lambda_t \idH \right) \right] \left( f_{t-1} - r^{\alpha} \right) + \eta_t \left( K(x'_t,\cdot) - \left[(1-\alpha)\LP_t+\alpha \LQ_t \right] r^{\alpha}  \right) - \eta_t \lambda_t r^{\alpha},
\end{aligned}
\end{equation}
%%%
%
Let us define the sequences $\{g_t\}_{t \in \N}$, $\{h_t\}_{t \in \N}$: 
\begin{equation*}
g_0= -r^{\alpha}  \ \ \ \ h_0= 0,
\end{equation*}
%%%
and 
%%%
\begin{equation}{\label{eq:g_t_h_t}}
\begin{aligned}
g_t &= \left[ \idH - \eta_t \left(\ELPQ+ \lambda_t \idH \right) \right] g_{t-1} - \eta_t \lambda_t r^{\alpha} \\ 
h_t &= \left[ \idH - \eta_t \left(LR_t  + \lambda_t \idH \right) \right] h_{t-1} + \eta_t \left( K(x'_t,\cdot) - LR_t r^{\alpha} \right) + \eta_t  \left[ \ELPQ -  LR_t \right]  g_{t-1}. \\
\end{aligned}
\end{equation}
%%%
%
By induction over \Expr{eq:g_t_h_t} we can verify:
\begin{equation}{\label{eq:decomposition_residual_g_t_h_t}}
f_t-r^{\alpha}= g_t + h_t.
\end{equation}
Notice that $\{g_t\}_{t \in \N}$ is a deterministic sequence, while $\{h_t\}_{t \in \N}$ is random. 

We can use the aforementioned variables to upperbound the Hilbert norm of the noise term as follows: 
\begin{equation}{\label{eq:norm_noise}}
\begin{aligned}
\expec\left[\norm{\epsilon_t}^2_{\Hilbert}  \,|\, \Xi_{t-1}\right]
&= \expec \left[ {\norm{ K(x'_t,\cdot) - (1-\alpha) L_t f_{t-1} - \alpha R_t f_{t-1} - \left( \CO r^{\alpha} - \CO f_{t-1} \right)}^2_{\Hilbert}  \,|\, \Xi_{t-1}} \right] \ \ \ \ (\text{\tiny{\Eq{eq:chi_t}}}) \\
& \leq  \expec \left[ {\norm{ K(x'_t,\cdot) - (1-\alpha) L_t f_{t-1} - \alpha R_t f_{t-1} }^2_{\Hilbert} \,|\, \Xi_{t-1}} \right]  (\text{\tiny{After developing the norm and using \Eq{eq:cov_ra} and \Eq{eq:LR_COV} }}) \\
& = \expec \left[ {\norm{ K(x'_t,\cdot) - (1-\alpha) L_t \left( r^{\alpha} + g_{t-1} + h_{t-1} \right) - \alpha R_t \left( r^{\alpha} + g_{t-1} + h_{t-1} \right) }^2_{\Hilbert} \,|\, \Xi_{t-1}} \right] (\text{\tiny{\Eq{eq:decomposition_residual_g_t_h_t}}}) \\
& \leq 3 \Bigg[ \expec \left[  \norm{K(x'_t,\cdot) - \left[ (1-\alpha) L_t  + \alpha R_t \right] r^{\alpha} }^2_{\Hilbert}   \,|\, \Xi_{t-1} \right] +  \expec \left[  \norm{ \left[ (1-\alpha) L_t + \alpha R_t \right] g_{t-1} }^2_{\Hilbert}   \,|\, \Xi_{t-1} \right] \\
& +  \expec \left[  \norm{ \left[ (1-\alpha) L_t + \alpha R_t \right] h_{t-1} }^2_{\Hilbert}   \,|\, \Xi_{t-1} \right] 
\Bigg] \ \ \ \  (\text{\tiny{Inequality $2 \dotH{a}{b} \leq  \norm{ a}^2_{\Hilbert} +  \norm{b}^2_{\Hilbert}  $}}).
\end{aligned}
\end{equation}
For the rest of this section, we will focus on  upperbound each of the terms in \Eq{eq:norm_noise}.  

Let us start by analyzing the deterministic sequence in  $\{g_t\}_{t \in \N}$. We rewrite the following inequality shown in 
\cite{Tarres2014} (Lemma B.3):
\begin{lemma}{\label{lemma:g_upperbound}}
Assume $t_0^{\theta} \geq a (C^2+b)$. Then, for all $t \in \N$, 
\begin{enumerate}
    \item $\norm{g_{t}}_{\Ltwo} \leq \frac{1}{\alpha}$
    \item $\norm{g_t+r^{\alpha}}_{\Hilbert} \leq \frac{3}{\alpha\sqrt{\lambda_t}}$.
\end{enumerate}
\end{lemma}
As a  consequence of this result we can easily verify the following inequality: 
\begin{lemma}{\label{lemma:LR_g_t}}
   $g_{t-1}$ satisfies the following inequalities:
\begin{equation}
 \norm{\ELPQ g_{t-1}}^2_{\Hilbert} \leq \frac{C^2}{\alpha^2}. 
\end{equation}
\end{lemma}
\begin{proof}
As $g_{t-1}$ is $\Xi_{t-1}-$ measurable we have that $\ELPQ(g_{t-1})=\CO g_{t-1}$ (see \ref{eq:LR_COV}),then: 
\begin{equation*}
\begin{aligned}
\norm{\ELPQ(g_{t-1})}^2_{\Hilbert} = \dotH{\CO(g_{t-1})}{\CO(g_{t-1})} &=\ExpecAlpha{\CO(g_{t-1})(y)g_{t-1}(y)} \ \ \ \ (\text{\tiny{\Eq{eq:dot_covariance} }}) \\
&=\ExpecAlpha{ \int K(x,y) g_{t-1}(x) g_{t-1}(y)  dP^{\alpha}(x) } \ \ \ \ (\text{\tiny{\Eq{eq:covariance_operator} }})  \\ 
& \leq C^2 \left( \ExpecAlpha{g_{t-1}(y)} \right)^2 \ \ \ \ (\text{\tiny{\Assumption{ass:kernel_map_upperbound} }})\\
& \leq C^2 \ExpecAlpha{g_{t-1}^2(y)} \ \ \ \ (\text{\tiny{Jensen's inequality}}) \\
& \leq \frac{C^2}{\alpha^2} \ \ \ \ (\text{\tiny{\Lemma{lemma:g_upperbound} }}).\\
\end{aligned}
\end{equation*}
\end{proof}
Now, let us continue with the random sequence $\{h_t\}_{t \in \N}$. We will start by defining the following operators for $t \in \N$ and $M_t \in \R_{+} \cup \{+\infty\}$ which will allow us to upper-bound the norm of $h_t$ with respect to a random variable with a bounded variance: 
%%%
\begin{equation}{\label{eq:bar_Lt_Rt}}
\begin{aligned}
 \overline{L}_t &= \Ind{ \abs{h_{t-1}(x_t)} \leq M_t } L_t \ \ \ \ \ \ \ \ \ \  \overline{\ELP}= \Expec{\overline{L}_t \,|\, \Xi_{t-1}}  \\
\underline{L}_t & = \Ind{ \abs{h_{t-1}(x_t)} \geq M_t } L_t   \ \ \ \ \ \ \ \ \ \ \underline{\ELP}= \Expec{\underline{L}_t   \,|\, \Xi_{t-1}} \\ 
 \overline{R}_t &= \Ind{ \abs{h_{t-1}(x'_t)} \leq M_t } R_t \ \ \ \ \ \ \ \  \ \  \overline{\ELQ}= \Expec{\overline{R}_t   \,|\, \Xi_{t-1}}   \\ 
\underline{R}_t &= \Ind{ \abs{h_{t-1}(x'_t)} \geq M_t } R_t  \ \ \ \ \ \ \ \  \ \ \underline{\ELQ}= \Expec{\underline{R}_t   \,|\, \Xi_{t-1}} \\
\end{aligned}
\end{equation}
%%%
Notice that:
%%%
\begin{equation}{\label{eq:Rt_Lt}}
L_t=\underline{L}_t+\overline{L}_t \ \ \ \ R_t=\underline{R}_t+\overline{R}_t.
\end{equation}
%%%
For $t \in \N$, define the following variables: 
%%%%%
\begin{equation}{\label{eq:barh_t_k_t}}
\begin{aligned}
     \overline{h}_t &:=  \left[ \idH - \eta_t \left( (1-\alpha) \overline{L}_t+ \alpha \overline{R}_t  + \lambda_t \idH \right) \right] h_{t-1} + \eta_t \left( K(x'_t,\cdot) - LR_t r^{\alpha} \right) + \eta_t  \left( \mathbf{LR} - LR_t \right)  g_{t-1} \\ & = h_t + \eta_t \left[ (1-\alpha) \underline{L}_t  + \alpha \underline{R}_t  \right] h_{t-1} \ \ \ \ (\text{\tiny{\Eq{eq:g_t_h_t}} and \Eq{eq:Rt_Lt} })  
     \\
k_t &:= \overline{h}_t - (1-\eta_t \lambda_t) h_{t-1} \\
&= h_t - \left[ \idH- \eta_t \left( (1-\alpha) \underline{L}_t  + \alpha \underline{R}_t  + \lambda_t \idH \right) \right] h_{t-1} \\   &= \eta_t \left[ -\left[ (1-\alpha)  \overline{L}_t  + \alpha \overline{R}_t \right] h_{t-1}
 + K(x'_t,\cdot) + \ELPQ(g_{t-1}) - LR_t \left( r^{\alpha} + g_{t-1} \right)
\right] \ \ \ \ (\text{\tiny{\Eq{eq:g_t_h_t}} and \Eq{eq:Rt_Lt} }).  
\end{aligned}
\end{equation}
%%%%%
%
\begin{lemma}{\label{lemma:h_component}}
Assume $t_0^{\theta} \geq 2a \left(b + 2 C^2\right)$. For all $t \in  \N$, $M_t \in \R_{+} \cup \{+\infty\}$, we have: 
\begin{equation}
\expec \left[\norm{\overline{h}_{t}}^2_{\Hilbert} \,|\, \Xi_{t-1} \right] \leq (1 - \eta_t \lambda_t)^2 \norm{h_{t-1}}^2_{\Hilbert} + 2 C^2 \eta_t^2 \left( \frac{2+\alpha^2 }{\alpha^2} \right).
\end{equation}
In particular, assume that $\frac{1}{2} \leq \theta \leq 1$ and $t_0^{\theta} \geq \max\{ 2ab,2\gamma, \gamma + \frac{2\theta-1}{\gamma} \}$ where  $\gamma= ab- (\theta-\frac{1}{2}) >0$, and fix  $B_1= a C \sqrt{2 \left( \frac{2+\alpha^2}{\alpha^2 \gamma} \right)}$. Then $\norm{h_{t-1}}_{\Hilbert} \geq  B_1 \bar{t}^{\frac{1}{2}- \theta}$ implies: 
\begin{equation}
\bar{t}^{\theta-\frac{1}{2}}\expec \left[\norm{\overline{h}_{t}}_{\Hilbert}  \,|\, \Xi_{t-1} \right] \leq (\bar{t}-1)^{\theta-\frac{1}{2}} \norm{h_{t-1}}_{\Hilbert}.
\end{equation}
\end{lemma}
\begin{proof}
For all $t \in \N$, let us define the random variable: 
\begin{equation*}
    \zeta_t := \left[ (1-\alpha) \overline{\ELP} + \alpha \overline{\ELQ}
    - \left( (1-\alpha)  \overline{L}_t + \alpha \overline{R}_t) \right) 
\right] h_{t-1}+ (\ELPQ-LR_t) g_{t-1} + (K(x'_t,\cdot) -LR_t r^{\alpha}).
\end{equation*}
Given the definition of $\overline{h}_t$ in \Eq{eq:barh_t_k_t}: 
%%%
\begin{equation}{\label{eq:h_bar_t}}
\overline{h}_t = \left[ \idH - \eta_t ((1-\alpha) \overline{\ELP} + \alpha \overline{\ELQ} + \lambda_t \idH ) \right] h_{t-1} + \eta_t \zeta_t.
\end{equation}
%%%
The independence of the incoming observations $(x_t,x'_t)$ and the fact $h_{t-1},g_{t-1}$ are $\Xi_{t-1}-$ measurable lead to:  %%%
\begin{equation}{\label{eq:con_expectation_psi}}
\begin{aligned}
\expec[\zeta_t \,|\, \Xi_{t-1} ] &=\expec[K(x'_t,\cdot) - LR_t r^{\alpha}] = \expec[K(x'_t,\cdot) - \left( (1-\alpha) r^{\alpha}(x_t) K(x_t,\cdot) + \alpha r^{\alpha}(x'_t) K(x'_t,\cdot) \right)  \,|\, \Xi_{t-1} ]  \\
& = \ExpecA{K(x',\cdot)} - \ExpecAlpha{r^{\alpha}(y) K(y,\cdot) }=0.
\end{aligned}
\end{equation}
%%%
\Eq{eq:h_bar_t} and the last observation implies: 
%%%
\begin{equation}{\label{eq:norm_zetat}}
\expec[\norm{\overline{h}_t}^2_{\Hilbert} \,|\, \Xi_{t-1}]= \norm{\left[ \idH - \eta_t ((1-\alpha)\overline{\ELP} + \alpha\overline{\ELQ}  + \lambda_t \idH) \right] h_{t-1}}^2_{\Hilbert} + \eta^2_t \expec[ \norm{\zeta_t}^2_{\Hilbert}  \,|\, \Xi_{t-1}].
\end{equation}
%%%
The next step is to upperbound each of the components in \Eq{eq:norm_zetat}. For the first element of the sum we have: 
%%%%
\begin{equation}
\begin{aligned}
 &  \norm{\left[\idH - \eta_t ((1-\alpha)\overline{\ELP} + \alpha\overline{\ELQ} + \lambda_t \idH) \right] h_{t-1}}^2_{\Hilbert} 
\\ & = (1- \lambda_t \eta_t)^2 \norm{h_{t-1}} ^2_{\Hilbert}   - 2 \eta_t (1-\eta_t \lambda_t) \dotH{ \left[(1-\alpha)\overline{\ELP} + \alpha \overline{\ELQ} \right] h_{t-1}}{h_{t-1}}  + \eta^2_t \norm{  \left[(1-\alpha)\overline{\ELP} + \alpha \overline{\ELQ} \right] h_{t-1}} ^2_{\Hilbert} \\
& =(1- \lambda_t \eta_t)^2 \norm{h_{t-1}} ^2_{\Hilbert} \\
& -2 \eta_t  (1-\eta_t \lambda_t)  \expec \left[ (1-\alpha) h^2_{t-1}(x_t)  \Ind{ \abs{h_{t-1}(x_t)} \leq M_t }
+ \alpha h^2_{t-1}(x'_t)  \Ind{ \abs{h_{t-1}(x'_t)} \leq M_t } 
  \,|\, \Xi_{t-1}  \right] \\
  &  +  \eta^2_t \norm{  \left[(1-\alpha)\overline{\ELP} + \alpha \overline{\ELQ} \right] h_{t-1}} ^2_{\Hilbert}  \ \ \ \ (\text{\tiny{\Eq{eq:bar_Lt_Rt} and the first point of \Eq{RKHS_properties}}}). \\
\end{aligned}
\end{equation}
%%%%
We can upper bound the last term in the previous expression by: 
%%%%
\begin{equation}
\begin{aligned}
& \eta^2_t \norm{  \left[(1-\alpha)\overline{\ELP} + \alpha \overline{\ELQ} \right] h_{t-1}} ^2_{\Hilbert}  
\\
& \leq  \eta^2_t 
 \expec \left[ \norm{  (1-\alpha) h_{t-1}(x_t) K(x_t, \cdot) \Ind{ \abs{h_{t-1}(x_t)} \leq M_t } 
 + \alpha h_{t-1}(x'_t) K(x'_t, \cdot)  \Ind{ \abs{h_{t-1}(x'_t)} \leq M_t }}^2_{\Hilbert}  \,|\, \Xi_{t-1} \right] \\ 
& \leq \eta^2_t   \expec \Big[ 
(1-\alpha)^2 K(x_t,x_t) h^2_{t-1}(x_t) \Ind{ \abs{h_{t-1}(x_t)} \leq M_t } \\
& + 2 (1-\alpha) \alpha K(x_t,x'_t) \abs{h_{t-1}(x_t) h_{t-1}(x'_t)} \Ind{ \abs{h_{t-1}(x_t)},\abs{h_{t-1}(x'_t)} \leq M_t } \\ & + \alpha^2 K(x'_t,x'_t) h^2_{t-1}(x'_t) \Ind{ \abs{h_{t-1}(x'_t)} \leq M_t }   \,|\, \Xi_{t-1} \Big]  \\
& \leq \eta^2_t C^2 \expec \Big[ 
(1-\alpha)^2 h^2_{t-1}(x_t) \Ind{ \abs{h_{t-1}(x_t)} \leq M_t }  + 2 (1-\alpha) \alpha  \abs{h_{t-1}(x_t) h_{t-1}(x'_t)} \Ind{ \abs{h_{t-1}(x_t)},\abs{h_{t-1}(x'_t)} \leq M_t } \\
& + \alpha^2 h^2_{t-1}(x'_t) \Ind{ \abs{h_{t-1}(x'_t)} \leq M_t }   \,|\, \Xi_{t-1} \Big] \ \ (\text{\tiny{\Assumption{ass:kernel_map_upperbound} }}) \\
& = \eta_t^2 C^2\expec \Big[ \left( (1-\alpha) h_{t-1}(x_t) \Ind{ \abs{h_{t-1}(x_t)} \leq M_t }  + \alpha h_{t-1}(x'_t) \Ind{ \abs{h_{t-1}(x'_t)} \leq M_t } \right)^2  \,|\, \Xi_{t-1} \Big] \\
& \leq \eta^2_t C^2 \expec \Big[ (1-\alpha) h^2_{t-1}(x_t) \Ind{ \abs{h_{t-1}(x_t)} \leq M_t } + \alpha  h^2_{t-1}(x'_t) \Ind{ \abs{h_{t-1}(x'_t)} \leq M_t }    \,|\, \Xi_{t-1} \Big] (\text{\tiny{Jensen's inequality}})
\end{aligned}
\end{equation}
%%%
Let us continue with the second component of \Eq{eq:norm_zetat}.  We start with an upperbound for the following term: 
%%%
\begin{equation}{\label{eq:L_R_t_g_t_1}}
\begin{aligned}
& \expec \Big[ \norm{LR_t g_{t-1}}^2_{\Hilbert}  \,|\, \Xi_{t-1} \Big] \\
& = \expec \Big[ \dotH{(1-\alpha)K(x_t,\cdot) g_{t-1}(x_t) + \alpha K(x'_t,\cdot) g_{t-1}(x'_t)}{(1-\alpha)K(x_t,\cdot) g_{t-1}(x_t) + \alpha K(x'_t,\cdot) g_{t-1}(x'_t)}  \,|\, \Xi_{t-1} \Big]    (\tiny{\text{\Eq{RKHS_properties}}}) \\
& = \expec \Big[ (1-\alpha)^2 K(x_t,x_t) g^2_{t-1}(x_t) +2 (1-\alpha) \alpha K(x_t,x'_t) g_{t-1}(x_t)  g_{t-1}(x'_t) \\
& + \alpha^2  K(x'_t,x'_t)  g^2_{t-1}(x'_t)  \,|\, \Xi_{t-1} \Big]  \ \ \ \ (\tiny{\text{The first point of \Eq{RKHS_properties}}}) \\
& \leq C^2 \expec \Big[ \left[ (1-\alpha) g_{t-1}(x_t)  + \alpha  g_{t-1}(x'_t) \right]^2  \,|\, \Xi_{t-1} \Big] \ \ \ \ (\text{\tiny{\Assumption{ass:kernel_map_upperbound} }}) \\
& \leq  C^2 \expec \Big[ (1-\alpha) g^2_{t-1}(x_t) + \alpha g^2_{t-1}(x'_t)  \,|\, \Xi_{t-1} \Big] \ \ \ \  (\text{\tiny{Jensen's inequality given $0\leq \alpha \leq 1$}})
 \\ 
& \leq  C^2 \ExpecAlpha{g^2_{t-1}(y)} \leq \frac{C^2}{\alpha^2} \ \ \ \ (\text{\tiny{\Lemma{lemma:g_upperbound} }}).
\end{aligned}
\end{equation}
%%%
Then,the second component of \Eq{eq:norm_zetat} satisfies the inequality : 
%%%
\begin{equation}{\label{eq:norm_psi_t}}
\begin{aligned}
& \expec \left[ \norm{\zeta_t}^2_{\Hilbert}    \,|\, \Xi_{t-1}  \right] \\
& \leq 2 \expec \left[  \norm{ \left[ \left((1-\alpha)\overline{\ELP} + \alpha \overline{\ELQ} \right)  -  \left((1-\alpha) \overline{L}_t +  \alpha \overline{R}_t \right) \right] h_{t-1} + \left(\ELPQ - LR_t) g_{t-1} \right) }^2_{\Hilbert}  \,|\, \Xi_{t-1}    \right] 
\\ &+ 2 \expec \left[  \norm{K(x'_t,\cdot) - LR_t r^{\alpha}}^2_{\Hilbert}
 \,|\, \Xi_{t-1}    \right] \ \ \ \  (\text{\tiny{$2 \dotH{a}{b} \leq  \norm{ a}^2_{\Hilbert} +  \norm{b}^2_{\Hilbert}$}})  \\
& \leq  2 \expec \Big[    \norm{ \left((1-\alpha) \overline{L}_t +  \alpha \overline{R}_t \right) h_{t-1} + LR_t g_{t-1}  }^2_{\Hilbert}\,|\, \Xi_{t-1}     \Big] \\
&+ 2 \expec \Big[   \norm{K(x'_t,\cdot)}^2_{\Hilbert}  \,|\, \Xi_{t-1}  \Big] \ \
 (\text{\tiny{After developing the norms and taking conditional expectations}})  \\
&  \leq 2 \Big[ 2  \expec \Big[    \norm{ \left((1-\alpha) \overline{L}_t +  \alpha \overline{R}_t \right) h_{t-1}}^2_{\Hilbert}   \,|\, \Xi_{t-1} \Big]  + 2 \expec \Big[  \norm{LR_t g_{t-1}}^2_{\Hilbert}  \,|\, \Xi_{t-1}     \Big]   +    C^2   \Big] (\text{\tiny{$2 \dotH{a}{b} \leq  \norm{ a}^2_{\Hilbert} +  \norm{b}^2_{\Hilbert}  $ and \Assumption{ass:kernel_map_upperbound} }} )
\\
&  \leq 2 C^2 \left[ 2 \expec \Big[ (1-\alpha) h^2_{t-1}(x_t) \Ind{ \abs{h_{t-1}(x_t)}\leq M_t } + \alpha  h^2_{t-1}(x'_t) \Ind{ \abs{h_{t-1}(x'_t)}\leq M_t }    \,|\, \Xi_{t-1} \Big]  +\frac{2}{\alpha^2} + 1 \right] 
 \ \ \  \ (\text{\tiny{\Eq{eq:L_R_t_g_t_1}}}).
\end{aligned}
\end{equation}
%%%
By putting together \Exprs{eq:norm_zetat}-\ref{eq:norm_psi_t}: 
%%%
\begin{equation}{\label{eq:h_first_inequality}}
\begin{aligned}
&\expec[\norm{\overline{h}_t}^2_{\Hilbert} \,|\, \Xi_{t-1}] \\ & \leq (1- \lambda_t \eta_t)^2 \norm{h_{t-1}} ^2_{\Hilbert}  \\
& - \eta_t \left( 2 - 2 \eta_t \lambda_t - 5 C^2 \eta_t
 \right)  \expec \left[ (1-\alpha) h^2_{t-1}(x_t)  \Ind{ \abs{h_{t-1}(x_t)} \leq M_t }
+ \alpha h^2_{t-1}(x'_t)  \Ind{ \abs{h_{t-1}(x'_t)} \leq M_t }  \,|\, \Xi_{t-1}  \right] 
\\
& 
 \ \ \ \ \ + 2 C^2 \eta^2_t \left( \frac{2+\alpha^2}{\alpha^2} \right) \\
 & \leq (1- \lambda_t \eta_t)^2 \norm{h_{t-1}} ^2_{\Hilbert}  +2 C^2 \eta^2_t \left( \frac{2+\alpha^2}{\alpha^2} \right).
\end{aligned}
\end{equation}
%%%
The last is a consequence of 
the hypothesis $t^{\theta}_0 \geq 2 a (b+2C^2)$ implies  $2- 2 \eta_t \lambda_t - 5 C^2 \eta_t \geq 0$ for all $t \in \N$. Thus we obtain the first inequality of \Lemma{lemma:h_component}. 

The second point of that lemma depends on the following inequality: 
%%%
\begin{equation}{\label{eq:inequality_abt}}
(1- \frac{1}{\bar{t}})^{1- 2 \theta} (1-\frac{ab}{\bar{t}})^2 \leq (1-\frac{\gamma}{\bar{t}}),
\end{equation}
%%%
where  $\theta \in [\frac{1}{2},1]$ and $\bar{t}=t+t_0$ and $t_0 \geq \max\{2ab, 2\gamma,\gamma + \frac{2\theta -1}{\gamma}\}$, where $\gamma=ab- \frac{(2\theta-1)}{2}$. 

Inequality \ref{eq:inequality_abt} can be verified as follows: 
\begin{equation*}
\begin{aligned}
\log \left[ (1- \frac{1}{\bar{t}})^{1- 2 \theta} (1-\frac{ab}{\bar{t}})^2(1-\frac{\gamma}{\bar{t}})^{-1} \right] & \leq -( 2 \theta-1) \log \left( 1- \frac{1}{\bar{t}} \right) + 2 \log \left( 1-\frac{ab}{\bar{t}} \right) - \log \left( 1-\frac{\gamma}{\bar{t}} \right)  \\
&\leq \frac{(2 \theta-1)}{\bar{t}} + \frac{(2 \theta-1)}{\bar{t}^2} - \frac{2 ab }{\bar{t}} + \frac{\gamma}{\bar{t}}+ \frac{\gamma^2}{\bar{t}^2} \\
& = - \frac{\gamma}{\bar{t}} + \frac{2 \theta-1 + \gamma^2}{\bar{t}^2} = \frac{\gamma}{t} \left(\frac{\gamma}{\bar{t}}-1\right) + \frac{2\theta-1}{t^2} \leq 0 ,
\end{aligned}   
\end{equation*}
where for the second equality we have used the inequalities $\log(1-x) \leq -x$ for all $x \in [0,1]$ and $\log(1-x) \geq -x -x^2$ for $x \in [0, \frac{1}{2} ]$. 

By using Inequalities \ref{eq:h_first_inequality} and \ref{eq:inequality_abt}, we can verify: 
\begin{equation}
\begin{aligned}
\expec \left[ \left(1- \frac{1}{\bar{t}}
\right)^{1-2 \theta} \norm{\bar{h}_t}_{\Hilbert}^2 - \norm{h_{t-1}}_{\Hilbert}^2 \Big\vert  \Xi_{t-1}  \right] & \leq \left(1- \frac{1}{\bar{t}}
\right)^{1-2 \theta}  ( 1- \frac{ab}{\bar{t}})^2 \norm{h_{t-1}}_{\Hilbert}^2 - \norm{h_{t-1}}_{\Hilbert}^2  +   \frac{2 a^2 C^2}{\bar{t}^{2 \theta}} \left( \frac{2+\alpha^2}{\alpha^2} \right) \\
& \leq - \frac{\gamma}{\bar{t}}\norm{h_{t-1}}_{\Hilbert}^2 +  \frac{2 a^2C^2}{\bar{t}^{2 \theta}} \left( \frac{2+\alpha^2}{\alpha^2} \right) (\text{\tiny{Thanks to the hypothesis on $t_0$ and $\theta$}} ).
\end{aligned}
\end{equation}
If $\norm{h_{t-1}}_{\Hilbert} \geq  B_1 \bar{t}^{\frac{1}{2}- \theta}$, we can verify: 
\begin{equation}
\begin{aligned}
\expec \left[ \left(1- \frac{1}{\bar{t}}
\right)^{1-2 \theta} \norm{\bar{h}_t}_{\Hilbert}^2 \Big\vert  \Xi_{t-1}  \right] 
 & \leq \left(1 - \frac{\gamma}{\bar{t}} \right)
\norm{h_{t-1}}_{\Hilbert}^2 +  \frac{2 a^2C^2}{\bar{t}^{2 \theta}} \left( \frac{2+\alpha^2}{\alpha^2} \right)  \\
& \leq \left(1 - \frac{\gamma}{\bar{t}} \right)
\norm{h_{t-1}}_{\Hilbert}^2 + \frac{ B_1^2 \gamma}{\bar{t}^{2 \theta}} \leq
\norm{h_{t-1}}_{\Hilbert}^2.
\end{aligned}
\end{equation}
After applying Jensen's inequality, we obtain the desired result. 
\end{proof}
\begin{lemma}{\label{lemma:k_component}}
Assume $t_0^{\theta} \geq a (C^2 +b)$ and $t^{1-\theta}_{0} \geq b( \alpha (M_t+1)+3)$; then,
\begin{equation}
\norm{k_t}_{\Hilbert} \leq  \frac{Cab^{-1}}{ \alpha\bar{t}^{2\theta-1}} \ \ \ \  \text{and} \ \ \ \  \expec \left[ \norm{k_t}^2_{\Hilbert} \Big\vert  \Xi_{t-1}  \right] \leq \frac{ 3 \eta_t^2 C^2 }{\alpha^2} \left[ \alpha^2(M_t^2+1) +1 \right].
\end{equation}
\end{lemma}
\begin{proof}
Let us start with the following inequality: 
\begin{equation}{\label{eq:LR_t_h_t}}
\begin{aligned}
& \norm{ \left( (1-\alpha) \overline{L}_t + \alpha \overline{R}_t \right)
h_{t-1}}^2_{\Hilbert} 
\\ &\leq (1-\alpha)^2 K(x_t,x_t) h^2_{t-1}(x_t) \Ind{ \abs{h_{t-1}(x_t)} \leq M_t }  \\ 
&+ 2 \alpha (1-\alpha) K(x_t,x'_t) \abs{h_{t-1}(x_t)}\abs{h_{t-1}(x'_t)} \Ind{ \abs{h_{t-1}(x'_t)},\abs{h_{t-1}(x'_t)} \leq M_t }  \\
& + \alpha^2 K(x'_t,x'_t) h^2_{t-1}(x'_t) \Ind{ \abs{h_{t-1}(x'_t)} \leq M_t } \ \ \ \ (\text{\tiny{First point of \ref{RKHS_properties} and after developing the norm}})   \\
& \leq C^2 \left( (1-\alpha)  h_{t-1}(x_t) \Ind{ \abs{h_{t-1}(x_t)}  \leq M_t } + \alpha  h_{t-1}(x'_t) \Ind{ \abs{h_{t-1}(x'_t)} \leq M_t }   \right)^2 \ \ \ \ (\text{\tiny{\Assumption{ass:kernel_map_upperbound} }}) \\
& \leq C^2 \left[  (1-\alpha) h^2_{t-1}(x_t) \Ind{ \abs{h_{t-1}(x_t)}  \leq M_t } + \alpha  h^2_{t-1}(x'_t) \Ind{ \abs{h_{t-1}(x'_t)}  \leq M_t }  \right] \ \ \ \ (\text{\tiny{Jensen's inequality given $0\leq \alpha \leq 1$}}) \\
& \leq C^2 M^2_t.
\end{aligned}
\end{equation}
Moreover by exploiting the hypothesis $r^{\alpha} \in \Hilbert$ and after following the same line of argumentation as in the previous inequality, we verify: 
\begin{equation}{\label{Eq:LR_t_r_g_H_L2}}
\begin{aligned}
\norm{LR_t(r^{\alpha}+g_{t-1})}^2_{\Hilbert} &\leq C^2 \left( (1-\alpha) (r^{\alpha}+g_{t-1})^2(x_t)  + \alpha (r^{\alpha}+g_{t-1})^2(x'_t)  \right) \\
&\leq C^2 (1-\alpha) \ExpecN{(r^{\alpha}+g_{t-1})^2(x)} + \alpha \ExpecA{(r^{\alpha}+g_{t-1})^2(x')} \\
& = C^2 \norm{r^{\alpha}+g_{t-1}}^2_{\ltwo},
\end{aligned}
\end{equation}
which implies: 
\begin{equation}{\label{Eq:LR_t_r_g}}
\norm{LR_t(r^{\alpha}+g_{t-1})}_{\Hilbert} \leq C \norm{r^{\alpha}+g_{t-1}}_{\ltwo}\leq C\left(  \norm{r^{\alpha}}_{\ltwo} +\norm{g_{t-1}}_{\ltwo} \right) \leq \frac{2C}{\alpha}.
\end{equation}
The last line is a consequence of the first point of \Lemma{lemma:g_upperbound} and the fact $r^{\alpha} \leq \frac{1}{\alpha}$.

We can, then, upperbound the norm $\norm{k_t}_{\Hilbert}$: 
\begin{equation*}
\begin{aligned}
\norm{k_t}_{\Hilbert}&= \eta_t  \left[ \norm{ \left( (1-\alpha) \overline{L}_t + \alpha \overline{R}_t \right)
h_{t-1}}_{\Hilbert}  + \norm{K(x_t,\cdot)}_{\Hilbert}  + \norm{\ELPQ g_{t-1}}_{\Hilbert} +  \norm{LR_t\left( g_{t-1} +r^{\alpha}\right)}_{\Hilbert} \right] \\
& = \eta_t   \left[ C(M_t +1 + \frac{1}{\alpha} +\frac{2}{\alpha} ) \right] \ \ \ \ (\text{\tiny{\Eq{eq:LR_t_h_t}, \Assumption{ass:kernel_map_upperbound}, \Lemma{lemma:LR_g_t}, \Eq{Eq:LR_t_r_g} }})   \\
& \leq \frac{C \eta_t}{\alpha} \left[  \alpha(M_t+1) + 3 \right] \\ & \leq \frac{\eta_t C}{\alpha\lambda_t} \ \ \  (\text{\tiny{Hypothesis $t_0^{1-\theta}  \geq b \left[ \alpha \left(M_t+1\right) +3 \right]$ implies $\alpha(M_t+1)+3 \leq \frac{1}{\lambda_t}$}})  \\   & \leq \frac{ C ab^{-1}}{\alpha\bar{t}^{2\theta-1}}. 
\end{aligned}
\end{equation*}
On the other hand, we obtain for the expected norm: 
\begin{equation*}
\begin{aligned}
\expec \left[  \norm{k_t}^2_{\Hilbert}  \Big\vert  \Xi_{t-1} \right] &\leq 3 \eta_t^2 \Big[ \expec \left[  \norm{ \left( (1-\alpha)\overline{L}_t  + \alpha \overline{R}_t \right) h_{t-1}}^2_{\Hilbert}  \Big\vert  \Xi_{t-1} \right] +   \expec \left[  \norm{ K(x'_t,\cdot) - LR_t \left( r^{\alpha} \right) }^2_{\Hilbert}  \Big\vert  \Xi_{t-1} \right] \\
& +   \expec \left[  \norm{ \left( \ELPQ - LR_t \right)  g_{t-1} }^2_{\Hilbert}  \Big\vert  \Xi_{t-1} \right]  \Big] \ \ \  \ (\text{\tiny{$2 \dotH{a}{b} \leq  \norm{ a}^2_{\Hilbert} +  \norm{b}^2_{\Hilbert}  $}} ) \\
& \leq 3 \eta_t^2 \Bigg[ \expec \left[ \norm{ \left( (1-\alpha)\overline{L}_t  + \alpha \overline{R}_t \right) h_{t-1}}^2_{\Hilbert}  \Big\vert  \Xi_{t-1}  \right] + \expec \left[  \norm{ K(x'_t,\cdot) }^2_{\Hilbert}  \Big\vert  \Xi_{t-1} \right] \\
&  + \expec \left[  \norm{   LR_t   g_{t-1} }^2_{\Hilbert}  \Big\vert  \Xi_{t-1} \right]  \Bigg] \ \ \ \  (\text{\tiny{After developing the norms and taking  conditional expectations}})  \\
& \leq  3 \eta_t^2 C^2 \left[  M_t^2 +  1+  \norm{ g_{t-1} }^2_{\ltwo} \right]  \ \ \ \  (\text{\tiny{\Eq{eq:LR_t_h_t}} and the same line of reasoning that in \Eq{Eq:LR_t_r_g_H_L2}}) \\ 
& \leq \frac{ 3 \eta_t^2 C^2 }{\alpha^2} \left[ \alpha^2(M_t^2+1) +1 \right]  \\
\end{aligned}
\end{equation*}
\end{proof}
\begin{lemma}{\label{lemma:upperbound_h}}
    For all $t \in \N$, assume $M_t \geq \frac{2C^2 a b^{-1} \bar{t}^{1-2\theta} }{\alpha}$, $t_0^{\theta} \geq 2a (C^2+b)$ and $t_0^{1-\theta} \geq b \left( \alpha(M_t+1) +3 \right)$, then 
    \begin{equation}
    \norm{h_t}_{\Hilbert} \leq \norm{\overline{h}_t}_{\Hilbert}.
    \end{equation}
\end{lemma}
\begin{proof}

We start with the following inequality that relates $\bar{h}_{t-1}$ and $k_t$, take $x \in \mathcal{X}$ such that $h_{t-1}(x) \geq M_t$, then: 
\begin{equation}{\label{eq:upperbound_h}}
\begin{aligned}
\bar{h}_t(x) &= (1- \lambda_t \eta_t) h_{t-1} (x) + k_t(x) \ \ \ \ (\text{\tiny{\Eq{eq:barh_t_k_t} }})   \\
&\geq (1- \lambda_t \eta_t) h_{t-1} (x)  -  \frac{C^2 ab^{-1}}{\alpha\bar{t}^{2\theta-1}} \ \ \ \ (\text{\tiny{Lemma\ref{lemma:k_component} }})  \\
& \geq (1- \lambda_t \eta_t) h_{t-1} (x)  - \frac{1}{2} h_{t-1} (x) \ \ \ \ (\text{\tiny{As a consequence of  $M_t \geq \frac{2C^2 a b^{-1} \bar{t}^{1-2\theta} }{\alpha}$ }}) \\
& \geq   C^2 \eta_t \frac{h_{t-1}(x)}{2}.
\end{aligned}
\end{equation}

The last identity is a consequence  of  assumption $t_0^{\theta} \geq 2 a (C^2+b)$ which implies $1-\eta_t \lambda_t - \frac{C^2 \eta_t}{2} \geq \frac{1}{2}$. 

Suppose we have $h_{t-1}(x_t) \geq M_t$ and $h_{t-1}(x'_t) \geq M_t$, then we have:
\begin{equation}{\label{eq:h_t_bar_h_t}}
h_t = \bar{h}_t - \eta_t \left[ (1-\alpha) L_t  + \alpha R_t  \right] h_{t-1} \ \ (\text{\tiny{\Eq{eq:barh_t_k_t}}}) 
\end{equation}
\begin{equation}{\label{eq:normh_M}}
\begin{aligned}
\norm{h_t}^2_{\Hilbert}&=\dotH{h_t}{h_t}=\norm{\bar{h}_t}^2_{\Hilbert} - 2 \eta_t \dotH{\bar{h}_t}{\left[ (1-\alpha) L_t  + \alpha R_t \right] h_{t-1} } + \eta^2_t\norm{ \left[  (1-\alpha) L_t  + \alpha R_t \right] h_{t-1} }^2_{\Hilbert} \\
& \leq  \norm{\bar{h}_t}^2_{\Hilbert} - 2 \eta_t \Big( (1-\alpha) \bar{h}_t(x_t) h_{t-1}(x_t) + \alpha \bar{h}_t(x'_t) h_{t-1}(x'_t) \Big) + \eta^2_t \left[ (1-\alpha)^2 \norm{L_t h_{t-1} }^2_{\Hilbert} \right. \\
&  \left. + \alpha^2 \norm{R_t h_{t-1} }^2_{\Hilbert} + 2 (1-\alpha) \alpha K(x_t,x'_t) h_{t-1}(x_t)  h_{t-1}(x'_t)   \right]  
 \\
& \leq \norm{\bar{h}_t}^2_{\Hilbert}
-  \eta_t^2 C^2 \left[ (1-\alpha) h^2_{t-1}(x_t) +  \alpha h^2_{t-1}(x'_t)\right] + 2 \eta_t^2 C^2 \alpha(1-\alpha)\abs{h_{t-1}(x_t)h_{t-1}(x'_t)}  \\
& + \eta^2_t \left[ (1-\alpha)^2 \norm{L_t h_{t-1} }^2_{\Hilbert}  + \alpha^2 \norm{R_t h_{t-1} }^2_{\Hilbert}   \right] \ \ \ \  (\text{\tiny{\Eq{eq:h_t_bar_h_t} and \Eq{eq:upperbound_h}}}) \\
& \leq \norm{\bar{h}_t}^2_{\Hilbert}
-  \eta_t^2 C^2 \left[ (1-\alpha) h_{t-1}(x_t) +  \alpha h_{t-1}(x'_t)\right]^2 + 2 \eta_t^2 C^2 \alpha(1-\alpha)\abs{h_{t-1}(x_t)h_{t-1}(x'_t)}\\ &+ \eta^2_t \left[ (1-\alpha)^2 \norm{L_t h_{t-1} }^2_{\Hilbert}  + \alpha^2 \norm{R_t h_{t-1} }^2_{\Hilbert}   \right] \ \ \ \  (\text{\tiny{Jensen's inequality given $0\leq \alpha \leq 1$}}) \\
& \leq \norm{\bar{h}_t}^2_{\Hilbert}
-  \eta_t^2 C^2 \left[ (1-\alpha)^2 h^2_{t-1}(x_t) +  \alpha^2 h^2_{t-1}(x'_t) \right] + \eta^2_t \left[ (1-\alpha)^2 \norm{L_t h_{t-1} }^2_{\Hilbert}  + \alpha^2 \norm{R_t h_{t-1} }^2_{\Hilbert}   \right] \\
& \leq\norm{\bar{h}_t}^2_{\Hilbert}
+ \eta^2_t (1-\alpha)^2 h^2_{t-1}(x_t) \left [ K(x_t,x_t) -C^2 \right] + \eta^2_t \alpha^2 
h^2_{t-1}(x'_t) \left [ K(x'_t,x'_t) -C^2 \right] \\
& \leq \norm{\bar{h}_t}^2_{\Hilbert} \ \ \ \ (\text{\tiny{\Assumption{ass:kernel_map_upperbound} }}).
\end{aligned}
\end{equation}
Let us continue with the case  $h_{t-1}(x_t) \geq M_t$  and $h_{t-1}(x'_t) < M_t$, then we have: 
\begin{equation*}
h_t= \bar{h}_t - (1-\alpha) \eta_t L_t h_{t-1} \ \ (\text{\tiny{\Eq{eq:barh_t_k_t}}}).
\end{equation*}
Then by following the same line of argumentation than in the previous point we get: 
\begin{equation*}
\begin{aligned}
\norm{h_t}^2_{\Hilbert} &= \norm{\bar{h}_t}^2_{\Hilbert} - 2 (1-\alpha) \eta_t \dotH{L_t h_{t-1}}{\bar{h}_t} + \eta_t^2 (1-\alpha)^2 \norm{L_t h_{t-1}}^2_{\Hilbert} \\
&  =\norm{\bar{h}_t}^2_{\Hilbert} - 2 (1-\alpha) \eta_t h_{t-1}(x_t) \bar{h}_t(x_t) + \eta_t^2(1-\alpha)^2 h^2_{t-1}(x_t) K(x_t,x_t) \\
& \leq \norm{\bar{h}_t}^2_{\Hilbert} + \eta_t^2(1-\alpha)^2 h^2_{t-1}(x_t) K(x_t,x_t) - (1-\alpha) C^2 \eta^2_t h^2_{t-1}(x_t) \ \ \ \ (\text{\tiny{ $\bar{h}_t(x_t) \geq C^2 \eta_t \frac{h_{t-1}(x_t)}{2}$ }}) \\
& \leq \norm{\bar{h}_t}^2_{\Hilbert} + \eta_t^2(1-\alpha)^2 h^2_{t-1}(x_t) K(x_t,x_t) - (1-\alpha)^2 C^2 \eta^2_t h^2_{t-1}(x_t) \\
&= \norm{\bar{h}_t}^2_{\Hilbert} + (1-\alpha)^2 \eta^2_t h^2_{t-1}(x_t) \left( K(x_t,x_t) - C^2 \right) 
 \\ & \leq \norm{\bar{h}_t}^2_{\Hilbert} \ \ \ \ (\text{\tiny{\Assumption{RKHS_properties}}}).
\end{aligned}
\end{equation*}
The case $h_{t-1}(x_t) < M_t$  and $h_{t-1}(x'_t) \geq M_t$ can be solve in a symmetric way. Finally, for $h_{t-1}(x_t) < M_t$  and $h_{t-1}(x'_t) < M_t$ , the inequality follows directly.
\end{proof}
\begin{lemma}{\label{lemma:upperbound_h_2}}
    Assume $\theta \in [\frac{1}{2},1]$,$t_0 \geq 3$, $b=a^{-1}$ , $t_0^{\theta} \geq 2+ 4C^2 a$ and $t_0^{1-\theta} \geq 8b$, Then, with probability at least $1-\delta$:
    \begin{equation}
    \sup_{0 \leq k \leq t} \norm{h_k}_{\Hilbert}(k+t_0+1)^{\theta - \frac{1}{2}} \leq \ \frac{ aC}{\alpha} \left[ 5 at^{\frac{1}{2}-\theta}_0  + (14Ca^2+18)\sqrt{\log{(\bar{t})}} \right]  \log\left(\frac{2}{\delta} \right):=B_{t,\delta}.
    \end{equation}
\end{lemma}
\begin{proof}
Let us start by verifying that the hypothesis of \Lemmas{lemma:h_component}, \ref{lemma:k_component}, and \ref{lemma:upperbound_h}. 

First  $t^{\theta}_0 \geq 2+ 4C^2 a = 2 a(b + 2C^2)$ and $\gamma=ab - \left( \theta - \frac{1}{2} \right) \in [\frac{1}{2},1]$, where we 
we have used the assumption $ab=1$.
The assumption $t_0 \geq 3$ implies
$t_0 \geq \max \left(2ab,2 \gamma, \gamma + \frac{2\theta-1}{\gamma} \right)$.
 
Finally, if we fix $M_t  = \frac{2C^2ab^{-1} \bar{t}^{1-2\theta}}{\alpha}$
, the fact that $t_0^{\theta} \geq 4C^2 a $ and $t_0^{1-\theta} \geq 8b$ implies:
%
%\begin{equation*}{\label{eq:h_xi}}
\begin{equation*}
\begin{aligned}
    t_0^{1-\theta} \geq \frac{ t_0^{1-\theta}}{2} +  \frac{ t_0^{1-\theta}}{2}= \frac{ t_0^{1-\theta}}{2} +  \frac{ t_0^{1-2 \theta} (t_0^{\theta})}{2} 
    \geq b \left( 4  + 2 C^2 ab^{-1} t_0^{1-2\theta} \right) \geq b \left( 4 + \alpha M_t  \right) \geq b \left( 3 + \alpha (M_t+1) \right).
\end{aligned}
\end{equation*}
Then the required assumptions are satisfied. 

Take $i \in \N$, if $\norm{h_{i-1}}_{\Hilbert} \geq B_1 \bar{t}^{\frac{1}{2}-\theta}$, where $B_1 = a C\sqrt{2 \left(\frac{2+\alpha^2}{\alpha^2 \gamma}\right)}$ , then: 
\begin{equation}{\label{eq:h_xi}}
\begin{aligned}
\norm{h_i}_{\Hilbert} & \leq \norm {\bar{h}_i}_{\Hilbert} \ \ \ \ (\text{\tiny{Lemma\ref{lemma:upperbound_h} }})  \\
& = \norm {\bar{h}_i}_{\Hilbert} - \expec \left[  \norm {\bar{h}_i}_{\Hilbert} \,|\, \Xi_{i-1}  \right] + \expec \left[  \norm {\bar{h}_i}_{\Hilbert} \,|\, \Xi_{i-1}  \right]  \\
&\leq \xi_i + \left[1 -\frac{1}{i+t_0} \right ]^{\theta-\frac{1}{2}} \norm{h_{i-1}}_{\Hilbert} \ \ \ \ (\text{\tiny{Second part of Lemma\ref{lemma:h_component} }}),
\end{aligned}
\end{equation}
where $\xi_i:=\norm{\bar{h}_i}_{\Hilbert} - \expec \left[ \norm{\bar{h}_i}_{\Hilbert} \,|\,  \Xi_{i-1}  \right]$.

Notice that the stochastic process $\{\xi_k\}_{k \in \N}$ defines a martingale difference sequence, which additionally satisfies the following inequalities: 
\begin{equation}{\label{eq:xi_first_inequality}}
\begin{aligned}
\abs{\xi_i}  &\leq \norm{\bar{h}_i-\expec \left[  \bar{h}_i \,|\,  \Xi_{i-1} \right]}_{\Hilbert} \\  
&=  \norm{ k_{i}
-\expec \left[ k_{i} \,|\,  \Xi_{i-1} \right]}_{\Hilbert} \ \ \ \ (\text{\tiny{\Eq{eq:barh_t_k_t} }})
\\
&\leq \norm{ k_{i}}_{\Hilbert} + \norm{\expec \left[ k_{i} \,|\,  \Xi_{i-1} \right]}_{\Hilbert} 
\\ & \leq \norm{ k_{i}}_{\Hilbert} + \expec  \left[ \norm{k_{i}}_{\Hilbert} \,|\,  \Xi_{i-1} \right] \ \ \ \ (\text{\tiny{Jensen's inequality}}) \\
&\leq \frac{2 C ab^{-1}}{\alpha(i+t_0)^{2\theta-1}} \ \ \ \ (\text{\tiny{\Lemma{lemma:k_component}}}).
\end{aligned}
\end{equation}
In a similar manner we can verify:
\begin{equation}{\label{eq:xi_second_inequality}}
\begin{aligned}
\expec  \left[ \xi_i^2   \,|\, \Xi_{i-1} \right] & = \expec  \left[\norm{ k_{i}
-\expec \left[ k_{i} \,|\,  \Xi_{i-1} \right] }_{\Hilbert}^2  \,|\, \Xi_{i-1} \right] \\
&\leq \expec  \left[  \norm{k_i}_{\Hilbert}^2  \,|\, \Xi_{i-1} \right]
-  \norm{\expec \left[ k_{i} \,|\,  \Xi_{i-1} \right]}^2_{\Hilbert} \\
&\leq \expec  \left[  \norm{k_i}_{\Hilbert}^2  \,|\, \Xi_{i-1} \right] \\
&\leq  \frac{ 3 \eta_i^2 C^2 }{\alpha^2} \left[ \alpha^2(M_i^2+1) +1 \right]   \ \ \ \ (\text{\tiny{\Lemma{lemma:k_component}}}) \\ & = \frac{3\eta_i^2 C^2}{\alpha^2} \left(\alpha^2 ( \frac{4 C^4 a^2b^{-2} (t+i)^{2(1-2\theta)}}{\alpha^2}+1)+1 \right) \\
& \leq \frac{3 \eta_i^2 C^2}{\alpha^2}( 4C^2 a^2 b^{-2} + 2) \\
&=  \frac{12 \eta_i^2 C^2}{\alpha^2}( C^2 a^2 b^{-2} + \frac{1}{2}) \\
& \leq \frac{12 \eta_i^2 C^2}{\alpha^2} (Cab^{-1}+1)^2
\end{aligned}
\end{equation}
Notice that the sequence $\left\{(k+t_0)^{\theta-\frac{1}{2}} \xi_{k}\right\}$ defines a difference martingale as well which satisfies the inequalities:
\begin{equation}{\label{eq:norm_zeta_M}}
\begin{aligned}
\abs{(i+t_0)^{\theta-\frac{1}{2}} \xi_{i}} & \leq \frac{2C a^2 (i+t_0)^{\frac{1}{2}-\theta}}{\alpha} \ \ \ \ (\text{\tiny{\Eq{eq:xi_first_inequality} }})  \leq  \frac{2C a^2 t_0^{\frac{1}{2}-\theta}}{\alpha}
\end{aligned}
\end{equation}
\begin{equation}{\label{eq:norm_zeta_sigma}}
\begin{aligned}
\sum_{k=1}^{t} \expec  \left[ \left( (k+t_0)^{\theta-\frac{1}{2}}\xi_k \right)^2 \,|\, \Xi_{k-1} \right] & \leq \frac{12 a^2 C^2}{\alpha^2} (Cab^{-1}+1)^2  \sum_{k=1}^t (k+t_0)^{-1} \ \ \ \ (\text{\tiny{\Eq{eq:xi_second_inequality} }})  \\
& \leq\frac{12 a^2 C^2}{\alpha^2} (Cab^{-1}+1)^2\log(1+\frac{t}{t_0}).
\end{aligned}
\end{equation}
Let us define the term: 
\begin{equation}{\label{eq:nu_definition}}
\nu_i = \sum_{j=1}^{i} \xi_j (j+t_0)^{\theta-\frac{1}{2}} \Ind{\norm{h_{j-1}}_{\Hilbert} \geq B_1 (j+t_0)^{\frac{1}{2}-\theta}}.
\end{equation}
Inequalities \ref{eq:norm_zeta_M} and  \ref{eq:norm_zeta_sigma} imply the hypothesis of preposition A.3 in \cite{Tarres2014} (\Lemma{lemma:propositionA3}) are satisfied. Then the probability of the event $\Delta$,
$P(\Delta) \geq 1-\delta$, where: 
\begin{equation}{\label{eq:def_event_delta}}
\begin{aligned}
\Delta= \Big\{ \sup_{1 \leq i \leq t} \abs{\nu_i}  &\leq 2 \left(  \frac{2C a^2 t_0^{\frac{1}{2}-\theta}}{3\alpha} + \frac{2\sqrt{3} a C}{\alpha} (Cab^{-1}+1) \sqrt{ \log(1+\frac{t}{t_0})} \right) \log{\left(\frac{2}{\delta}\right)} \\
&\leq \frac{4 a C}{\alpha} \left( \frac{  a t_0^{\frac{1}{2}-\theta}}{3} + \sqrt{3}(Cab^{-1}+1)\sqrt{\log\left(1+\frac{t}{t_0}\right)} \right) \log{\left(\frac{2}{\delta}\right)} \\
& = \frac{4 a C}{\alpha} \left( \frac{  a t_0^{\frac{1}{2}-\theta}}{3} +\sqrt{3}(Ca^2+1) \sqrt{\log\left(1+\frac{t}{t_0}\right)} \right) \log{\left(\frac{2}{\delta}\right)} \Big\}.
\end{aligned}
\end{equation}
Assume that the event $\Delta$ holds and  let $u_k$ for all $k \in \N$ be: 
\begin{equation}
u_k= \norm{h_k} (k+t_0)^{\theta-\frac{1}{2}}.
\end{equation}
For all the elements $k \leq t $, let: 
\begin{equation}{\label{eq:def_m}}
m= \max\{ j \leq k : \norm{h_j}_{\Hilbert} < B_1(j+t_0+1)^{\frac{1}{2}-\theta} \}.
\end{equation}
If $m<k$, then:
\begin{equation}{\label{eq:upperbound_um}}
\begin{aligned}
u_{m+1} &\leq \left[ \left(\frac{m+t_0}{m+1+t_0} \right)^{\theta-\frac{1}{2}}  \norm{h_{m}}_{\Hilbert}  +   \abs{\xi_{m+1}} \right]  (m+1+t_0)^{\theta-\frac{1}{2}}  \ \ \ \ (\text{\tiny{\Eq{eq:h_xi} }})  \\
& <  \left[ \left(\frac{m+t_0}{m+t_0+1} \right)^{\theta-\frac{1}{2}} (m+t_0+1)^{\frac{1}{2}-\theta} B_1 +\frac{2 C a^2}{\alpha}(m+1+t_0)^{1-2 \theta}
\right] (m+1+t_0)^{\theta-\frac{1}{2}} \ \ (\text{\tiny{\Eq{eq:xi_first_inequality} and \Eq{eq:def_m}  }})  \\
&   \leq   a C \sqrt{2 \left( \frac{2+\alpha^2}{\alpha^2 \gamma} \right)}  + \frac{2C a^2}{\alpha} t_0^{\frac{1}{2} -\theta} \leq  a C \sqrt{ 4 \left(\frac{2+\alpha^2}{\alpha^2} \right)}  + \frac{2C a^2}{\alpha} t_0^{\frac{1}{2} -\theta} \ \ \ \ (\text{\tiny{ $\gamma \in [\frac{1}{2},1]$}  })  
\\ & \leq \frac{aC}{\alpha} \left( \sqrt{4(2+\alpha^2)} + 2 a t_0^{\frac{1}{2}-\theta}   \right) \leq \frac{2aC}{\alpha} \left( \sqrt{3} + a t_0^{\frac{1}{2}-\theta}   \right).
\end{aligned}
\end{equation}
Given \Expr{eq:h_xi}, we can verify:

%%%
\begin{equation}
(i+t_0)^{\theta-\frac{1}{2}} \norm{h_i}_{\Hilbert} \leq (i+t_0)^{\theta-\frac{1}{2}} \xi_i + (i-1+t_0)^{\theta-\frac{1}{2}} \norm{h_{i-1}}_{\Hilbert} 
\end{equation}
%%%

Then by recursion and given \ref{eq:nu_definition}, we get: 
\begin{equation}
u_k \leq u_{m+1} + \nu_{k} -\nu_{m+1}.
\end{equation}
For $\delta$ sufficiently small, we have: 
\begin{equation}
\begin{aligned}
u_k &\leq u_{m+1}+ \abs{\nu_k} + \abs{\nu_{m+1}}\\
& \leq \frac{2aC}{\alpha} \left[ \left(\frac{4}{3}+1
\right) at^{\frac{1}{2}-\theta}_0 + \sqrt{3} + 4 \sqrt{3} (Ca^2+1)
\sqrt{\log \left(1+\frac{t}{t_0}\right)}\right]   \log \left(\frac{2}{\delta} \right)  \ \ \ \ (\text{\tiny{\Eq{eq:def_event_delta}  and \Eq{eq:upperbound_um} }})  \\
&=\frac{aC}{\alpha} \left[ \left(\frac{14}{3}
\right) at^{\frac{1}{2}-\theta}_0 + 2\sqrt{3} + 8 \sqrt{3} (Ca^2+1)
\sqrt{\log \left(1+\frac{t}{t_0}\right)}\right]   \log \left(\frac{2}{\delta} \right)  \ \ \ \ (\text{\tiny{\Eq{eq:def_event_delta}  }})  \\
& \leq  \frac{aC}{\alpha} \left[ 5 at^{\frac{1}{2}-\theta}_0 + 4 + 14(Ca^2+1)\sqrt{\log{\left(\frac{\bar{t}}{t_0}\right)}} \right]  \log\left(\frac{2}{\delta} \right) \\
&\leq  \frac{ aC}{\alpha} \left[ 5 at^{\frac{1}{2}-\theta}_0  + (14Ca^2+18)\sqrt{\log{(\bar{t})}} \right]  \log\left(\frac{2}{\delta} \right),
\end{aligned}
\end{equation}
where the fact that $t_0 \geq 3$ implies
$\sqrt{\log{(t+t_0)}} \geq 1$, meaning $ 4 + 14(Ca^2+1)\sqrt{\log{\left(\frac{\bar{t}}{t_0}\right)}} \leq (14Ca^2+18)\sqrt{\log{(\bar{t})}}  \log\left(\frac{2}{\delta} \right)$.
\end{proof}

\inlinetitle{Proof of \Theorem{thm:sample_error}}{.}~
\begin{proof}
Let us start with a basic inequality that will be useful during the proof, suppose $f$ is 
$\Xi_{j-1}$ measurable then we have: 
%%%%
\begin{equation}{\label{eq:norm_H_norm_L_t_1}}
\begin{aligned}
&\expec \left[ \norm{LR_j f}^2_{\Hilbert}  \,|\, \Xi_{j-1} \right] \\
& =  \expec \left[ \dotH{(1-\alpha)f(x_j)K(x_j,\cdot) + \alpha f(x'_j) K(x'_j,\cdot) }{(1-\alpha)f(x_j)K(x_j,\cdot) + \alpha f(x'_j) K(x'_j,\cdot)}  \,|\, \Xi_{j-1} \right]  \\ 
& = \expec \left[ (1-\alpha)^2 K(x_j,x_j)f^2(x_j)+ 2 (1-\alpha) \alpha K(x_j,x'_j) f(x_j)f(x'_j) + \alpha^2 K(x'_j,x'_j) f^2(x'_j)  \,|\, \Xi_{j-1} \right] \\
& \leq  C^2 \expec \left[ \left((1-\alpha) f(x_j)+\alpha f(x'_j) \right)^2 \,|\, \Xi_{j-1} \right]  \ \ \ \ (\text{\tiny{\Assumption{ass:kernel_map_upperbound} }})  \\
& \leq  C^2 \expec  \left[(1-\alpha) f^2(x_j)+\alpha f^2(x'_j)  \,|\, \Xi_{j-1} \right]  (\text{\tiny{Jensen's inequality }})
= C^2 \norm{f}^2_{\ltwo}.
\end{aligned}
\end{equation}
%%%
Notice that for $f \in \Hilbert$:
\begin{equation}{\label{eq:norm_operator_LR}}
\begin{aligned}
    \norm{LR_j f}_{\Hilbert} &\leq (1-\alpha) \abs{\dotH{K(x_j,\cdot)}{f}} \norm{K(x_j,\cdot)}_{\Hilbert} + \alpha \abs{\dotH{K(x'_j,\cdot)}{f}} \norm{K(x_j,\cdot)}_{\Hilbert} \\
   & \leq \left[
    (1-\alpha) \norm{K(x_j,\cdot)}^2_{\Hilbert}  + \alpha \norm{K(x'_j,\cdot)}^2_{\Hilbert} \right]\norm{f}_{\Hilbert}  \ \ \ \ (\text{\tiny{Cauchy–Schwarz inequality}}) \\
    &\leq C^2 \norm{f}_{\Hilbert}
\end{aligned}
\end{equation}
%%%
which implies $\norm{LR_j} \leq C^2$.

Fix $t\in \N$, $\delta \in [0,1]$, and let 
%%%
\begin{equation}
B_{t,\delta}=   \frac{ aC}{\alpha} \left[ 5 at^{\frac{1}{2}-\theta}_0  + (14Ca^2+18)\sqrt{\log{(\bar{t})}} \right]  \log\left(\frac{2}{\delta} \right).
\end{equation}
%%%
And the following stochastic process.
%%%
\begin{equation}
\Upsilon_j= \eta_j \bar{\Pi}_{j+1}^t \epsilon_j \Ind{\norm{h_{j-1}}_{\Hilbert} (j+t_0)^{\theta-\frac{1}{2}} \leq B_{t,\delta} }.
\end{equation}
%%%
Verifying that the sequence $\{\Upsilon_j\}_{j\in \N}$ is a difference martingale is easy. The idea to finish the proof is to apply \Lemma{lemma:propositionA3} to the sequence $\Upsilon_j$, this means, we should  show the existence of $M>0$ and $\sigma^2 >0$ such that $\norm{\Upsilon_j}_{\ltwo} \leq M$ and $\sum_{j=1}^t \expec \left[\norm{\Upsilon_j}^2_{\ltwo} \,|\, \Xi_{j-1}\right ] \leq \sigma^2$. 

Let us start by identifying $\sigma^2$. Suppose $\norm{h_{j-1}}_{\Hilbert} \leq B_{t,\delta} (j+t_0)^{\frac{1}{2}-\theta}$, then by using the decomposition  $f_j= r^{\alpha} + g_{j} + h_{j} $ and the inequalities stated in \Lemma{lemma:g_upperbound} we have: 
%%%
\begin{equation}{\label{eq:norm_noise_2}}
\begin{aligned}
\expec \left[ \norm{\epsilon_j}^2_{\Hilbert} \,|\, \Xi_{j-1}\right] &= \expec \left[ \norm{ ( \CO -  LR_j) f_{j-1} + K(x'_j,\cdot) -  \CO r^{\alpha}}^2_{\Hilbert} \,|\, \Xi_{j-1} \right]  \\
& \leq  \expec \left[\norm{K(x'_j,\cdot) - LR_j f_{j-1}}^2_{\Hilbert}   \,|\, \Xi_{j-1}\right]  \ \  (\text{\tiny{After developing the norm and taking conditional expectations}})  \\
& \leq  \expec \left[\norm{  K(x'_j,\cdot) - LR_j r^{\alpha}  - LR_j  g_{j-1} - LR_j h_{j-1}}^2_{\Hilbert}   \,|\, \Xi_{j-1}\right] \\ 
& \leq 4 \expec\left[ \norm{K(x'_j,\cdot)}^2_{\Hilbert} + \norm{LR_j r^{\alpha} }^2_{\Hilbert} +\norm{LR_j  g_{j-1}  }^2_{\Hilbert} + \norm{LR_j h_{j-1}}^2_{\Hilbert} \,|\, \Xi_{j-1}\right] \ \  (\text{\tiny{$2 \dotH{a}{b} \leq  \norm{ a}^2_{\Hilbert} +  \norm{b}^2_{\Hilbert}  $}} )  \\
& \leq 4 \left[ C^2 + \frac{C^2}{\alpha^2} + \frac{C^2}{\alpha^2} +\expec\left[ \norm{LR_j}^2 \norm{h_{j-1}}^2_{\Hilbert} \,|\, \Xi_{j-1}\right] \right]  \  \ (\text{\tiny{\Assumption{ass:kernel_map_upperbound} },  \Lemma{lemma:g_upperbound} and \Eq{eq:norm_H_norm_L_t_1}})  \\
&  \leq 4 C^2 \left[1+\frac{2}{\alpha^2} + C^2 (j+t_0)^{1-2\theta} B^2_{t,\delta}\right] \ \ (\text{\tiny{ \Eq{eq:norm_operator_LR}}}) \\
&:= B'_{j,t,\delta}.   \\
\end{aligned}
\end{equation}
%%%
If we use the isometry of the operator $\CO^{\frac{1}{2}}: \ltwo \rightarrow \Hilbert$ and the fact that it is a compact operator then there exists an orthonormal eigensystem $(\mu_{k},\phi_{k})_{k \in \N}$ of $\CO$, where $\{\mu_{k}\}_{k \in \N}$  are strictly positive and arranged in decreasing order (see Proposition 2.2 in \cite{Dieuleveut2017}). Let us define $a_i= \eta_i \lambda_i + \eta_i \mu_j$. 

First notice that for $j \leq t$, given Eq \ref{eq:pi_operator}: 
\begin{equation}{\label{eq:norm_pi_operator_upperbound}}
\begin{aligned}
\norm{\bar{\Pi}_{j+1}^t}  &\leq
\norm{\prod_{i=j+1}^{t} (\idH- \eta_i \mathbf{A}_i)}  \leq  \prod_{i=j+1}^{t} (1-\eta_i (\lambda_i+ \mu_{j})) = \prod_{i=j+1}^{t} (1-a_i).
\end{aligned}
\end{equation}
%%%
Then we can verify the following inequality: 
%%%
\begin{equation}
\begin{aligned}
\sum_{j=1}^t \expec \left[ \norm{\Upsilon_j}^2_{\ltwo}  \,|\, \Xi_{j-1}\right] &= \sum_{j=1}^t \expec \left[ \norm{ \CO^{\frac{1}{2} }\Upsilon_j}^2_{\Hilbert}  \,|\, \Xi_{j-1}\right]  = \sum_{j=1}^{t} \eta^2_j 
\expec \left[ \norm{ \CO^{\frac{1}{2}}  \bar{\Pi}_{j+1}^t\epsilon_j}^2_{\Hilbert}  \,|\, \Xi_{j-1}\right] \\
& = \sum_{j=1}^t \left( \eta_j^2 \norm{\bar{\Pi}_{j+1}^t \CO \bar{\Pi}_{j+1}^t} \right) \expec \left[ \norm{\epsilon_j}^2_{\Hilbert}  \,|\, \Xi_{j-1}\right] \\
& \leq \sum_{j=1}^t \eta_j^2 B'_{j,t,\delta} \norm{\bar{\Pi}_{j+1}^t \CO \bar{\Pi}_{j+1}^t}  \ \ \ \ (\text{\tiny{\Eq{eq:norm_noise_2}}}) \\
& \leq \sup_{\{\mu_k: k \in \N\}} \sum_{j=1}^t \eta^2_j B'_{j,t,\delta} \mu_{k} \prod_{i=j+1}^t (1-a_i)^2 
\ \ \ \ (\text{\tiny{\Eq{eq:norm_pi_operator_upperbound}}}) 
\\
&= \sup_{\{\mu_k: k \in \N\}}  \left[ \sup_{j} \eta_j B'_{j,t,\delta} \prod_{i=j+1}^t (1-a_i) \right] \left[ \sum_{j=1}^{t} \eta_j  \mu_k \prod_{i=j+1}^t (1-a_i)  \right].
\end{aligned}
\end{equation}
%%%%
For a large value of $t_0$ we can verify for the first element of the product: 
%%%%%%%%%%%
\begin{equation}{\label{eq:upperboundB_ai}}
\begin{aligned}
\sup_{j} \eta_j B'_{j,t,\delta} \prod_{i=j+1}^t (1-a_i) 
&\leq \sup_{j} \eta_j B'_{j,t,\delta} \prod_{i=j+1}^t (1-\eta_i \lambda_i)  \\ & \leq 
\sup_{j} \eta_j B'_{j,t,\delta} \prod_{i=j+1}^t (1-\eta_i \lambda_i) \\
&\leq 4a C^2 \sup_{j}  \frac{j+t_0}{\bar{t}}\left( \frac{1+\frac{2}{\alpha^2}}{(j+t_0)^{\theta}} + \frac{ C^2 B^2_{t,\delta}}{(j+t_0)^{3 \theta-1}} \right) \\
& \leq \frac{4 a C^2}{\bar{t}^{\theta}} \left( 1+ \frac{2}{\alpha^2} + \frac{ C^2 B^2_{t,\delta}}{\bar{t}^{(2 \theta-1)}}  \right).
\end{aligned}
\end{equation}
%%%%%%%%%%%
For the second element of the product, we can verify: 
%%%%%%%%%%%%
\begin{equation}
{\label{eq:sum_eigenvalues}}
\begin{aligned}
\sum_{j=1}^t \eta_j \mu_k \prod_{i=j+1}^t (1- a_i) &\leq \sum_{j=1}^t (1 -(1-\eta_j \mu_k)) \prod_{i=j+1}^t( 1- \eta_i \mu_k) \\
&= 1 - \prod_{i=1}^{t} (1-\eta_i \mu_k) \leq 1.
\end{aligned}
\end{equation}
%%%%%%%%%%
By combining both bounds \ref{eq:upperboundB_ai} and \ref{eq:sum_eigenvalues}  we obtain: 
%%%%%%%%%%%%
\begin{equation}{\label{eq:sigma_t}}
\sum_{j=1}^t \expec \left[ \norm{\Upsilon_j}^2_{\ltwo} \,|\, \Xi_{j-1}\right ] \leq \frac{4aC^2}{\bar{t}^{\theta}} \left(1+\frac{2}{\alpha^2} + \frac{C^2B^2_{t,\delta}}{\bar{t}^{2\theta-1}}  \right)
\end{equation}
%%%%%%%%%%%
Now we will identify $M$. Let us start by upperbounding the following term via \Lemma{lemma:g_upperbound}: 
%%%%%%
\begin{equation}{\label{eq:chi_M}}
\begin{aligned}
\norm{K(x'_j,\cdot) - LR_j(f_{j-1})}_{\Hilbert} &= 
\norm{K(x'_j,\cdot) - LR_j(r^{\alpha} + g_{j-1} +h_{j-1})}_{\Hilbert} \\
& \leq \norm{K(x'_j,\cdot)}_{\Hilbert}+ \norm{LR_j(r^{\alpha} + g_{j-1})}_{\Hilbert} +  \norm{LR_j(h_{j-1})}_{\Hilbert} \\
& \leq C + \norm{LR_j} \norm{(r^{\alpha} + g_{j-1})}_{\Hilbert} + \norm{LR_j} \norm{h_{j-1}}_{\Hilbert} \\
& \leq  C + \frac{3 C^2}{\alpha\sqrt{\lambda_{j-1}}}  + C^2 B_{t,\delta} (j+t_0)^{\frac{1}{2}-\theta}  (\text{\tiny{\Lemma{lemma:g_upperbound}, \Assumption{ass:kernel_map_upperbound} and \Eq{eq:norm_operator_LR}}}) \\
& := C_{j,t,\delta}.
\end{aligned}
\end{equation}
By using the fact that $f_{j-1}$ is $\Xi_{j-1}$ measurable we have: 
%%%
\begin{equation}{\label{eq:chi_j}}
\begin{aligned}
\epsilon_j &=( \CO - LR_{j})f_{j-1} +  K(x'_j,\cdot) - \CO r^{\alpha}\\
&=( \CO - LR_{j})f_{j-1} + K(x'_j,\cdot) - \ExpecAlpha{K(y,\cdot) r^{\alpha}(y)} (\text{\tiny{\Expr{eq:covariance_operator} }}) \\
& = K(x'_j,\cdot) - LR_j(f_{j-1}) - \expec \left[ K(x'_j,\cdot) -LR_j(f_{j-1})   \,|\, \Xi_{j-1} \right].
\end{aligned}
\end{equation}
%%%%
where the last inequality is due to the definition of the likelihood-ratio and the fact that the observations are independent in time. 

We upperbound the norm $\norm{\epsilon_{j}}_{\Hilbert}$ by: 
%%%
\begin{equation}{\label{eq:upperbound_chi_j}}
\begin{aligned}
\norm{\epsilon_j}_{\Hilbert} &= \norm{K(x'_j,\cdot) - LR_j(f_{j-1}) - \expec \left[ K(x'_j,\cdot) -LR_j(f_{j-1})   \,|\, \Xi_{j-1} \right]}_{\Hilbert} \ \ \ \ (\text{\tiny{\Eq{eq:chi_j})}})  \\
& \leq \norm{K(x'_j,\cdot) - LR_j(f_{j-1})}_{\Hilbert} + \norm{\expec \left[ K(x'_j,\cdot) -LR_j(f_{j-1})   \,|\, \Xi_{j-1} \right]}_{\Hilbert} \\
&\leq \norm{K(x'_j,\cdot) - LR_j(f_{j-1})}_{\Hilbert} + \expec \left[ \norm{ K(x'_j,\cdot) -LR_j(f_{j-1})}_{\Hilbert}  \,|\, \Xi_{j-1} \right] \ \ \ \ (\text{\tiny{Jensen's inequality}})   \\
& \leq 2 \left( C + \frac{3 C^2 }{\sqrt{\lambda_{j-1}}}  + C^2 B_{t,\delta} (j+t_0)^{\frac{1}{2}-\theta} \right)  
\ \ \ \ (\text{\tiny{\Eq{eq:chi_M}}}) \\ &= 2 C_{j,t,\delta}.
\end{aligned}
\end{equation}
%%%
Therefore by using the hypothesis $t_0^{\theta} \geq 2+ 4 C^2 a$, we can deduce $\frac{C \sqrt{a}}{\bar{t}^{\frac{\theta}{2}}} \leq 1$ and: 
%%%
\begin{equation}{\label{eq:M}}
\begin{aligned}
\norm{\Upsilon_j}_{\ltwo} &=
\norm{ \CO^{\frac{1}{2}}\Upsilon_j}_{\Hilbert}  \leq \eta_j \norm{\CO^{\frac{1}{2}}\Pi_{j+1}^{t} \epsilon_j}_{\Hilbert}  \leq  2 \eta_j  C_{j,t,\delta} \norm{\Pi_{j+1}^{t}\CO\Pi_{j+1}^{t}}^{\frac{1}{2}} \ \ \ \ (\text{\tiny{\Eq{eq:upperbound_chi_j}}}) \\
& \leq 2 C \sup_{j} \eta_j C_{j,t,\delta} \prod_{i=j+1}^t (1- \eta_j \lambda_j)  \ \ \ \ (\text{\tiny{As $\norm{\CO} \leq C^2$}})  \\
& = 2 a C^2  \sup_{j} \frac{j+t_0}{\bar{t}} \left(\frac{1}{(j+t_0)^{\theta}} + \frac{ 3 C\sqrt{a}}{\alpha (j+t_0)^{\left(\frac{3\theta-1}{2}\right)}}  + \frac{C B_{t,\delta}}{(j+t_0)^{2\theta-\frac{1}{2}}}  \right) \ \ \ \ (\text{\tiny{\Eq{eq:chi_M}}}) \\
& \leq  2 a C^2 \left(  \frac{1}{\bar{t}^\theta} +\frac{3C\sqrt{a}}{\alpha \bar{t}^{\left(\frac{3\theta-1}{2}\right)}}  + \frac{C B_{t,\delta}}{\bar{t}^{2\theta-\frac{1}{2}}} \right) \\
& \leq  \frac{2 a C^2}{\bar{t}^{\theta}} \left( 1 +\frac{3C\sqrt{a}}{\alpha \bar{t}^{\left(\frac{\theta-1}{2}\right)}}  + \frac{C B_{t,\delta}}{\bar{t}^{\theta-\frac{1}{2}}} \right) \\
& \leq  \frac{2 \sqrt{a} C}{\bar{t}^\frac{\theta}{2}} \left(  1 + \frac{3C^2 a}{\alpha \bar{t}^{\left(\theta-\frac{1}{2}\right)}} + \frac{ C B_{t,\delta}}{\bar{t}^{\theta-\frac{1}{2}}} \right) \ \ \ \ (\text{\tiny{The hypothesis $t_0^{\theta} \geq 2+ 4C^2a$ implies $\frac{t_0^{\frac{\theta}{2}}}{C \sqrt{a}}\geq 1$}})
\end{aligned}
\end{equation}
%%% 
Both Inequalities\,\ref{eq:sigma_t} and \ref{eq:M} imply that the hypothesis of Proposition A.3 is satisfied. Then the inequality holds with a probability at least 
$1-\delta$ for $\delta \in (0,1)$ we have: 
%%%
\begin{equation*}
\begin{aligned}
\sup_{1 \leq k \leq t} \norm{\sum_{j=1}^k \Upsilon_j}_{\ltwo} & \leq 4 \frac{\sqrt{a}C}{\bar{t}^{\frac{\theta}{2}}}  \left(  \frac{1}{3} + \frac{C^2a}{\alpha \bar{t}^{\left(\theta-\frac{1}{2}\right)} } + \frac{C B_{t,\delta}}{3 \bar{t}^{\left(\theta-\frac{1}{2}\right)}} + 1 + \frac{\sqrt{2}}{\alpha} + \frac{ C B_{t,\delta}}{\bar{t}^{\left(\theta-\frac{1}{2}\right)} }   \right) \log \left(\frac{2}{\delta} \right)  \\
& =  4 \frac{\sqrt{a}C}{\bar{t}^{\frac{\theta}{2}}}  \left( \frac{4}{3}  + \frac{\sqrt{2}}{\alpha}+ \left(\frac{4C}{3}\right)
\frac{ B_{t,\delta}}{\bar{t}^{\left(\theta-\frac{1}{2}\right)}}
+ \frac{C^2 a}{ \alpha\bar{t}^{\left(\theta-\frac{1}{2}\right)}}
\right) \log \left(\frac{2}{\delta} \right) \\
& \leq 
8 \frac{\sqrt{a}C}{\bar{t}^{\frac{\theta}{2}}}   \left(  \frac{1+\alpha}{\alpha}  \right) \log \left(\frac{2}{\delta} \right)  +  \left( \frac{16C^3}{3 \alpha} \right) \left[  5  a^{\frac{5}{2}} + a^{\frac{3}{2}}(14Ca^2+18)\sqrt{\log{(\bar{t})}}   \right] \frac{\log^2 \left(\frac{2}{\delta} \right)}{\bar{t}^{\frac{3\theta-1}{2}}} \\
& + \left( \frac{4 a^{\frac{3}{2}} C^3}{ \alpha} \right) \frac{\log \left(\frac{2}{\delta}\right)  }{\bar{t}^{\frac{3\theta-1}{2}}}   \\ 
&  \leq 
8 \frac{\sqrt{a}C}{\bar{t}^{\frac{\theta}{2}}}   \left(  \frac{1+\alpha}{\alpha}  \right) \log \left(\frac{2}{\delta} \right)  + \left( \frac{16C^3}{3 \alpha} \right) \left[   5   a^{\frac{5}{2}} + a^{\frac{7}{2}}(14C+2)\sqrt{\log{(\bar{t})}}  \right] \frac{\log^2 \left(\frac{2}{\delta} \right)}{\bar{t}^{\frac{3\theta-1}{2}}} \\
& + \left( \frac{4 a^{\frac{7}{2}} C^3}{ \alpha} \right) \frac{\log \left(\frac{2}{\delta}\right)  }{\bar{t}^{\frac{3\theta-1}{2}}} \ \ \ \ (\text{\tiny{The hypothesis $a \geq 4$}}) \\
& \leq 
8 \frac{\sqrt{a}C}{\bar{t}^{\frac{\theta}{2}}}   \left(  \frac{1+\alpha}{\alpha}  \right) \log \left(\frac{2}{\delta} \right)  +  \left[   \frac{32 a^{\frac{5}{2}} C^3}{\alpha} + \frac{4a^{\frac{7}{2}}C^3}{\alpha}(20C+4)\sqrt{\log{(\bar{t})}}  \right] \frac{\log^2 \left(\frac{2}{\delta} \right)}{\bar{t}^{\frac{3\theta-1}{2}}} \\
& + \left( \frac{4 a^{\frac{7}{2}} C^3}{ \alpha} \right) \frac{\log \left(\frac{2}{\delta}\right)  }{\bar{t}^{\frac{3\theta-1}{2}}} \ \ \ \ \\
& \leq 
8 \frac{\sqrt{a}C}{\bar{t}^{\frac{\theta}{2}}}   \left(  \frac{1+\alpha}{\alpha}  \right) \log \left(\frac{2}{\delta} \right)  +  \left[  \frac{32 a^{\frac{5}{2}} C^3}{\alpha} + \frac{4a^{\frac{7}{2}}C^3}{\alpha}(20C+4+ \frac{1}{\log(2)})\sqrt{\log{(\bar{t})}}  \right] \frac{\log^2 \left(\frac{2}{\delta} \right)}{\bar{t}^{\frac{3\theta-1}{2}}} \\
& \leq 
8 \frac{\sqrt{a}C}{\bar{t}^{\frac{\theta}{2}}}   \left(  \frac{1+\alpha}{\alpha}  \right) \log \left(\frac{2}{\delta} \right)  +  \left[  \frac{32 a^{\frac{5}{2}} C^3}{\alpha} + \frac{8a^{\frac{7}{2}}C^3}{\alpha}(10C+3)\sqrt{\log{(\bar{t})}}  \right] \frac{\log^2 \left(\frac{2}{\delta} \right)}{\bar{t}^{\frac{3\theta-1}{2}}} \\
& \leq   \frac{\sqrt{a}  B_4}{\bar{t}^{\frac{\theta}{2}}} \log \left(\frac{2}{\delta} \right) + \left[ B_5 a^{\frac{5}{2}} + B_6 a^{\frac{7}{2}} \sqrt{\log{\bar{t}}} \right] \frac{\log^2 \left(\frac{2}{\delta} \right)}{\bar{t}^{\frac{3\theta-1}{2}}}.
\end{aligned}
\end{equation*}
%%%
Where we have used the assumption $t_0 \geq 3$ and the constants $B_4,B_5,B_6$. 
%%%
\begin{equation*}
8 C \left(  \frac{1+\alpha}{\alpha} \right) \leq  \frac{16 C}{\alpha}=B_4  \ \ \ B_5= \frac{32 C^3}{\alpha} \ \ \ B_6= \frac{8 C^3(10C+3)}{\alpha}.
\end{equation*}
%%%
\end{proof}

\inlinetitle{Upperbound for $\mathcal{E}'_{\text{sample}}(t)=\norm{\sum_{j=1}^{t} \eta_j\Pi_{j+1}^t (A_jf_{\lambda_j}-b_j) }_{\Hilbert}$}{.}

In this section, we focus on developing the required components for proving \Theorem{thm:sample_error_Hilbert}.
\begin{lemma}{\label{lemma:norm_gradient_descent}} We have:
\begin{enumerate}
\item $\norm{A_t f_{\lambda_t}-b_t}_{\Hilbert} \leq  \frac{1}{\sqrt{\lambda_t}} \left( \frac{C+1}{\alpha} + C\right),$ \text{ if } $t_0^{1-\theta} \geq b$; 
\item $\expec \left[ \norm{A_t f_{\lambda_t} - b_t}^2_{\Hilbert} \right] \leq 2 C^2  \left( \frac{1+\alpha^2}{\alpha^2} \right)$.
\end{enumerate}%
\begin{proof}
By the definition given of $f_{\lambda_t}$ in \Eq{eq:f_lambda}, we have that for any $\lambda>0$: 
\begin{equation}
 \ExpecAlpha{(f_{\lambda}-r^{\alpha})^2(y)}  + \lambda \norm{f_{\lambda}}^2_{\Hilbert} \leq  \ExpecAlpha{(r^{\alpha})^2(y)} \leq \frac{1}{\alpha^2},
\end{equation}
which implies: 
\begin{equation}{\label{eq:upperbound_f_lambda_Hilbert}}
\norm{f_{\lambda}}_{\Hilbert} \leq \frac{1}{\alpha \sqrt{\lambda}}.
\end{equation}
%%%
On the other hand, using the close-form solution of $f_{\lambda}$ we get: 
\begin{equation}{\label{eq:upperbound_f_lambda_ltwo}}
\norm{f_{\lambda}}_{\Ltwo}=\norm{(\CO+\lambda I)^{(-1)} \CO r^{\alpha}}_{\Ltwo} \leq \norm{(\CO+\lambda I)^{(-1)} \CO } \norm{r^{\alpha}}_{\Ltwo} \leq \frac{1}{\alpha}.
\end{equation}
Moreover, we know: 
%%%%
\begin{equation}{\label{eq:LRt_f_lambdat}}
\begin{aligned}
\expec \left[ \norm{LR_t f_{\lambda_t}}^2_{\Hilbert}\right] &=  \expec \left[ \dotH{(1-\alpha)f_{\lambda_t}(x_t)K(x_t,\cdot) + \alpha f_{\lambda_t}(x'_t) K(x'_t,\cdot) }{(1-\alpha)f_{\lambda_t}(x_t)K(x_t,\cdot) + \alpha f_{\lambda_t}(x'_t) K(x'_t,\cdot)} \right]  \\ 
& = \expec \left[ (1-\alpha)^2 K(x_t,x_t)f_{\lambda_t}^2(x_t)+ 2 (1-\alpha) \alpha K(x_t,x'_t) f_{\lambda_t}(x_t)f_{\lambda_t}(x'_t) + \alpha^2 K(x'_t,x'_t) f_{\lambda_t}^2(x'_t) \right] \\
& \leq  C^2 \expec \left[ \left((1-\alpha) f_{\lambda_t}(x_t)+\alpha f_{\lambda_t}(x'_t) \right)^2\right]  \ \ \ \ (\text{\tiny{\Assumption{ass:kernel_map_upperbound} }})  \\
& \leq  C^2 \expec  \left[(1-\alpha) f_{\lambda_t}^2(x_t)+\alpha f_{\lambda_t}^2(x'_t)  \right] \ \ \ \ (\text{\tiny{Jensen's inequality}})
\\ 
&= C^2 \norm{f_{\lambda_t}}^2_{\ltwo} \leq \frac{C^2}{\alpha^2}.
\end{aligned}
\end{equation}
%%%
%
Then by putting together these elements, we can proof the first point of \Lemma{lemma:norm_gradient_descent}:
\begin{equation*}
\begin{aligned}
\norm{A_t f_{\lambda_t}-b_t}_{\Hilbert} &\leq (1-\alpha)\norm{ K(x_t,\cdot) f_{\lambda_t}(x_t)}_{\Hilbert} + \alpha \norm{K(x'_t,\cdot) f_{\lambda_t}(x'_t)}_{\Hilbert} + \lambda_t \norm{ f_{\lambda_t}}_{\Hilbert} + \norm{K(x'_t,\cdot)}_{\Hilbert} \\
& \leq  C( \frac{1}{\alpha \sqrt{\lambda_t}}+1)+ \frac{\sqrt{\lambda_t}}{\alpha} \ \ \ \ (\text{\tiny{\Eq{eq:upperbound_f_lambda_Hilbert} }})\\
& \leq \frac{1}{\sqrt{\lambda_t}} ( \frac{C}{\alpha}+C\sqrt{\lambda_t}+ \frac{\lambda_t}{\alpha})\\
& \leq \frac{1}{\sqrt{\lambda_t}} \left( \frac{C+1}{\alpha} + C\right),
\end{aligned}
\end{equation*}
where in the last equality we have used the hypothesis $t_0^{1-\theta} \geq b$, which implies $\lambda_t \leq 1$. 

Given the definition of $f_{\lambda}$, we have
$\CO f_{\lambda}+\lambda f_{\lambda}= \CO r^{\alpha}$,  which leads to: 
\begin{equation}
    A_tf_{\lambda_t}-b_t=LR_tf_{\lambda_t}-K(x'_t,\cdot) + \lambda_t f_{\lambda_t}=  (LR_t - \CO) f_{\lambda_t} +  \CO r^{\alpha} - K(x'_t,\cdot),
\end{equation}
Then, we verify the second point of \Lemma{lemma:norm_gradient_descent}: 
\begin{equation*}
\begin{aligned}
\expec \left[ \norm{A_t f_{\lambda_t}-b_t}^2_{\Hilbert} \right] &= \expec \left[ \norm{(LR_t - \CO) f_{\lambda_t} +  \CO r^{\alpha} - K(x'_t,\cdot)}^2_{\Hilbert} \right] \\
& \leq 2 \left[  \expec \left[ \norm{(LR_t - \CO) f_{\lambda_t}}^2_{\Hilbert} +  \norm{\CO r^{\alpha} - K(x'_t,\cdot)}^2_{\Hilbert} \right] \right] \ \ \  \ (\text{\tiny{$2 \dotH{a}{b} \leq  \norm{ a}^2_{\Hilbert} +  \norm{b}^2_{\Hilbert}  $}} )  \\
&\leq 2 \left[  \expec \left[ \norm{LR_t  f_{\lambda_t}}^2_{\Hilbert} +  \norm{ K(x'_t,\cdot)}^2_{\Hilbert} \right] \right]  \ \ \ \  (\text{\tiny{After developing the norm and taking expectations}}) \\
& \leq 2 C^2  \left( \frac{1+\alpha^2}{\alpha^2} \right)  \ \ \ \  (\text{\tiny{\Eq{eq:LRt_f_lambdat} and \Assumption{ass:kernel_map_upperbound} }}).
\end{aligned}
\end{equation*}
\end{proof}
\inlinetitle{Proof \Theorem{thm:sample_error_Hilbert}}{.}~
\begin{proof}
The idea of the proof is to use \Lemma{lemma:propositionA3} to generate a probabilistic bound for the quantity: 
\begin{equation*}
\mathcal{E}'_{\text{sample}}(t)=\norm{\sum_{j=1}^{t} \eta_j\Pi_{j+1}^t (A_jf_{\lambda_j}-b_j) }_{\Hilbert}.
\end{equation*}
We have shown in \Sec{sec:martingale_reversed_martingale} that the process $\{\eta_j \Pi_{j+1}^t (A_j f_{\lambda_j} - b_j)\}_{j=1}^{t}$ is a reversed martingale difference with respect to the sequence of sigma algebras $\mathcal{B}_j=\sigma((x_j,x'_j),...,(x_t,x'_t),...)$. The only element to finish the proof is to identify  
$M$ and $\sigma^2$. 

Given the definition of the random variables $\{A_t\}_{t\in \N}$, then for $t>1$ we can verify $A_t$ is a positive linear operator: 
\begin{equation}
    \dotH{A_tf}{f}=(1-\alpha) f^2(x_t) + \alpha f^2(x'_t) + \lambda_t \norm{f}^2_{\Hilbert} \geq 0\ \ \text{ for } f \in \Hilbert.
\end{equation}
Moreover as $\norm{A_t}\geq \lambda_t$ we have:
\begin{equation}
\norm{\idH-\eta_t A_t} \leq (1-\eta_t \lambda_t).
\end{equation}
Let us consider the following group of expressions: 
\begin{equation}{\label{eq:random_operator_upperbound}}
\begin{aligned}
\eta_j  \norm{\Pi_{j+1}^t} &=\eta_j \prod_{i=j+1}^{t} (1-\eta_i \lambda_i)  
\leq \eta_j \exp\left(- \sum_{i=j+1}^{t} \eta_i \lambda_i \right) \\
& =  \eta_j \exp\left(- \sum_{i=j+1}^{t} \frac{a b}{t_0+i} \right)  \leq \eta_j \exp\left(-ab \log\left(\frac{\bar{t}}{t_0+j+1}\right)\right) \\
  & = \frac{a}{(j+t_0)^{\theta}} \left(\frac{t_0+j+1}{\bar{t}} \right)^{ab}  =  a \frac{(t_0+j)^{ab-\theta}}{\bar{t}^{ab}} (1+\frac{1}{t_0+j})^{ab} \\
  & \leq a \frac{(t_0+j)^{ab-\theta}}{\bar{t}^{ab}} (1+\frac{1}{t_0})^{ab} \leq  \frac{e a(t_0+j)^{ab-\theta}}{\bar{t}^{ab}}. 
\end{aligned}
\end{equation}
This implies:
\begin{equation}{\label{eq_upperbound_zeta_j_hilbert}}
\begin{aligned}
\expec \left[ \norm{\zeta_j}^2_{\Hilbert}  \,|\, \mathcal{B}_{j+1}  \right] &=\expec \left[ \norm{ \eta_j\Pi_{j+1}^t (A_jf_{\lambda_j}-b_j)}^2_{\Hilbert} \,|\, \mathcal{B}_{j+1} \right] \\
& \leq \eta^2_j  \norm{\Pi_{j+1}^t}^2 \expec \left[ \norm{A_jf_{\lambda_j}-b_j)}^2_{\Hilbert} \,|\, \mathcal{B}_{j+1} \right] \\
& \leq \frac{2 (e a C )^2(t_0+j)^{2ab-2\theta}}{\bar{t}^{2ab}}  \left(\frac{1+\alpha^2}{\alpha^2} \right)  \ \ \ \ (\text{\tiny{\Lemma{lemma:norm_gradient_descent}, \Eq{eq:random_operator_upperbound}} and \Assumptions{ass:independence}, \ref{ass:kernel_map_upperbound}}).
\end{aligned}
\end{equation}
If $t_0 \geq 2$, we will find: 
\begin{equation}
\begin{aligned}
\sum_{j=1}^{t} \expec \left[ \norm{\zeta_j}^2_{\Hilbert} \,|\, \mathcal{B}_{j+1}  \right] & \leq \sum_{j=1}^t \frac{2 (e a C )^2(t_0+j)^{2ab-2\theta}}{\bar{t}^{2ab}}  \left(\frac{1+\alpha^2}{\alpha^2} \right) \\
& \leq \frac{2 (e a C )^2}{\bar{t}^{2ab}}\left(\frac{1+\alpha^2}{\alpha^2} \right) \int_{1}^t (t_0+s)^{2ab-2\theta}ds \\
& \leq \frac{2 (e a C )^2}{\bar{t}^{2ab}} \left(\frac{1+\alpha^2}{\alpha^2} \right)\left(\frac{1}{2ab - 2 \theta+1}\right) \bar{t}^{2ab-2\theta+1} \\ 
&= \begin{cases}
 \frac{ (e a C )^2}{ab - \theta+\frac{1}{2}}  \bar{t}^{- 2 \theta+1} &,  \ \ \text{if } \ \ \ ab > \theta - \frac{1}{2} \\
  \frac{ (e a C )^2}{ \theta-\frac{1}{2}- ab}  \bar{t}^{- 2 ab} &, \ \ \text{if } \ \ \ ab < \theta - \frac{1}{2}
\end{cases} \\
& \leq \frac{ (e a C )^2}{\abs{ab - \theta+\frac{1}{2}}} \bar{t}^{-2\left( ab \wedge (\theta- \frac{1}{2})\right) } 
\end{aligned}
\end{equation}
On the other hand, if $t_0^{1-\theta} \geq b$ we have: 
\begin{equation}
\begin{aligned}
\norm{\zeta_j}_{\Hilbert} &= \norm{ \eta_j\Pi_{j+1}^t (A_jf_{\lambda_j}-b_j)}_{\Hilbert} \\ &\leq  \frac{1}{\sqrt{\lambda_t}} \frac{e a(t_0+j)^{ab-\theta}}{\bar{t}^{ab}} \left( \frac{C+1}{\alpha} + C\right)  \ \ \ \ (\text{\tiny{\Lemma{lemma:norm_gradient_descent} and \Eq{eq:random_operator_upperbound}}})  \\ 
 & = \frac{e a(t_0+j)^{ab-\frac{(3\theta +1)}{2}}}{\sqrt{b} \bar{t}^{ab}} \left( \frac{C+1}{\alpha} + C\right) \\
 &= \begin{cases}
  \frac{e a}{\sqrt{b}} \left( \frac{C+1}{\alpha} + C\right) \bar{t}^{-\frac{(3\theta +1)}{2}} &,  \ \ \text{if } \ \ \ ab >  \frac{3\theta - 1}{2} \\
   \frac{e a}{\sqrt{b}} \left( \frac{C+1}{\alpha} + C\right) \bar{t}^{-2ab} &, \ \ \text{if } \ \ \ ab < \frac{3\theta - 1}{2}
\end{cases} \\
& =   \frac{e a}{\sqrt{b}} \left( \frac{C+1}{\alpha} + C\right) \bar{t}^{-\left(ab \wedge  \frac{3\theta - 1}{2} \right)}
\end{aligned}
\end{equation}
Then by \Lemma{lemma:propositionA3} we get with probability $1-\delta$: 
\begin{equation}
\begin{aligned}
    \mathcal{E}'_{\text{sample}}(t) &\leq  2 \left( \frac{e a}{3 \sqrt{b}} \left( \frac{C+1}{\alpha} + C\right) \bar{t}^{-\left(ab \wedge  \frac{3\theta - 1}{2} \right)}  + \sqrt{\frac{1}{\abs{ab-\theta+\frac{1}{2}}}} eaC \bar{t}^{-\left( ab \wedge (\theta- \frac{1}{2})\right)}  \right) \log\left( \frac{2}{\delta} \right) \\
    &= a b^{-\frac{1}{2}} B'_4 \bar{t}^{-\left(ab \wedge  \frac{3\theta - 1}{2} \right)}    + B'_5 a \bar{t}^{-\left( ab \wedge (\theta- \frac{1}{2})\right)}, 
\end{aligned}
\end{equation}
where 
\begin{equation}
B'_4= \frac{ 2 e}{3 } \left( \frac{C+1}{\alpha} + C\right) \log \left( \frac{2}{\delta}\right)  \ \ \ \ B'_5=2 \sqrt{\frac{1}{\abs{ab-\theta+\frac{1}{2}}}} eC    \log\left( \frac{2}{\delta} \right).
\end{equation}
\end{proof}
\end{lemma}

\section{Auxiliary results}{\label{app:auxiliary_results}}

The following result is frequently used in Appendix \ref{app:technical_results}. It was first proved by the Proposition A.3 in \cite{Tarres2014}. We include the result for completeness.
\begin{lemma}{\label{lemma:propositionA3}} (\textbf{Proposition A.3 (Pinelis-Bernstein)} \cite{Tarres2014} )
Let $\zeta_i$ be a martingale difference sequence in a Hilbert space. Suppose that almost surely $\norm{\zeta_i} \leq M$ and $\sum_{i=1}^{t} \expec \left[ \norm{\zeta_i}^2_{\Hilbert} \,|\, \Xi_{i-1}\right]  \leq \sigma_t^2$. Then the following holds with probability at least $1-\delta$ (with $\delta \in (0,1)$),
\begin{equation*}
\sup_{1 \leq k \leq t} \norm{\sum_{i=1}^k \zeta_i} \leq 2 \left( \frac{M}{3}+ \sigma_t  \right) \log(\frac{2}{\delta}).
\end{equation*}
\end{lemma}
The last inequality can be as well be applied for $\zeta_i$ being a reversed martingales difference sequence in a Hilbert space. With a small change where 
$\sum_{i=1}^{t} \expec \left[ \norm{\zeta_i}^2_{\Hilbert} \,|\, \Xi_{i-1}\right]  \leq \sigma_t^2$ is replaced by 
$\sum_{i=1}^{t} \expec \left[ \norm{\zeta_i}^2_{\Hilbert} \,|\, \mathcal{B}_{i+1} \right]  \leq \sigma_t^2$, where $\mathcal{B}_{i+1}$ is the sigma-algebra generated by observations after index $i$. 

\end{document}